%% file: paper.tex
\documentclass[10pt]{article} 

\usepackage{etoolbox}
\newcommand{\arxiv}[1]{\iftoggle{icml}{}{#1}}
\newcommand{\icml}[1]{\iftoggle{icml}{#1}{}}
\newtoggle{icml}
\global\toggletrue{icml}
\global\togglefalse{icml}

\icml{
\usepackage[accepted]{icml2024}
}

\newcommand{\loose}{\looseness=-1}

\usepackage[utf8]{inputenc} 
\usepackage[T1]{fontenc}    
\usepackage{url}            
\usepackage{booktabs}       
\usepackage{amsfonts}       
\usepackage{nicefrac}       
\usepackage{microtype}      

\usepackage{tocloft}            

\usepackage{enumitem}

\usepackage{breakcites}

\newtoggle{draft}
\togglefalse{draft}

\newtoggle{golfold}
\togglefalse{golfold}
\newcommand{\golfold}[1]{\iftoggle{golfold}{#1}{}}

\newtoggle{golfnew}
\toggletrue{golfnew}
\newcommand{\golfnew}[1]{\iftoggle{golfnew}{#1}{}}

\usepackage{mathrsfs}

\usepackage{algorithm}
\usepackage{verbatim}
\usepackage[noend]{algpseudocode}
\newcommand{\multiline}[1]{\parbox[t]{\dimexpr\linewidth-\algorithmicindent}{#1}}

\usepackage{multicol}

\usepackage{colortbl}

\usepackage{setspace}

\usepackage{transparent}

\usepackage{inconsolata}
\usepackage[scaled=.90]{helvet}
\usepackage{xspace}

\usepackage{pifont}
\input{arxiv_style}
\input{dylan}
\input{macros}

\let\underbar\undefined
\input{widebar}

\icml{
\usepackage{hyperref}
\newcommand{\alghyperref}[1]{\hyperref[#1]{Alg.~\ref*{#1}}}
}

\usepackage[suppress]{color-edits}
\addauthor{ak}{ForestGreen}
\addauthor{ys}{MidnightBlue}
\addauthor{df}{BurntOrange}
\addauthor{lw}{Red}
\addauthor{cut}{Purple}

\arxiv{
\usepackage[final]{showlabels}

}

\makeatletter
\let\OldStatex\Statex
\renewcommand{\Statex}[1][3]{%
  \setlength\@tempdima{\algorithmicindent}%
  \OldStatex\hskip\dimexpr#1\@tempdima\relax}
\makeatother

\usepackage{accents}
\usepackage{wrapfig}
\usepackage{tikz}
\usetikzlibrary{decorations.pathreplacing}
\usepackage{varwidth}

 \addtocontents{toc}{\protect\setcounter{tocdepth}{0}}

\let\oldparagraph\paragraph
\arxiv{\renewcommand{\paragraph}[1]{\oldparagraph{#1.}}}
\icml{\renewcommand{\paragraph}[1]{\textbf{#1.}}}

\newcommand{\paragraphi}[1]{\par\noindent\emph{#1.}}

\algrenewcommand\algorithmicrequire{\textbf{require}}
\newcommand{\algcommentlight}[1]{\textcolor{blue!70!black}{\transparent{0.5}\small{\texttt{\textbf{//\hspace{2pt}#1}}}}}

\usepackage{yfonts}

\arxiv{
    \title{Rich-Observation Reinforcement Learning \\with Continuous Latent Dynamics}
    \author{
      Yuda Song$^1$ \; Lili Wu$^2$ \; Dylan J. Foster$^2$ \; Akshay Krishnamurthy$^2$\\
      \vspace{-2mm} \\
      \normalsize{$^1$Carnegie Mellon University \qquad $^2$Microsoft Research}\\
      \vspace{-2mm} \\
      \normalsize{\texttt{yudas@cs.cmu.edu},\; \texttt{\{liliwu,dylanfoster,akshaykr\}@microsoft.com}}
    }

\date{}
}

\icml{
\icmltitlerunning{Rich-Observation Reinforcement Learning with Continuous Latent Dynamics}
}

\icml{
  \algrenewcommand\algorithmicindent{0.5em}
  }

\begin{document}

\arxiv{\maketitle}

\icml{
\twocolumn[
\icmltitle{Rich-Observation Reinforcement Learning with Continuous Latent Dynamics}


\icmlsetsymbol{equal}{*}

\begin{icmlauthorlist}
  \icmlauthor{Yuda Song}{cmu}
  \icmlauthor{Lili Wu}{msr}
  \icmlauthor{Dylan J. Foster}{msr}
  \icmlauthor{Akshay Krishnamurthy}{msr}
  \end{icmlauthorlist}
  
  \icmlaffiliation{cmu}{Carnegie Mellon University}
  \icmlaffiliation{msr}{Microsoft Research}
  
  \icmlcorrespondingauthor{Yuda Song}{yudas@cs.cmu.edu}

\vskip 0.3in
]
\printAffiliationsAndNotice{}  
}

\begin{abstract}
\input{abstract}
\end{abstract}



\icml{\vspace{-0.5cm}}

\section{Introduction}
\label{sec:intro}
\input{section_introduction}

\section{Problem Setting}
\label{sec:prelim}
\input{section_preliminaries}

\section{Statistical Complexity for the \framework Framework}
\label{sec:golf}
\input{section_golf}

\section{Efficient Algorithms for the \framework Framework}
\label{sec:algorithms}
\input{section_algorithms}

\section{Experiments}
\label{sec:experiments}
\input{section_experiments}

\section{Discussion}
\label{sec:discussion}
\input{section_discussion}

\clearpage

\icml{
\section*{Impact Statement}
This paper presents work whose goal is to advance the theoretical
foundations of reinforcement learning. There are many potential societal consequences of our work, none of which we feel must be specifically highlighted here.
}

\bibliography{refs} 
\icml{
\bibliographystyle{icml2024}
}

\clearpage
\onecolumn

\renewcommand{\contentsname}{Contents of Appendix}
\addtocontents{toc}{\protect\setcounter{tocdepth}{2}}
{\hypersetup{hidelinks}
\tableofcontents
}

\appendix  

\section{Additional Related Work}
\label{app:related}
\input{appendix_related}

\section{Offline RL Results with Representation Learning}
\label{app:replearn_offline}
\input{appendix_replearn_offline}

\clearpage
\section{Additional Experimental Results and Details}
\label{app:experiments}
\input{appendix_experiments}

\clearpage
\section{Technical Background}
\label{app:technical}
\input{appendix_auxiliary}

\clearpage

\section{Proofs from~\creftitle{sec:golf}}
\label{app:golf}
\input{appendix_golf}

\clearpage

\section{Proofs from~\creftitle{sec:algorithms}}
\input{appendix_main}



\end{document}


%% file: arxiv_style.tex
\arxiv{
\usepackage[letterpaper, left=1in, right=1in, top=1in, bottom=1in]{geometry}
\PassOptionsToPackage{hypertexnames=false}{hyperref}  
\usepackage{parskip}
\usepackage[dvipsnames]{xcolor}
\usepackage[colorlinks=true, linkcolor=blue!70!black, citecolor=blue!70!black,urlcolor=black,breaklinks=true]{hyperref}
}

\icml{
\usepackage[colorlinks=true, linkcolor=blue!70!black, citecolor=blue!70!black,urlcolor=blue!70!black,breaklinks=true]{hyperref}
}

\usepackage{microtype}
\usepackage{hhline}

\makeatletter
\newcommand{\neutralize}[1]{\expandafter\let\csname c@#1\endcsname\count@}
\makeatother

\newenvironment{thmmod}[2]
  {\renewcommand{\thetheorem}{\ref*{#1}#2}%
   \neutralize{theorem}\phantomsection
   \begin{theorem}}
  {\end{theorem}}

\usepackage{algorithm}

\arxiv{
\usepackage{natbib}
\bibliographystyle{plainnat}
\bibpunct{(}{)}{;}{a}{,}{,}
}

\usepackage{amsthm}
\usepackage{mathtools}
\usepackage{amsmath}
\usepackage{bbm}
\usepackage{amsfonts}
\usepackage{amssymb}

\usepackage{xpatch}

\usepackage{thmtools}
\usepackage{thm-restate}
\declaretheorem[name=Theorem,parent=section]{theorem}
\declaretheorem[name=Lemma,parent=section]{lemma}
\declaretheorem[name=Assumption, parent=section]{assumption}
\declaretheorem[name=Condition, parent=section]{condition}
\declaretheorem[qed=$\triangleleft$,name=Example,parent=section]{example}
\declaretheorem[name=Remark,style=definition, parent=section]{remark}
\declaretheorem[name=Proposition, parent=section]{proposition}

\makeatletter
  \renewenvironment{proof}[1][Proof]%
  {%
   \par\noindent{\bfseries\upshape {#1.}\ }%
  }%
  {\qed\newline}
  \makeatother

\theoremstyle{definition}  

\newtheorem{corollary}{Corollary}[section]

\theoremstyle{plain}
\newtheorem{definition}{Definition}[section]

\xpatchcmd{\proof}{\itshape}{\normalfont\proofnameformat}{}{}
\newcommand{\proofnameformat}{\bfseries}

\usepackage[nameinlink,capitalize]{cleveref}

\newcommand{\pref}[1]{\cref{#1}}
\newcommand{\pfref}[1]{Proof of \pref{#1}}

\renewcommand{\eqref}[1]{\texorpdfstring{\hyperref[#1]{(\ref*{#1})}}{(\ref*{#1})}}

\crefformat{equation}{#2Eq. (#1)#3}
\Crefformat{equation}{#2Eq. (#1)#3}

\Crefformat{figure}{#2Figure #1#3}
\Crefname{assumption}{Assumption}{Assumptions}
\Crefformat{assumption}{#2Assumption #1#3}
\Crefname{subsubsection}{Section}{Sections}
\crefformat{subsubsection}{#2Section #1#3}
\Crefformat{subsubsection}{#2Section #1#3}
\Crefname{alg}{Alg.}{Algs.}

\usepackage{crossreftools}
\newcommand{\creftitle}[1]{\crtcref{#1}}

\usepackage{xparse}

\ExplSyntaxOn
\DeclareDocumentCommand{\XDeclarePairedDelimiter}{mm}
 {
  \__egreg_delimiter_clear_keys: 
  \keys_set:nn { egreg/delimiters } { #2 }
  \use:x 
   {
    \exp_not:n {\NewDocumentCommand{#1}{sO{}m} }
     {
      \exp_not:n { \IfBooleanTF{##1} }
       {
        \exp_not:N \egreg_paired_delimiter_expand:nnnn
         { \exp_not:V \l_egreg_delimiter_left_tl }
         { \exp_not:V \l_egreg_delimiter_right_tl }
         { \exp_not:n { ##3 } }
         { \exp_not:V \l_egreg_delimiter_subscript_tl }
       }
       {
        \exp_not:N \egreg_paired_delimiter_fixed:nnnnn 
         { \exp_not:n { ##2 } }
         { \exp_not:V \l_egreg_delimiter_left_tl }
         { \exp_not:V \l_egreg_delimiter_right_tl }
         { \exp_not:n { ##3 } }
         { \exp_not:V \l_egreg_delimiter_subscript_tl }
       }
     }
   }
 }

\keys_define:nn { egreg/delimiters }
 {
  left      .tl_set:N = \l_egreg_delimiter_left_tl,
  right     .tl_set:N = \l_egreg_delimiter_right_tl,
  subscript .tl_set:N = \l_egreg_delimiter_subscript_tl,
 }

\cs_new_protected:Npn \__egreg_delimiter_clear_keys:
 {
  \keys_set:nn { egreg/delimiters } { left=.,right=.,subscript={} }
 }

\cs_new_protected:Npn \egreg_paired_delimiter_expand:nnnn #1 #2 #3 #4
 {
  \mathopen{}
  \mathclose\c_group_begin_token
   \left#1
   #3
   \group_insert_after:N \c_group_end_token
   \right#2
   \tl_if_empty:nF {#4} { \c_math_subscript_token {#4} }
 }
\cs_new_protected:Npn \egreg_paired_delimiter_fixed:nnnnn #1 #2 #3 #4 #5
 {
  \mathopen{#1#2}#4\mathclose{#1#3}
  \tl_if_empty:nF {#5} { \c_math_subscript_token {#5} }
 }
\ExplSyntaxOff

\XDeclarePairedDelimiter{\supnorm}{
  left=\lVert,
  right=\rVert,
  subscript=\infty
  }

%% file: dylan.tex
\DeclarePairedDelimiter{\abs}{\lvert}{\rvert} %
\DeclarePairedDelimiter{\brk}{[}{]}
\DeclarePairedDelimiter{\crl}{\{}{\}}
\DeclarePairedDelimiter{\prn}{(}{)}
\DeclarePairedDelimiter{\nrm}{\|}{\|}

\DeclareMathOperator{\En}{\mathbb{E}}

\newcommand{\Ehat}{\wh{\bbE}}

\DeclareMathOperator*{\argmin}{arg\,min} 
\DeclareMathOperator*{\argmax}{arg\,max}             

\newcommand{\wt}[1]{\widetilde{#1}}
\newcommand{\wh}[1]{\widehat{#1}}
\newcommand{\wb}[1]{\widebar{#1}}

\def\ddefloop#1{\ifx\ddefloop#1\else\ddef{#1}\expandafter\ddefloop\fi}
\def\ddef#1{\expandafter\def\csname bb#1\endcsname{\ensuremath{\mathbb{#1}}}}
\ddefloop ABCDEFGHIJKLMNOPQRSTUVWXYZ\ddefloop
\def\ddefloop#1{\ifx\ddefloop#1\else\ddef{#1}\expandafter\ddefloop\fi}
\def\ddef#1{\expandafter\def\csname b#1\endcsname{\ensuremath{\mathbf{#1}}}}
\ddefloop ABCDEFGHIJKLMNOPQRSTUVWXYZ\ddefloop
\def\ddef#1{\expandafter\def\csname sf#1\endcsname{\ensuremath{\mathsf{#1}}}}
\ddefloop ABCDEFGHIJKLMNOPQRSTUVWXYZ\ddefloop
\def\ddef#1{\expandafter\def\csname c#1\endcsname{\ensuremath{\mathcal{#1}}}}
\ddefloop ABCDEFGHIJKLMNOPQRSTUVWXYZ\ddefloop
\def\ddef#1{\expandafter\def\csname h#1\endcsname{\ensuremath{\widehat{#1}}}}
\ddefloop ABCDEFGHIJKLMNOPQRSTUVWXYZ\ddefloop
\def\ddef#1{\expandafter\def\csname hc#1\endcsname{\ensuremath{\widehat{\mathcal{#1}}}}}
\ddefloop ABCDEFGHIJKLMNOPQRSTUVWXYZ\ddefloop
\def\ddef#1{\expandafter\def\csname t#1\endcsname{\ensuremath{\widetilde{#1}}}}
\ddefloop ABCDEFGHIJKLMNOPQRSTUVWXYZ\ddefloop
\def\ddef#1{\expandafter\def\csname tc#1\endcsname{\ensuremath{\widetilde{\mathcal{#1}}}}}
\ddefloop ABCDEFGHIJKLMNOPQRSTUVWXYZ\ddefloop
\def\ddefloop#1{\ifx\ddefloop#1\else\ddef{#1}\expandafter\ddefloop\fi}
\def\ddef#1{\expandafter\def\csname scr#1\endcsname{\ensuremath{\mathscr{#1}}}}
\ddefloop ABCDEFGHIJKLMNOPQRSTUVWXYZ\ddefloop

\newcommand{\indic}{\mathbbm{1}}    

\newcommand{\eps}{\epsilon}
\newcommand{\veps}{\varepsilon}

\newcommand{\ldef}{\vcentcolon=}
\newcommand{\rdef}{=\vcentcolon}

%% file: macros.tex
\renewcommand{\ast}{\star}

\newcommand{\phistar}{\phi^{\star}}

\newcommand{\piunif}{\pi^{\mathsf{unif}}}

\newcommand{\optdp}{$\mathsf{OptDP}$\xspace}

\newcommand{\bhat}{\wh{b}}

\newcommand{\dec}{\mathsf{dec}_{\veps}}%
\newcommand{\Jm}{J^{\sM}}%
\newcommand{\sMbar}{\sss{\Mbar}}%
\newcommand{\Mbarlat}{\Mbar^{\mathrm{latent}}}%
\newcommand{\Mbarlatent}{\Mbar^{\mathrm{latent}}}%

\newcommand{\id}{\mathrm{id}}%
\newcommand{\sinit}{s_{\mathsf{root}}}%
\newcommand{\sfin}{s_{\mathsf{leaf}}}%
\newcommand{\Mlat}{M^{\mathrm{latent}}}%
\newcommand{\Mlatent}{M^{\mathrm{latent}}}%

\newcommand{\Plat}{P^{\mathrm{latent}}}
\newcommand{\Platent}{\Plat}

\newcommand{\sss}[1]{{\scriptscriptstyle #1}}
\newcommand{\sM}{\sss{M}}

\newcommand{\pim}[1][M]{\pi_{\sss{#1}}}
\newcommand{\Dhels}[2]{D^{2}_{\mathsf{H}}\prn*{#1,#2}}

\newcommand{\Mbar}{\wb{M}}

\newcommand{\unif}{\mathsf{unif}}

\newcommand{\iid}{i.i.d.\xspace}

\newcommand{\Qstar}{Q^{\star}}

\newcommand{\DS}{D_{\cS}}
\newcommand{\DA}{D_{\cA}}

\newcommand{\op}{\mathsf{op}}

\renewcommand{\epsilon}{\varepsilon}

\newcommand{\dims}{\dim_{\cS}}
\newcommand{\dima}{\dim_{\cA}}
\newcommand{\dimsa}{\dim_{\cS\cA}}

\newcommand{\ind}[1]{^{#1}}

\newcommand{\pistar}{\pi^{\star}}
\newcommand{\pihat}{\wh{\pi}}

\newcommand{\mathand}{\quad\text{and}\quad}

\newcommand{\framework}{\settingname}

\newcommand{\poly}{\mathrm{poly}}
\newcommand{\bigoh}{\cO}
\newcommand{\bigoht}{\wt{\cO}}

\newcommand{\Ccov}{C_{\mathrm{cov}}}

\newcommand{\golf}{\textsf{GOLF}\xspace}
\newcommand{\golfdbr}{\textsf{GOLF.DBR}\xspace}

\newcommand{\mainalg}{\textsf{CRIEE}\xspace}
\newcommand{\replearnog}{\textsf{BCRL}\xspace}
\newcommand{\replearn}{\textsf{BCRL.C}\xspace}
\newcommand{\settingname}{\textsf{RichCLD}\xspace}

\newcommand{\music}{\textsf{MuSIK}\xspace}
\newcommand{\pclast}{\textsf{PCLaSt}\xspace}

\newcommand{\tv}{\mathsf{TV}}

\newcommand{\disc}[1]{\mathsf{disc}_{#1}}
\newcommand{\ball}[1]{\mathsf{ball}_{#1}}
\newcommand{\reg}{\mathsf{Reg}}
\newcommand{\ccov}{C_\mathrm{cov}}
\newcommand{\cconc}{C_\mathrm{conc}}
\newcommand{\erep}{\veps_\mathrm{rep}}

\newcommand{\ehist}{\veps_\mathrm{hist}}
\newcommand{\unifpi}{\pi^{\mathsf{unif}}_\eta}
\newcommand{\algname}{\textsf{CRIEE}} 

\newcommand{\Lip}{\mathrm{Lip}}

\newcommand{\osb}{\mathrm{osb}}
\newcommand{\apx}{\mathrm{apx}}
\newcommand{\err}{\mathrm{err}}

\newcommand{\tr}{\mathrm{tr}}
\newcommand{\psb}{\widetilde{\cP}}
\newcommand{\realb}{\cP}

\newcommand{\dimeta}{d_{\eta}}
\newcommand{\dgamma}{d_{\gamma}}

\newcommand{\bm}[1]{\mathrm{d}\nu(#1)}
\newcommand{\ballunif}[1]{\upsilon(#1 \mid \ball{\eta}[\phi^\ast_h](#1))}


%% file: widebar.tex
\makeatletter
\let\save@mathaccent\mathaccent
\newcommand*\if@single[3]{%
  \setbox0\hbox{${\mathaccent"0362{#1}}^H$}%
  \setbox2\hbox{${\mathaccent"0362{\kern0pt#1}}^H$}%
  \ifdim\ht0=\ht2 #3\else #2\fi
  }
\newcommand*\rel@kern[1]{\kern#1\dimexpr\macc@kerna}
\newcommand*\widebar[1]{\@ifnextchar^{{\wide@bar{#1}{0}}}{\wide@bar{#1}{1}}}
\newcommand*\underbar[1]{\@ifnextchar_{{\under@bar{#1}{0}}}{\under@bar{#1}{1}}}
\newcommand*\wide@bar[2]{\if@single{#1}{\wide@bar@{#1}{#2}{1}}{\wide@bar@{#1}{#2}{2}}}
\newcommand*\under@bar[2]{\if@single{#1}{\under@bar@{#1}{#2}{1}}{\under@bar@{#1}{#2}{2}}}
\newcommand*\wide@bar@[3]{%
  \begingroup
  \def\mathaccent##1##2{%
    \let\mathaccent\save@mathaccent
    \if#32 \let\macc@nucleus\first@char \fi
    \setbox\z@\hbox{$\macc@style{\macc@nucleus}_{}$}%
    \setbox\tw@\hbox{$\macc@style{\macc@nucleus}{}_{}$}%
    \dimen@\wd\tw@
    \advance\dimen@-\wd\z@
    \divide\dimen@ 3
    \@tempdima\wd\tw@
    \advance\@tempdima-\scriptspace
    \divide\@tempdima 10
    \advance\dimen@-\@tempdima
    \ifdim\dimen@>\z@ \dimen@0pt\fi
    \rel@kern{0.6}\kern-\dimen@
    \if#31
      \overline{\rel@kern{-0.6}\kern\dimen@\macc@nucleus\rel@kern{0.4}\kern\dimen@}%
      \advance\dimen@0.4\dimexpr\macc@kerna
      \let\final@kern#2%
      \ifdim\dimen@<\z@ \let\final@kern1\fi
      \if\final@kern1 \kern-\dimen@\fi
    \else
      \overline{\rel@kern{-0.6}\kern\dimen@#1}%
    \fi
  }%
  \macc@depth\@ne
  \let\math@bgroup\@empty \let\math@egroup\macc@set@skewchar
  \mathsurround\z@ \frozen@everymath{\mathgroup\macc@group\relax}%
  \macc@set@skewchar\relax
  \let\mathaccentV\macc@nested@a
  \if#31
    \macc@nested@a\relax111{#1}%
  \else
    \def\gobble@till@marker##1\endmarker{}%
    \futurelet\first@char\gobble@till@marker#1\endmarker
    \ifcat\noexpand\first@char A\else
      \def\first@char{}%
    \fi
    \macc@nested@a\relax111{\first@char}%
  \fi
  \endgroup
}
\newcommand*\under@bar@[3]{%
  \begingroup
  \def\mathaccent##1##2{%
    \let\mathaccent\save@mathaccent
    \if#32 \let\macc@nucleus\first@char \fi
    \setbox\z@\hbox{$\macc@style{\macc@nucleus}_{}$}%
    \setbox\tw@\hbox{$\macc@style{\macc@nucleus}{}_{}$}%
    \dimen@\wd\tw@
    \advance\dimen@-\wd\z@
    \divide\dimen@ 3
    \@tempdima\wd\tw@
    \advance\@tempdima-\scriptspace
    \divide\@tempdima 10
    \advance\dimen@-\@tempdima
    \ifdim\dimen@>\z@ \dimen@0pt\fi
    \rel@kern{0.6}\kern-\dimen@
    \if#31
      \underline{\rel@kern{-0.6}\kern\dimen@\macc@nucleus\rel@kern{0.4}\kern\dimen@}%
      \advance\dimen@0.4\dimexpr\macc@kerna
      \let\final@kern#2%
      \ifdim\dimen@<\z@ \let\final@kern1\fi
      \if\final@kern1 \kern-\dimen@\fi
    \else
      \underline{\rel@kern{-0.6}\kern\dimen@#1}%
    \fi
  }%
  \macc@depth\@ne
  \let\math@bgroup\@empty \let\math@egroup\macc@set@skewchar
  \mathsurround\z@ \frozen@everymath{\mathgroup\macc@group\relax}%
  \macc@set@skewchar\relax
  \let\mathaccentV\macc@nested@a
  \if#31
    \macc@nested@a\relax111{#1}%
  \else
    \def\gobble@till@marker##1\endmarker{}%
    \futurelet\first@char\gobble@till@marker#1\endmarker
    \ifcat\noexpand\first@char A\else
      \def\first@char{}%
    \fi
    \macc@nested@a\relax111{\first@char}%
  \fi
  \endgroup
}
\makeatother

%% file: abstract.tex
\arxiv{Sample-efficiency and reliability remain major bottlenecks toward wide
adoption of reinforcement learning algorithms in continuous settings
with high-dimensional perceptual inputs.
Toward addressing these challenges, 
we introduce a new theoretical framework, \framework (``Rich-Observation RL with Continuous Latent Dynamics''), in which the agent performs control based on high-dimensional observations, but the environment is governed by low-dimensional latent states and Lipschitz continuous dynamics. 
Our main contribution is a new algorithm for this setting that is provably statistically and computationally efficient. 
The core of our algorithm is a new representation learning objective; we show that prior representation learning schemes tailored to discrete dynamics do not naturally extend to the continuous setting. 
Our new objective is amenable to practical implementation, and empirically, we find that it compares favorably to prior schemes in a standard evaluation protocol. 
We further provide several insights into the statistical complexity of the \framework framework, in particular proving that certain notions of Lipschitzness that admit sample-efficient learning in the absence of rich observations are insufficient in the rich-observation setting. \loose}

\icml{Sample-efficiency and reliability remain major bottlenecks toward wide
adoption of reinforcement learning algorithms in continuous settings
with high-dimensional perceptual inputs. Toward addressing these challenges, 
we introduce a new theoretical framework, \framework (``Rich-Observation RL with Continuous Latent Dynamics''), where the agent performs control based on high-dimensional observations, but the environment is governed by low-dimensional latent states and Lipschitz continuous dynamics. 
Our main contribution is an algorithm for \framework that is provably statistically and computationally efficient. 
The core of our algorithm is a new representation learning objective; we show that prior representation learning schemes tailored to discrete dynamics do not naturally extend to the continuous setting. 
Our objective is amenable to practical implementation, and empirically, it compares favorably to prior schemes in a standard evaluation protocol. 
We further provide several insights into the statistical complexity of the \framework framework, in particular proving that certain notions of Lipschitzness that admit sample-efficient learning in the absence of rich observations are insufficient in the rich-observation setting. \loose
}


%% file: section_introduction.tex
It is becoming increasingly common to deploy algorithms for reinforcement learning and control in systems where the underlying (``latent'') dynamics are nonlinear, continuous, and low-dimensional, yet the agent perceives the environment through high-dimensional (``rich'') observations such as images from a camera \citep{wahlstrom2015pixels,levine2016end,kumar2021rma,nair2022r3m,baker2022video,brohan2022rt}. These domains demand that agents (i) efficiently explore in the face of complex nonlinearities, and (ii) learn continuous representations that respect the structure of the latent dynamics, ideally in tandem with exploration. In spite of extensive empirical investigation into modeling and algorithm design \citep{laskin2020curl,yarats2021image,hafner2023mastering}, sample-efficiency and reliability remain major challenges \citep{dean2019robust}, and our understanding of fundamental algorithmic principles for representation learning and exploration is still in its infancy. \loose

Toward understanding algorithmic principles and fundamental limits for 
reinforcement learning and control with high-dimensional observations, a recent 
line of theoretical research adopts the framework of
 \emph{rich-observation reinforcement learning} 
 \citep[c.f.,][]{du2019provably,misra2020kinematic,mhammedi2020learning,zhang2022efficient,mhammedi2023representation}. 
\icml{This line of work has led to new provably efficient algorithms, but, to date, has focused on systems with discrete (``tabular'') or linear latent dynamics, which is unsuitable for most real-world control applications.\loose}
\arxiv{Rich-observation RL provides a mathematical framework for the design and analysis of algorithms that perform exploration in the presence of high-dimensional observations, with an emphasis on generalization and sample-efficiency. However, existing work in this domain is largely restricted to systems with discrete (``tabular'') latent dynamics, which is unsuitable for most real-world control applications.\loose}

\paragraph{The \framework framework}
We initiate the study of theoretically-sound algorithms for reinforcement learning with continuous latent dynamics by introducing a new framework, the \framework (``Rich-Observation RL with Continuous Latent Dynamics'') framework. In the \framework framework, the agent performs control based on high-dimensional observations (e.g., images from a camera), but the underlying state obeys Lipschitz continuous dynamics. 
Lipschitz continuous dynamics have been studied extensively in the absence of rich observations~\citep{kakade2003exploration,shah2018q,henaff2019explicit,ni2019learning,song2019efficient,cao2020provably,sinclair2023adaptive} and are versatile enough to capture applications ranging from robotic control~\citep{underactuated} to online resource allocation~\citep{sinclair2023adaptive}.
Our framework addresses these applications in a more realistic setting where the agent perceives the system through high-dimensional feedback. \loose

The central challenge in rich-observation \arxiv{reinforcement learning}\icml{RL}---the
\framework framework included---is that representation learning and exploration
must be interleaved. Agents need to learn a good representation to
guide exploration, but doing so requires
gathering data from throughout the space, and is difficult
unless the agent already knows how to explore. 
\icml{Since exploration under
Lipschitz dynamics in the absence of rich observations is
relatively well-understood, the main algorithmic challenge can be summarized as:\loose}
\arxiv{Absent high-dimensional observations, the literature on Lipschitz MDPs \citep{kakade2003exploration,shah2018q,henaff2019explicit,sinclair2019adaptive,ni2019learning,song2019efficient,sinclair2020adaptive,cao2020provably,sinclair2023adaptive} provides algorithms---typically based on (adaptive) discretization---that can provably perform exploration in the presence of Lipschitz dynamics, but lifting such techniques to the more challenging \framework framework requires combining them with representation learning. In this context, perhaps the most fundamental algorithm design challenge for the \framework framework can be summarized as:\loose}
\begin{quote}
  \begin{center}
        \emph{How can we design representation learning schemes that gracefully compose with exploration in the presence of continuous latent dynamics?
    }
    \end{center}
  \end{quote}
In this paper, we develop representation learning schemes that address this challenge, and combine them with principled exploration, yielding provable algorithms for reinforcement learning in the \framework framework. 

\arxiv{\subsection{Contributions}}
\icml{\paragraph{Contributions}}
We develop new statistical, algorithmic, and empirical
results in the \framework framework. \loose

\icml{\begin{enumerate}[label=\textnormal{(C\arabic*)},topsep=0pt,itemsep=-1ex,partopsep=0ex,parsep=1ex]}
\arxiv{\begin{enumerate}[label=\textnormal{(C\arabic*)},topsep=0pt,itemsep=0ex,partopsep=0ex,parsep=1ex]}

\item \label{item:c1} \textbf{Statistical complexity.} 
We establish that the \framework framework is statistically tractable by analyzing a variant of a general-purpose \arxiv{(but computationally inefficient) }algorithm for reinforcement learning, \golf \citep{jin2021bellman,xie2023role,amortila2024mitigating}. The sample complexity \arxiv{guarantee }has a nonparametric flavor, scaling exponentially with the latent state dimension, as expected. 
\arxiv{Due to non-trivial challenges related to misspecification, this requires a novel analysis based on a variant of the \emph{coverability} parameter introduced by \citet{xie2023role} along with an application of the robust regression technique of~\citet{amortila2024mitigating}.}%
\icml{We also show that weaker notions of continuity are insufficient for statistically tractable learning, which separates the rich-observation setting from the classical one.} 
\arxiv{We complement this positive result by showing that weaker notions of Lipschitzness known to be sufficient for sample-efficient learning in Lipschitz MDPs without rich observations \citep{sinclair2019adaptive,song2019efficient,sinclair2020adaptive,cao2020provably,sinclair2023adaptive} are no longer tractable in the \framework framework, thereby establishing a separation between these frameworks.}

\item \label{item:c2} \textbf{Representation learning.}  We provide a new
  representation learning scheme, Bellman Consistent Representation 
  Learning (\replearn), which provably learns a
  representation that enables downstream exploration and
  reward maximization in the \framework framework. 
  \arxiv{\replearn is derived by adapting a certain min-max-min representation learning objective studied in the context of 
  low-rank MDPs and tabular Block MDPs \citep{zhang2022efficient,mhammedi2023efficient,modi2021model} to handle continuous dynamics, and is amenable to practical implementation through first-order methods.}%
  We complement this
  result by showing that standard representation learning procedures
  considered in prior work do not readily adapt to handle Lipschitz
  continuous dynamics as-is
  \citep{misra2020kinematic,lamb2023guaranteed,mhammedi2023representation}.\loose
  
\item \label{item:c3} \textbf{End-to-end algorithm.} 
By interleaving \replearn with exploration, we obtain \mainalg, a provably (statistically and computationally) efficient algorithm for learning in the \framework framework. 
\arxiv{\mainalg is computationally efficient whenever the \replearn objective can be solved efficiently, and is amenable to practical implementation.}
  
\item \label{item:c4} \textbf{Practical implementation and empirical evaluation.}  We
  derive a practical variant of \replearn and provide a qualitative
  and quantitative evaluation in visual navigation environments~\citep{koul2023pclast} and the visual D4RL benchmark~\citep{lu2022challenges}. 
  \arxiv{Focusing on representation learning, we show that \replearn effectively recovers the latent dynamics structure in a collection of two-dimensional navigation environments \citep{koul2023pclast}, given access to exploratory data. We then show in visual D4RL benchmark \citep{lu2022challenges} that the representations learned by \replearn are competitive to 
  the ones from the previous best-performing representation algorithm \citep{koul2023pclast}.}
\end{enumerate}
Together, we believe these results constitute a useful starting point for further theoretical investigation into continuous \arxiv{reinforcement learning}\icml{RL} with rich observations.
We discuss avenues for future work, including improved sample complexity, weakened continuity assumptions, and adaptivity, in~\pref{sec:discussion}.%

\paragraph{Paper organization}
\cref{sec:prelim} introduces the \framework framework and our sample
complexity desiderata. \cref{sec:golf} addresses~\ref{item:c1} while
\cref{sec:algorithms} addresses~\ref{item:c2}
and~\ref{item:c3}. Experiments~\ref{item:c4} are presented in~\pref{sec:experiments},
and we conclude with a discussion in~\cref{sec:discussion}. 
Proofs and additional details, including further related work, are
deferred to the appendix.


%% file: section_preliminaries.tex
\arxiv{
In this section we recall the basic online reinforcement learning
protocol, then formally introduce the \framework framework. \loose
}

\paragraph{Markov decision processes} We consider an episodic finite-horizon
Markov decision process (MDP) with horizon $H \in \bbN$. An MDP 
$M := (\cX, \cA,H, P, R)$ consists of state space $\cX$,
 action space $\cA$, transition distribution $P = \crl*{P_h: 
\cX\times\cA \to \Delta(\cX)}_{h=1}^H$, and reward function\footnote{For simplicity, we assume that 
the reward is known.} $R = \crl*{R_h : \cX \times \cA \to
[0,1]}_{h=1}^H$. 
Executing a nonstationary policy $\pi = (\pi_1,\ldots,\pi_H)$, where 
each $\pi_h \in (\cX \to \Delta(\cA))$, for an
episode induces a trajectory $\tau = (x_1,a_1,r_1,\ldots,x_H,a_H,r_H)$
via the process $x_h \sim P_h(x_{h-1},a_{h-1})$, $a_h \sim \pi_h(x_h)$ and $r_h =
R_h(x_h,a_h)$ for all $h\in\brk{H}$, without loss generality, we assume that 
there is a fixed initial state $x_1$.
We let
$\bbP^{\pi}$ and $\bbE^\pi$ denote the law and expectation under this process, and we define 
the occupancy measures $d^\pi_h(x) = \bbP^{\pi}(x_h = x)$ and 
$d^\pi_h(x,a) = \bbP^{\pi}(x_h = x, a_h = a)$. 
Following convention, we assume $\sum_{h=1}^H r_h \leq 1$ almost surely.
We define $J(\pi)=\En^{\pi}\brk[\big]{\sum_{h=1}^{H}r_h}$ as the expected
reward under the policy $\pi$ and let $\pi^\star \in \argmax_{\pi \in \Pi}J(\pi)$ be the optimal policy that satisfies Bellman's equations, where $\Pi$ is the set of all randomized non-stationary policies.

\paragraph{The \framework model}
A \framework model is an MDP with particularly structured
dynamics and rewards, corresponding to a rich-observation MDP with a
continuous latent state space.\loose

A rich-observation
MDP~\citep{krishnamurthy2016pac,du2019provably,misra2020kinematic,zhang2022efficient,mhammedi2023representation}
is an MDP with a \emph{latent state space} $\cS$ and decoders
$\{\phi_h^\star\}_{h=1}^H: \cX \to \cS$ such that:
(1) the reward function depends only on the latent state
$s_h\ldef{}\phistar_h(x_h)$, i.e., $R_h(x_h,a_h) =
R^{\mathrm{latent}}_h(\phi_h^\star(x_h),a_h)$, and (2) the transition
dynamics operate on the latent state in the sense that, for each
$x_{h},a_{h}$ the dynamics evolve as $s_{h+1} \sim
P^{\mathrm{latent}}_h(\phi_{h}^\star(x_h),a_h)$ and $x_{h+1} \sim
E_{h+1}(s_{h+1})$. Here $P_h^{\mathrm{latent}}: \cS \times \cA \to
\Delta(\cS)$ is the \emph{latent dynamics}, $R_h^{\mathrm{latent}}:
\cS\times\cA\to [0,1]$ is the \emph{latent reward} (we often omit the
superscript latent from these objects), and $E_h: \cS \to \Delta(\cX)$
is an \emph{emission distribution}. \arxiv{Under this structure, the
  trajectory $\tau$ can be augmented with \emph{latent states}
  $s_1,\ldots,s_H$ such that $\tau :=
  (s_1,x_1,a_1,r_1,\ldots,s_H,x_H,a_H,r_H)$. Going forward, we refer
  to $x_h$ as an \emph{observation} and refer to $s_h$ as a
  \emph{latent state}. }
$\mathrm{supp} E_h(\cdot \mid s_h) \cap \mathrm{supp} E_h(\cdot \mid
s'_h) = \emptyset, \; \forall s_h\neq s'_h \in \cS$.

In a \framework model, we further posit that (a) the latent state and
action spaces are continuous and (b) the latent dynamics are Lipschitz
continuous w.r.t. the latent states and actions. 
\arxiv{This allows us to
model problems with continuous, non-linear dynamics and rich sensory
inputs and departs from prior work on rich-observation MDPs that
either considered discrete latent state spaces or linear latent
dynamics.} 
Concretely, we assume that the latent state space is a metric space 
$(\cS,D_{\cS})$ with \emph{covering dimension} $\dims\in\bbR_{+}$. That is, for any $\eta > 0$,
there exists a set of \emph{covering states} $\cS_\eta$ with size $S_\eta := |\cS_\eta| \leq (2/\eta)^{\dims}$
such that 
\begin{align*}
  \forall s \in \cS,\; \exists s_\eta \in \cS_\eta\;:\; D_\cS(s, s_\eta)\leq \eta/2.
\end{align*}
Analogously, we assume that the action space $\cA$ is a metric space
$(\cA,D_{\cA})$ with covering dimension $\dima$ and covering set
$\cA_\eta$ with size $A_\eta$. We use $\cS_\eta$ and $\cA_\eta$ to
refer to \emph{fixed} but arbitrary coverings. We define a joint
metric over state-action pairs via $D((s,a), (s',a'))=D_{\cS}(s,s') +
D_{\cA}(a,a')$, and abbreviate $\dimsa := \dims + \dima$.

\arxiv{
\begin{example}
When $\cS$ and $\cA$ are Euclidean unit balls in $\bbR^{d}$, we have $\dims=\dima=d$.  
\end{example}
}

To enable sample-efficient learning guarantees and take advantage of
the metric structure for $\cS$ and $\cA$, we make a continuity
assumption on the latent dynamics\icml{ and posit that they are
Lipschitz continuous in \emph{total variation distance}.}
\arxiv{, inspired by the literature on Lipschitz MDPs
\citep{kakade2003exploration,shah2018q,henaff2019explicit,sinclair2019adaptive,ni2019learning,song2019efficient,sinclair2020adaptive,cao2020provably,sinclair2023adaptive}. While there are many canonical notions of continuity in the literature, we
focus on perhaps the simplest, Lipschitz continuity with respect to
\emph{total variation distance}.}%
\begin{assumption}[Lipschitz dynamics]
  \label{ass:lipschitz}
  For every $h \in [H]$, for all $s,s'\in\cS$ and
  $a,a'\in\cA$,\footnote{For probability measures $\bbP$ and $\bbQ$
    over a measurable space $(\cX,\mathscr{E})$, total variation
    distance is defined via 
    $\nrm{\bbP-\bbQ}_{\tv}=\sup_{E\in\mathscr{E}}\abs{\bbP(E)-\bbQ(E)}=\frac{1}{2}\int\abs{d\bbP-d\bbQ}$.
  } 
  \begin{align*}
    \| P_h(\cdot \mid s,a) - P_h(\cdot \mid s',a') \|_{\tv} & \leq D((s,a), (s',a')),\\
    \abs*{R_h(s,a)- R_h(s',a')} & \leq D((s,a), (s',a')).
  \end{align*}
\end{assumption}
This assumption asserts that nearby states and actions lead to similar
transitions and rewards, but otherwise allows the
dynamics to be arbitrarily nonlinear. Thus it captures many
control-theoretic settings, as described in the next example.

\begin{example}
  Consider a system with $\cS=\bbR^{\dims},\cA=\bbR^{\dima}$, and where transitions and rewards follow the law
  \begin{align*}
    s_{h+1}=f(s_h,a_h)+\omega_h,\mathand{}r_h=g(s_h,a_h),
  \end{align*}
  for $\smash{f:\cS\times\cA\to\cS}$, $\smash{g:\cS\times\cA\to\bbR}$, and
  $\smash{\omega_h\sim{}\cN(0,I_{\dims})}$. Then \cref{ass:lipschitz}
  holds whenever $f$ and $g$ are Lipschitz, i.e.,
  $\nrm*{f(s,a)-f(s',a')}_2\leq{}D((s,a),(s',a'))$ and
  $\abs*{g(s,a)-g(s',a')}\leq{}D((s,a),(s',a'))$.
\end{example}

Lipschitz dynamics and several weaker assumptions have been studied in
the absence of rich observations in prior
work~\citep{kakade2003exploration,shah2018q,ni2019learning,song2019efficient,cao2020provably,sinclair2023adaptive}. We
show later that some of these weaker assumptions, which are sufficient
in the classical setting, do not enable sample efficient learning with
rich observations. However, we leave a deeper understanding of more
refined continuity assumptions to future work. \loose

\paragraph{Function approximation and learning objective}
We consider online reinforcement learning in the function
approximation setting. Here, the algorithm interacts with an unknown
\framework MDP in episodes, where in the $t^{\mathrm{th}}$ episode,
the algorithm selects policy $\pi\ind{t}$ and collects trajectory
$\tau\ind{t}$ by executing $\pi\ind{t}$ in the MDP. The goal of the
algorithm is to identify an $\veps$-optimal policy $\pihat$ such that
$J(\pistar) - J(\pihat) \leq \veps$ with probability at least
$1-\delta$.

\arxiv{For the \framework framework, we do not observe the latent state $s_h$
directly, and must learn from the observations $x_h$, which
necessitates representation learning.} 
To facilitate achieving this learning goal in a sample-efficient
manner, we assume access to a \emph{decoder class} $\Phi \subset (\cX
\to \cS)$ containing the true decoders $\phi_h^\star$.

\begin{assumption}[Decoder realizability]\label{assum:decoder_realizability}
  We have $\phi_h^\star \in \Phi$ for all $h \in [H]$.  
\end{assumption}

Given this, we say an algorithm is sample efficient if it learns an
$\veps$-optimal policy in
$\poly(H,\log(|\Phi|/\delta),\veps^{-\poly(\dim_{\cS\cA})})$
episodes.\footnote{Following the convention in the
rich-observation literature, we assume for simplicity that $|\Phi| <
\infty$ and provide sample complexity bounds that scale with $\log
|\Phi|$, but it is trivial to extend our results to other notions of
statistical complexity for $\Phi$.} Crucially, there is no dependence
on $|\cX|$. We note that this type of guarantee generalizes existing
results for (a) Block MDPs, which have finite $\cS$ and $\cA$ with the
identity metric, and (b) Lipschitz MDPs in the absence of rich
observations, where a ``nonparametric'' sample complexity of
$\veps^{-(\dim_{\cS\cA}+2)}$ is optimal~\citep{sinclair2023adaptive}.\loose

\paragraph{Additional notation}
For a pair of policies $\pi,\pi'\in\Pi$, we 
define $\pi \circ_t \pi'$ as the policy that acts according to $\pi$ for the 
first $t-1$ steps and $\pi'$ for the remaining steps $t,\ldots,H$.
We use the shorthand $x_h \sim \pi$
to indicate that $x_h$ is drawn from the law $\bbP^\pi$, and use the
shorthand $(x_h, a_h) \sim \pi$ analogously. The notation $\bigoht(\cdot)$ indicates that a bound holds up to factors polylogarithmic in parameters appearing in the expression.


%% file: section_golf.tex
At first glance, it may not be apparent to the reader whether the \framework framework
is even tractable. Hence, in this section, we perform a preliminary investigation into statistical complexity, deferring the development of computationally efficient
algorithms to~\pref{sec:algorithms}. We present two results: (1)
we show that the \framework framework is indeed tractable, and the sample
complexity scaling as $O(\epsilon^{-\poly(\dimsa)})$ is achievable (\cref{thm:golf}),
and (2) we prove a statistical separation between the framework
and its non-rich-observation counterpart, the Lipschitz MDP,
showing that weaker notions of Lipschitzness that lead to
sample-efficient learning in the latter are intractable in the former (\cref{thm:lower}).
\paragraph{Upper bound for the \framework framework}
We provide an upper bound for the
\framework model by instantiating a variant of the computationally inefficient \golf algorithm
of~\citet{jin2021bellman,xie2023role}. Prior analyses of this algorithm consider
general value function approximation and obtain sample complexity scaling
with the structural parameters for the MDP such as the \emph{Bellman-Eluder
  dimension} \citep{jin2021bellman} or \emph{coverability}
\citep{xie2023role}. Roughly speaking, these structural parameters
measure the number of distributions in the MDP (induced by policies)
that one must visit before one can extrapolate to any other
distribution. Both quantities are known to be small in several MDP classes
of interest; notably, for tabular Block MDPs, both scale only with the number of latent states
$\abs{\cS}$ and number of actions $\abs{\cA}$, and are independent of
the size of the observation space.\loose

Unfortunately, due to the continuity of the latent state space in the
\framework model, Bellman-Eluder dimension (as well as other complexity measures \citep{jiang2017contextual,du2021bilinear}) and
coverability can both be unbounded, leading to vacuous guarantees from
prior analyses.  To address this, we introduce a notion of
\emph{approximate coverability} (\pref{def:approximate_coverability}),
which extends coverability to allow for a certain form of
misspecification.  A second challenge arises because the natural value
function class to use in \golf (Lipschitz functions composed with decoders) is
infinitely large and must be discretized to admit uniform
convergence. Discretization introduces an approximation error, which
has an unfavorable interaction with the misspecification of the MDP
(i.e., the approximation parameter in approximate coverability), and
results in a slow convergence rate even with careful treatment of
these error terms. We address this by employing the recent
disagreement-based regression (\textsf{DBR}) technique
of~\citet{amortila2024mitigating} that avoids the interaction between
these error terms arising from misspecification.
To conclude, we show that \framework framework satisfies
approximate coverability, and by carefully trading off
misspecification with distribution shift, we can show that the
\framework model is indeed learnable. 

\begin{theorem}[PAC upper bound for \framework framework; informal]\label{thm:golf}
    Suppose
    \cref{ass:lipschitz,assum:decoder_realizability} hold.
    For any $\delta \in (0,1)$ and $\veps\in(0,1)$, with probability at least $1-\delta$,
    \golfdbr(\pref{alg:golf}) outputs a policy $\widehat \pi$ satisfying
    $J(\pi^\ast) - J(\widehat \pi) \leq \epsilon$ with sample complexity
    \begin{align*}
        \cO\prn*{\frac{H^{(2\dimsa+\dima+3)} \log \prn*{TH|\Phi| / \delta \epsilon}} 
        {\epsilon^{(2\dimsa+\dima+2)}}}.
    \end{align*}
  \end{theorem}
  
See~\pref{app:golf} for a formal statement and proof. 
Regarding
the sample complexity, the exponent on $\epsilon$ scales with
$2\dimsa+\dima+2$, which is worse than the exponent $\dimsa+2$
in the minimax rate for Lipschitz MDPs without rich observations \citep{sinclair2023adaptive}.
There are two primary sources for this in our analysis.  First, we
incur a quadratic dependence on the effective state space size;
specialized to tabular Block MDPs, this takes
$O(\abs{\cS}^2)$ instead of $O(\abs{\cS})$.  The second arises from sampling uniformly over the action covering set $\cA_\eta$ to estimate Bellman errors; this yields the additional $\dima$ term. 
Both of these
can be avoided in Lipschitz MDPs but manifest in all existing analyses for
rich-observation settings (the former in tabular Block
MDPs~\citep{jiang2017contextual,misra2020kinematic,jin2021bellman,zhang2022efficient,mhammedi2023representation}
and the latter in Theorem 4 of~\citet{jiang2017contextual}).  However,
it remains open to determine if either of these dependencies are
necessary in rich-observation settings.

\paragraph{Lower bound under weaker continuity}
\pref{ass:lipschitz} places a rather stronger Lipschitz continuity
assumption on the latent dynamics. In the absence of rich
observations, prior work obtains sample-efficient algorithms under
weaker conditions. Specifically, the best existing results for
model-free
methods~\citep{sinclair2019adaptive,song2019efficient,cao2020provably}
assume only that $Q^\star$ and $V^\star$ (the optimal value functions) are Lipschitz continuous with
respect to the metric $D\prn{\cdot,\cdot}$, while the best results
for model-based methods~\citep{sinclair2023adaptive} measure
continuity via the 1-Wasserstein distance rather than via total
variation.\footnote{For probability measures $\bbP$ and $\bbQ$ the
1-Wasserstein distance is defined as $\nrm{\bbP-\bbQ}_{\mathsf{W}}:=
\sup_{f \in \cF} \crl*{ \int fd\bbP - fd\bbQ}$, where $\cF$ is the set
of all 1-Lipschitz functions. Total variation distance upper bounds
1-Wasserstein distance, but the converse is not true.}  It is
therefore natural to ask whether these weaker conditions enable
tractable learning in the \framework framework.

Our next result shows that the weakest of these assumptions---
$Q^\star$ and $V^\star$ Lipschitzness---is not sufficient
for sample-efficient learning with rich observations. Formally, the
assumption is that for all $h,s,s',a,a'$:
\arxiv{\begin{align}
  \begin{aligned}
    \abs*{Q_h^\star(s,a) - Q_h^\star(s',a')} & \leq D\prn*{(s,a),(s',a')}
    \mathand ~~\abs*{V_h^\star(s) - V_h^\star(s')} & \leq D_{\cS}\prn*{s,s'}.
  \end{aligned}\label{eq:qstar_lipschitz}
\end{align}}\icml{\begin{align}
    \begin{aligned}
      \abs*{Q_h^\star(s,a) - Q_h^\star(s',a')} & \leq D\prn*{(s,a),(s',a')} 
      \mathand\\ \abs*{V_h^\star(s) - V_h^\star(s')} & \leq D_{\cS}\prn*{s,s'}.
    \end{aligned}\label{eq:qstar_lipschitz}
  \end{align}}
\begin{theorem}[Lipschitz $Q^\star/V^\star$ lower bound; informal]
  \label{thm:lower}
For rich-observation MDPs satisfying the $Q^\star/V^\star$-Lipschitz latent dynamics assumption
\eqref{eq:qstar_lipschitz}, any algorithm requires
$\wt{\Omega}(\min\crl*{|\Phi|, \abs{\cX}^{1/2}, 2^{\Omega(H)}})$ episodes
to learn an $\veps$-optimal
policy in the worst case for an absolute constant $\veps>0$,
even when $\abs{\cA}=2$ and $\abs*{\cS_\eta}=\bigoh(H)$ for all $\eta\geq{}0$.
\end{theorem}
See~\pref{app:lower_bound} for a formal statement and proof.  The
result shows that sample complexity scaling with $\log|\Phi|$ and
independent of $\abs{\cX}$---the
gold standard for rich-observation RL---is not possible under \arxiv{the
assumption of }$Q^\star/V^\star$-Lipschitz latent
dynamics. 
The mechanism at play is that $Q^\star/V^\star$-Lipschitzness does not
ensure Bellman completeness---a function approximation condition crucial to the analysis of \golf---while TV-Lipschitzness does.
Interestingly, the lower bound does not apply to the
intermediate assumption of latent Wasserstein Lipschitzness, which also does not ensure completeness; 
understanding the statistical complexity of the latter setting is an important open problem.


%% file: section_algorithms.tex
In this section, we turn our focus to algorithm development, and
present efficient algorithms for learning in the \framework
framework. As highlighted in the introduction, the central challenge
for rich-observation RL in the presence of continuous dynamics is to 
develop a representation learning approach that (i) is provably
sample efficient, (ii) captures the dynamics structure of the latent state
space, and (iii) is computationally tractable. We address this
challenge in \cref{sec:rep_learn} by providing a new representation
learning objective, \replearn, then build on this development in
\cref{sec:mainalg} to provide a new algorithm for online exploration
in the \framework framework.

\begin{remark}
  \replearnog was introduced by~\citet{modi2021model} with no name and was further studied by~\citet{zhang2022efficient,mhammedi2023efficient}, who both use the name \textsf{RepLearn}. We introduce the name \replearnog for ``Bellman Consistent Representation Learning'' for this procedure, and call our variation \replearn, as it is designed for continuous dynamics.
\end{remark}

\subsection{Representation Learning with Continuous Latent Dynamics:
  \replearn}
\label{sec:rep_learn}

A principle shared by many prior works on representation learning for RL is that one should capture the information necessary to
represent \emph{Bellman backups} of functions of interest (e.g., value
functions~\citep{zhang2022efficient,mhammedi2023efficient,modi2021model}). For
the \framework framework, we show that the Bellman backup of \emph{any}
bounded function is Lipschitz with respect to the true latent state
$\phi^\ast(x)$. Building on these prior works, we aim to learn a representation such that
Bellman backups of functions of interest can be approximated by
Lipschitz functions of the learned latent state; stated equivalently:
\emph{we learn a representation that respects the Lipschitz structure of the
latent dynamics.} \loose

Concretely, fix time step $h$ and suppose we have a
dataset of $(x_h,a_h,x_{h+1})$ tuples in which $(x_h,a_h)$ are drawn
from a data
distribution $\rho_h \in \Delta(\cX\times\cA)$ and $x_{h+1} \sim P_h(x_h,a_h)$. Let $\Lip\subset \cS \times
\cA \to [0,L]$ denote the set of $L$-bounded functions that are $1$-Lipschitz
with respect to the metric $D$. We aim to learn a decoder
$\phi_h$ that minimizes the following population-level objective:
\begin{align}
  \label{eq:replearn_obj}
    \max_{f \in \cF} \min_{g \in \Lip}
    \En_{\rho_h}\bigg[\big(g(\phi_h(x_h),a_h) -
      \cP_h[f](x_h,a_h)\big)^2\bigg]
\end{align}
where $\cF: \cX \to [0,L]$ is a given class of \emph{discriminators}
whose Bellman backups we would like to approximate and $\cP_h[\cdot]$ is the (reward-free) Bellman backup operator, defined via $\cP_h[f](x_h,a_h) := \En[f(x_{h+1}) \mid x_h,a_h]$. 
The novel twist over prior work \icml{\citep{modi2021model}}\arxiv{\citep{zhang2022efficient,mhammedi2023efficient,modi2021model}} is that by constraining the ``prediction head'' $g \in \Lip$ (which is composed with the representation $\phi_h$),
we ensure that any decoder $\phi$ with low objective value can
approximate Bellman backups via Lipschitz functions of the learned
latent state. This, in turn, constrains the learned latent
  space to respect the continuity structure of the true latent
  dynamics.
We leave $\cF$ and $\rho_h$ as free parameters here but will instantiate them 
concretely in~\cref{sec:mainalg}.

\paragraph{Algorithm and guarantee}
The main challenge in minimizing the population objective
\cref{eq:replearn_obj} from samples is that it involves the conditional expectation
$\En[f(x_{h+1}) \mid x_h, a_h]$, which 
leads to the well-known ``double sampling'' bias. Following~\citet{modi2021model}, 
we introduce a nested inner optimization problem to
de-bias. With this bias correction, we obtain our main representation
learning algorithm: for each $h \in \brk{H}$, given dataset $\cD_h$ of
$(x_h,a_h,x_{h+1})$ tuples, we solve the following min-max-min problem:
\icml{
\begin{align}
  \label{eq:rep_learn}
  \phi_h\!\gets\! \argmin_{\phi_h \in \Phi} \max_{f \in \cF} 
  \crl[\bigg]{\min_{g \in \Lip} \widehat{\ell}(\phi_h, g, f) \!-\! \widehat{\mathsf{opt}}(f)} 
\end{align} 
}
\arxiv{
\begin{align}
  \label{eq:rep_learn}
  \phi_h\!\gets\! \argmin_{\phi_h \in \Phi} \max_{f \in \cF} 
  \crl[\bigg]{\min_{g \in \Lip} \widehat{\ell}_{\cD_h}(\phi_h, g, f) - \min_{\widetilde
  \phi_h \in \Phi, \widetilde g \in \Lip} \ell_{\cD_h}(\widetilde \phi_h, \widetilde g, f)},
\end{align}
}
where\footnote{$\widehat{\En}[\cdot]$ denotes sample average: $\widehat{\En}_{(x)\sim \cD}[f] := \frac{1}{|\cD|}\sum_{(x)\in\cD}f(x)$.}
\icml{
\begin{align*}
  \widehat{\ell}_{\cD_h}(\phi, g, f) &\ldef \widehat \En_{\cD_h} \brk*{ (g(\phi(x_h),a_h) - f(x_{h+1}))^2 },\\
  \mathand ~~ \widehat{\mathsf{opt}}(f) &\ldef \min_{\widetilde
    \phi_h \in \Phi, \widetilde g \in \Lip} \widehat{\ell}(\widetilde \phi_h, \widetilde g, f).
\end{align*}
}
\arxiv{
  \begin{align*}
    \widehat{\ell}_{\cD_h}(\phi, g, f) &\ldef \widehat \En_{\cD_h} \brk*{ (g(\phi(x_h),a_h) - f(x_{h+1}))^2 }.
  \end{align*}
}

\begin{algorithm}[tp] 
  \caption{\replearn: Bellman Consistent Representation Learning with
  Continuous Latent Dynamics}
  \begin{algorithmic}[1] 
    \State \textbf{input:} Layer $h\in\brk{H}$, dataset $\cD_h$ of $\prn{x_h,a_h,x_{h+1}}$ tuples, decoder class $\Phi$, discriminator 
  class $\cF$.
  \arxiv{
  \State Define loss function 
  \begin{align*}
    \ell_{\cD_h}(\phi, g, f) \ldef \widehat \En_{(x_h,a_h,x_{h+1})\sim \cD_h} \brk*{ (g(\phi(x_h,a_h)) - f(x_{h+1}))^2 }.
  \end{align*} }
  \icml{
    \State Define loss function $\ell_{\cD_h}(\phi, g, f) \ldef$
    \begin{align*}
       \widehat \En_{(x_h,a_h,x_{h+1})\sim \cD_h} \brk*{ (g(\phi(x_h,a_h)) - f(x_{h+1}))^2 }.
    \end{align*} 
  }
  \State Solve the \emph{min-max-min} optimization problem:
  \arxiv{
    \begin{align*}
      \phi_h \gets \argmin_{\phi_h \in \Phi} \max_{f \in \cF} 
      \crl[\bigg]{\min_{g \in \Lip}\ell_{\cD_h}(\phi_h, g, f) 
        - \min_{\widetilde
        \phi_h \in \Phi, \widetilde g \in \Lip} \ell_{\cD}(\widetilde \phi_h, \widetilde g, f)}.
    \end{align*}
  }
  \icml{
  \begin{align*}
    \phi_h \gets \argmin_{\phi_h \in \Phi} \max_{f \in \cF} 
    \crl[\bigg]{&\min_{g \in \Lip}\ell_{\cD_h}(\phi_h, g, f) \\
      - &\min_{\widetilde
      \phi_h \in \Phi, \widetilde g \in \Lip} \ell_{\cD}(\widetilde \phi_h, \widetilde g, f)}.
  \end{align*}
  }
  \State \textbf{return} $\phi_h$.
  \end{algorithmic}\label{alg:rep_learn}
  \end{algorithm}

We call this algorithm \replearn, and provide the full pseudocode in~\cref{alg:rep_learn}.
The guarantee for \replearn is stated in terms of a new concept called a \emph{pseudobackup} operator, which we define as
\begin{align}\label{eq:pseudobackup}
  \psb_{\cD_h,\cV}: f \mapsto \argmin_{v \in \cV} \widehat \En_{\cD_h}\brk*{ \prn*{v(x_h,a_h) - f(x_{h+1})}^2 },
 \end{align}
given dataset $\cD_h$ and function class $\cV \subset (\cX \times \cA) \to [0,L]$. 
Informally, this represents the best approximation to $\cP_h[f]$ over the class $\cV$. 
For the remainder of the paper, we choose $\cV \ldef \Lip \circ \phi_h$ for some $\phi_h \in \Phi$, and we condense the notation to $\psb_{\cD_h,\phi_h}$ to make the dependence on the decoder explicit. We now state the main guarantee for \replearn.
\begin{theorem}[Guarantee of \replearn]\label{thm:rep_learn}
  Suppose \cref{ass:lipschitz,assum:decoder_realizability} hold. Fix $h \in [H]$ and $\delta \in (0,1)$, and define
 $\cF_{h+1} = \Lip \circ \Phi:  \cX \to [0,L]$.
Let  $\cD_h$ be a dataset of $N$ \iid tuples $(x_h,a_h,x_{h+1})$ sampled as $(x_h,a_h) \sim \rho_h$ and $x_{h+1} \sim P_h(x_h,a_h)$ where $\rho_h \in \Delta(\cX\times\cA)$. 
  Then with probability at least $1-\delta$, the decoder $\phi_h$ produced by
  \replearn ensures that for all $f \in \cF_{h+1}$
  \begin{align*}
   \En_{\rho_h}\brk*{
     \prn*{\psb_{\cD_h, \phi_h}[f](x,a) - \realb_h[f](x,a)}^2 } \leq \erep(N,\delta),
 \end{align*}
 where
 \begin{align*}
   \erep(N,\delta) \ldef \widetilde{\cO}\prn*{ \frac{L^{2\dimsa} \log\prn*{ \abs{\Phi}/\delta}}{N^{\frac{1}{\dimsa+1}} }}.
 \end{align*}
\end{theorem}
As an immediate consequence, in~\pref{app:replearn_offline}, we show how the
theorem can be applied to offline reinforcement learning, 
yielding guarantees for approximate dynamic programming (ADP)
algorithms (e.g., fitted Q iteration) in the learned latent space.  As
should be expected, the sample complexity bound has a nonparametric flavor,
scaling exponentially with the latent dimensionality, but crucially it
does not depend on the size of the observation space.

Regarding computational complexity, the optimization problem
in~\pref{eq:rep_learn} is amenable to gradient based techniques, but
may be difficult to solve in practice. For the simpler tabular Block MDP
setting, prior work~\citep{zhang2022efficient,mhammedi2023efficient,modi2021model}
has developed an iterative counterpart to \cref{alg:rep_learn}, which obtains a slightly worse statistical
guarantee but is much more computationally viable. In
\cref{app:experiments}, we give a continuous extension of this
iterative scheme (see pseudocode
in~\pref{alg:rep_learn_iter}), which we use in our experiments in~\pref{sec:experiments}.

\paragraph{Other representation learning objectives}
Within the line of theoretical research on representation learning for
RL, a number of other schemes have been developed and analyzed, with a
focus on tabular Block MDPs. The most notable objectives are
contrastive learning~\citep{misra2020kinematic} and multistep inverse
kinematics~\citep{lamb2023guaranteed,mhammedi2023representation}.
Here, we briefly highlight that these schemes do not extend as-is to
accommodate continuous latent dynamics in the \framework framework,
providing further motivation behind our objective \eqref{eq:replearn_obj}. For simplicity, we
discuss these objectives at the population level. 

The contrastive learning objective of~\citet{misra2020kinematic}
involves distinguishing transition tuples that can be generated by the
dynamics from those that cannot. The procedure samples pairs of tuples $(x^1_h,a^1_h,x^1_{h+1})$ and
$(x^2_h,a^2_h,x^2_{h+1})$ \iid where $x_h,a_h \sim \rho_h$ and
$x_{h+1} \sim P_h(x_h,a_h)$ for a data collection distribution $\rho_h \in
\Delta(\cX\times\cA)$. Then, it 
assigns the ``real transition'' $(x^1_h,a^1_h,x^1_{h+1})$ a positive
label and the ``fake transition'' $(x^1_h,a^1_h,x^2_{h+1})$ a negative
label. Finally, a classifier of the form $f(\phi_h(x_h), a_h,
\phi_{h+1}(x_{h+1}))$ is trained to minimize classification error on
this dataset. One can show that the optimal population-level
classifier takes the form
\begin{align}
  \label{eq:homer}
\frac{P_h(\phi_{h+1}^\star(x_{h+1}) \mid \phi_h^\star(x_h),a_h)}{P_h(\phi_{h+1}^\star(x_{h+1}) \mid \phi_h^\star(x_h),a_h) + \widetilde{\rho}_h(\phi_{h+1}^\star(x_{h+1}))},
\end{align}
where $\widetilde{\rho}_h(\cdot)$ is the marginal distribution over
$x_{h+1}$ under the data collection process. The key property of this optimal
classifier---used by~\citet{misra2020kinematic}---is that it only depends
on the latent states, and does not directly depend on the observations
themselves. This is also true in the \framework framework; unfortunately, even though the latent dynamics
are Lipschitz, the optimal classifier in \cref{eq:homer} is not
guaranteed to be a Lipschitz function of the latent state. This is
problematic, as we need to construct a low-complexity function class
that contains the predictor in \cref{eq:homer} in order to provide
sample-efficient learning guarantees.

The multistep inverse kinematics approach
\citep{lamb2023guaranteed,mhammedi2023representation} involves predicting the
action $a_h\sim\unif(\cA)$ at time $h$ from the observation $x_h$ and a future
observation $x_t$ for fixed $t>h$ under a roll-out policy $\pi$. When
the action space is finite, one can show that the optimal
population-level objective takes the form
\begin{align}
  (x_h,a_h,x_{t}) \mapsto \frac{\bbP^{\pi}(\phi^\star(x_{t}) \mid
  \phi^\star(x_h),a_h)}{\sum_{a\in\cA}\bbP^{\pi}(\phi^\star(x_{t}) \mid
  \phi^\star(x_h),a)}.
  \label{eq:musik}
\end{align}
Similar to contrastive learning, the key property of this objective
is that it depends on the observation only through the corresponding
latent state, a central property used
by~\citet{mhammedi2023representation}. However, analogously to
contrastive learning, this property alone is not sufficient for
sample-efficient learning in the \framework, because we need to construct a low-complexity function
class to express the optimal predictor \eqref{eq:musik} in this latent
space. Unfortunately, the optimal predictor for multistep inverse
kinematics objective may not be a Lipschitz function of the latent
state, even when the transition dynamics themselves are Lipschitz.
\begin{proposition}[Informal]\label{prop:homer}
  In the \framework framework, the optimal population-level regression function for contrastive
  learning and multistep inverse kinematics may not be Lipschitz with
  respect to the metric on latent state-action space. 
  \end{proposition}  
See \cref{app:homer} for a formal statement and proof.

\subsection{Efficient Exploration with Continuous Latent Dynamics:
  \mainalg}
\label{sec:mainalg}

\begin{algorithm}[t] 
  \caption{\mainalg: Continuous Representation Learning with
    Interleaved Explore-Exploit}
  \begin{algorithmic}[1] 
\State \textbf{input:} Decoders $\Phi$, iterations
  $T$, discretization scale $\eta$,
\icml{\Statex \hspace{-0.65cm} }parameters
  $\crl*{\lambda^t}_{t\in\brk{T}}$, $\crl*{\wh{\alpha}^t}_{t\in\brk{T}}$.
  \State Initialize datasets $\cD^0_{1,h} \gets \emptyset$,
  $\cD^0_{2,h} \gets \emptyset$ for all $h\in\brk{H}$.
  \State Initialize $\pi^0 = \{\pi_h\}_{h=1}^H$ arbitrarily.
  \For{$t=1, \dots, T$}
  \Statex[1] \algcommentlight{Gather data.}
  \State \multiline{
For $h\in\brk{H}$,
    gather tuples:
    { \icml{\small}
  \begin{align*}
    &(x_h,a_h,x_{h+1}) \sim \pi^{t-1} \circ_h \unifpi,\\
    &(x_h,a_h,x_{h+1}) \sim \pi^{t-1} \circ_{h-1} \unifpi.
  \end{align*}}%
Add the first to $\cD^t_{1,h}$ and the second to $\cD^t_{2,h}$. Let $\cD^t_h \gets \cD^t_{1,h} \cup \cD^t_{2,h}$.} \label{line:data_collection}
  \Statex[1] \algcommentlight{Learn representation.}\vspace{2pt}
  \State \mbox{Call 
  \icml{\alghyperref{alg:rep_learn}
  (or \alghyperref{alg:rep_learn_iter})}
  \arxiv{\pref{alg:rep_learn} or \pref{alg:rep_learn_iter}}
  with $\cF_{h+1}$ in \pref{eq:discriminator_class} \arxiv{from \pref{sec:golf_proof_pre}}:} \label{line:rep_learn}
 { \icml{\small}
  \begin{align*}
    \phi^{t}_h \gets \replearn_h(\cD^t_h, \Phi, \cF_{h+1}),\; \forall h \in [H].
  \end{align*}}%
  \State \label{line:bonus} Define exploration bonus (cf. \pref{eq:count}):
  {\icml{\small} 
  \begin{equation*}\label{eq:empirical_bonus}
    ~~~~~~~\widehat b_h^t(x,a)\! \ldef\! \min\crl*{\widehat \alpha^t \sqrt{\frac{1}{N_{\eta,\phi^t_h}(x,a,\cD^t_{1,h}) + \lambda^t}}, 2}.
  \end{equation*}}%
  \Statex[1] \algcommentlight{Learn policy.}
  \State \label{line:optdp} 
  \multiline{Call 
  \icml{\alghyperref{alg:optdp}} \arxiv{\pref{alg:optdp}} with estimated decoder and bonuses:
  {\icml{\small}
  \begin{align*}
    &\pi^t \gets \mathsf{OptDP}\prn[\big]{\{\cQ_h\}_{h=1}^H,
    \{\cD_{h}\}_{h=1}^H, \{\widehat
    b_h^t\}_{h=1}^H,\eta},\\
    &\text{with }\cQ_h := 
    \{ (x,a) \mapsto w^\top \disc{\eta}[\phi_h^t](x,a) : \|w \|_{\infty}
    \leq 2\}.
  \end{align*}}
}
  \EndFor
  \State \textbf{return:} $\widehat{\pi} \ldef \frac{1}{T}\sum_{t=1}^T \pi^t$.
  \end{algorithmic}\label{alg:main_alg}
  \end{algorithm}

  \begin{algorithm}[tp] 
    \caption{\optdp: Optimistic Dynamic Programming}
    \begin{algorithmic}[1] 
    \State \textbf{input:} Function class $\{\cQ_h\}_{h=1}^H$, dataset
    $\{\cD_h\}_{h=1}^H$, exploration bonus $\{\bhat_h\}_{h=1}^H$,
    discretization scale $\eta>0$.
    \State Set $V_{H+1}(x) = 0,\; \forall x\in\cX$.  
    \For{$h=H,\dots,1$}
    \Statex[1] \algcommentlight{Recursively update value functions via
      pseudobackups (cf. \cref{eq:pseudobackup}).}
    \State Learn value functions with pseudobackups:
    \begin{align*}
      &q_h \ldef \psb_{\cD_h,\cQ_h}[V_{h+1}],\\
      &Q_{h}(x,a) \ldef R_h(x,a) + \widehat b_h(x,a) + q_h(x,a),\\
      &V_h(x) \ldef \max_{a \in \cA_\eta} Q_h(x,a),\\
      &\pi_h(x) \ldef \argmax_{a \in \cA_\eta} Q_h(x,a). 
    \end{align*}
    \EndFor
    \State \textbf{return:} $\{\pi_h\}_{h=1}^H$.
    \end{algorithmic}\label{alg:optdp}
    \end{algorithm}

Equipped with \replearn, we can turn our attention to online
reinforcement learning. Our high-level approach, also used in prior
work
\citep{misra2020kinematic,zhang2022efficient,mhammedi2023representation},
is to interleave representation learning and exploration: we
iteratively learn a new decoder based on data the algorithm has
gathered, then use this decoder within an exploration scheme to
acquire new information. In what follows, we describe the algorithm in
detail and present theoretical guarantees.\loose

\paragraph{Notation for discretization} 
Our exploration scheme is based on \emph{discretization} of the
learned latent state space, which requires additional notation.
Recall that for any $\eta > 0$, we have covering sets $\cS_\eta$ and
$\cA_\eta$ for the latent state and action spaces, respectively,
and define $\piunif_\eta(x)=\unif(\cA_\eta)$ as the policy that
  uniformly explores over the cover.
Given
a pair $(s,a) \in \cS \times \cA$, we define $\disc{\eta}(s,a) \in
\cS_\eta\times \cA_\eta$ as any covering element for $(s,a)$ in
$\cS_\eta\times\cA_\eta$ such that $D((s,a),\disc{\eta}(s,a))
\leq \eta$.\footnote{If $(s,a)$ is covered by more than one element in
$\cS_\eta\times\cA_\eta$, we break ties in an arbitrary but consistent
fashion.}  Next, we define $\ball{\eta}(s,a) := \crl*{
  (\tilde{s},\tilde{a}) \in \cS\times\cA : \disc{\eta}(s,a) =
  \disc{\eta}(\tilde{s},\tilde{a})}$ as the ``ball'' of $(s,a)$ pairs
that map to the same covering element. Finally, we define $\cB_\eta :=
\crl*{\ball{\eta}(s_\eta,a_\eta) : s_\eta, a_\eta \in \cS_\eta
  \times\cA_\eta}$; note that the definitions ensure that
$\cB_\eta$ is a partition of $\cS\times\cA$.

It is also useful to lift these definitions to observation
space. In particular, for any decoder $\phi \in \Phi$, we use the notation
$\disc{\eta}[\phi](x,a) := \disc{\eta}(\phi(x),a)$ and
$\ball{\eta}[\phi](x,a) := \crl*{
  (\tilde{x},\tilde{a})\in\cX\times\cA: \disc{\eta}[\phi](x,a) =
  \disc{\eta}[\phi](\tilde{x},\tilde{a}) }$.

\paragraph{Algorithm and guarantee} 
Our main algorithm, which we call \mainalg (Continuous Representation Learning with
    Interleaved Explore-Exploit), is
    displayed in~\pref{alg:main_alg}. Our algorithm builds upon the
    work of \citet{zhang2022efficient} in the context of tabular Block
    MDPs, and interleaves representation learning, via \replearn, with an
    exploration scheme based on optimistic dynamic programming.

    The algorithm proceeds for $T$ iterations. For
    each iteration $t$, after collecting data in
    \cref{line:data_collection}, we learn decoders
    $\{\phi_h^t\}_{h=1}^H$ by applying \replearn (\cref{alg:rep_learn})
    to all of the data gathered so far, with a particular choice of
    discriminator class defined in \pref{eq:discriminator_class} of \cref{sec:mainalg_proof}.
    The decoders are then used to define (i) an exploration
bonus (\cref{line:bonus}) and (ii) a
function class $\cQ$, with which we perform an optimistic form of
approximate dynamic programming (\cref{line:optdp}).

In more detail, the exploration bonus is a count-based bonus over balls \emph{in the
learned latent space}. Formally, for a decoder $\phi$ and dataset $\cD$, we define
\begin{align}\label{eq:count}
  N_{\eta,\phi}(x,a,\cD)\! :=\!\!\!\! \sum_{\widetilde{x},\widetilde{a} \in \cD}\!\!\!\! \indic\crl[\big]{(\widetilde{x},\widetilde{a}) \in \ball{\eta}[\phi](x,a)},
\end{align}
which counts the number of times we have visited each ball in the
learned latent space.  We define the exploration bonus $\widehat
b_h^t(x,a)$ in~\cref{line:bonus} to incentivize the agent to
visit balls with low counts, and then invoke optimistic dynamic
programming (\optdp; \cref{alg:optdp}).
\optdp is a standard approach, but we apply it with a function class derived
from discretizing the learned latent state---specifically
$\cQ=\crl*{(x,a) \mapsto w^\top \disc{\eta}[\phi_h^t](x,a)}$, where
$\| w\|_{\infty} \leq 2$---for fitting Bellman backups, which results
in a policy that we deploy in the next iteration. The main guarantee for \mainalg is as follows.
\begin{theorem}[PAC guarantee for \mainalg;
  informal] \label{thm:main_alg}
  Suppose \cref{ass:lipschitz,assum:decoder_realizability} hold.
  For any $\veps,\delta\in(0,1)$, with an appropriate choice of parameters, 
  \mainalg outputs a policy $\widehat{\pi}$ such that $J(\pi^\ast) -
  J(\widehat{\pi}) \leq \epsilon$ with probability at least
  $1-\delta$, and with sample complexity
  \begin{align*}
    \cO\prn*{\frac{H^{\cO(\dimsa^2)} \log \prn*{TH|\Phi| / \delta \epsilon}} 
  {\epsilon^{\cO(\dimsa^2)}}}.
  \end{align*}
\end{theorem}
\mainalg achieves a similar statistical guarantee to the one
in~\pref{thm:golf}, however the constants in the exponents on $\epsilon$ and $H$ are
somewhat worse. On the other hand, \mainalg is much more
computationally viable than \golf, because \replearn---the main
computational bottleneck---admits a practical implementation.

\paragraph{Analysis and technical challenges}
The main technical challenge in
\mainalg---beyond the design and analysis of \replearn---arises from
the continuity of the learned representation, which precludes
us from analyzing \mainalg using the ``implicit'' latent model approach of~\citet{zhang2022efficient}.
Instead, we rely on the pseudobackup operators defined in
\cref{eq:pseudobackup}, but a key challenge is that these operators do
not inherit standard properties of Bellman backups.  In particular,
the pseudobackup $\psb_{\cD_h,\phi_h}$ is not guaranteed to be a
linear operator (i.e., $\psb_{\cD_h,\phi_h}[f+g] \ne
\psb_{\cD_h,\phi_h}[f] + \psb_{\cD_h,\phi_h}[g]$), and it is not
guaranteed to be ``monotone'' (i.e. $f \leq g$ pointwise does not
imply $\psb_{\cD_h,\phi_h}[f] \leq \psb_{\cD_h,\phi_h}[g]$
pointwise). Without these properties, we cannot naively appeal to
standard techniques used in the analysis of optimistic algorithms; for
example, we cannot apply the simulation lemma, performance difference
lemma \citep{kakade2003sample}, or telescoping decomposition
\citep{jiang2017contextual} for regret.\loose

We resolve this issue algorithmically, by discretizing the learned representation in
the optimistic DP
phase of the algorithm (via the
definition of the bonuses in \cref{line:bonus} and the class $\cQ$ in
\cref{line:optdp}). This allows
us to define \emph{linear pseudobackups}, and, by a careful
analysis, we can show that these operators satisfy linearity
and monotonicity, addressing the above issue. This
solution introduces discretization error
between the continuous representation obtained from \replearn and the
one used for planning, but by carefully tracking these errors we obtain~\pref{thm:main_alg}.


%% file: section_experiments.tex
We now present proof-of-concept empirical validations for
our algorithms in \cref{sec:algorithms}, with a focus on the
representation learning component \replearn.
We consider a maze
environment~\citep{koul2023pclast} and a locomotion
benchmark~\citep{lu2022challenges}, both with
visual (rich) observations. We study the following questions: (1) Do the
representations from \replearn respect the structure (e.g.,
Lipschitzness) of the 
latent dynamics? (2) Are they useful for downstream reward optimization? and
(3) Is the Lipschitz constraint for the inner minimization in~\cref{eq:rep_learn} essential?
We prioritize studying representation learning in isolation (as opposed to studying it in tandem with exploration) because
in practice, the \replearn subroutine can be composed with a variety of
RL methods (as we show below), so
we expect it may be more broadly useful than the full \mainalg
algorithm.\loose

\paragraph{Implementation}
We implement the iterative version of \replearn mentioned in
\cref{sec:rep_learn}; full pseudocode is displayed
in~\pref{alg:rep_learn_iter} in \cref{app:experiments}. At a high
level, the algorithm learns a
representation (decoder) by growing a set of
discriminators over multiple iterations, where in each iteration we
(i) find a
decoder that can approximate the Bellman backups of the current
discriminator set, (ii) find a new discriminator that witnesses
large error for the current decoder (if possible), and (iii) add this to the
discriminator set; this scheme approximates the idealized
``min-max-min'' objective in \cref{eq:rep_learn}. We use deep neural networks to parameterize the decoders $\phi\in\Phi$, the
discriminators $f\in\cF$, and the prediction heads $g\in\Lip$; architecture details
are given in~\pref{app:experiments}. We ensure Lipschitzness of the
prediction heads using spectral
normalization~\citep{miyato2018spectral}, which rescales all
weight matrices in the neural network to have spectral norm $1$ after
each update. We train one decoder for all time steps since our experimental domains admit stationary dynamics. See~\pref{app:experiments} for hyperparameter settings and additional details.

\subsection{Maze Environment}

\icml{
\begin{figure}[t]
    \centering
    \setlength{\tabcolsep}{2pt}
    \renewcommand{\arraystretch}{0.8}
    \begin{tabular}{cccc}
     & \textsf{Hallway} & \textsf{Spiral}  & \textsf{Room} \\
     \raisebox{0.35in}{\rotatebox[origin=t]{90}{\replearn}} & 
     \includegraphics[width=0.28\linewidth]{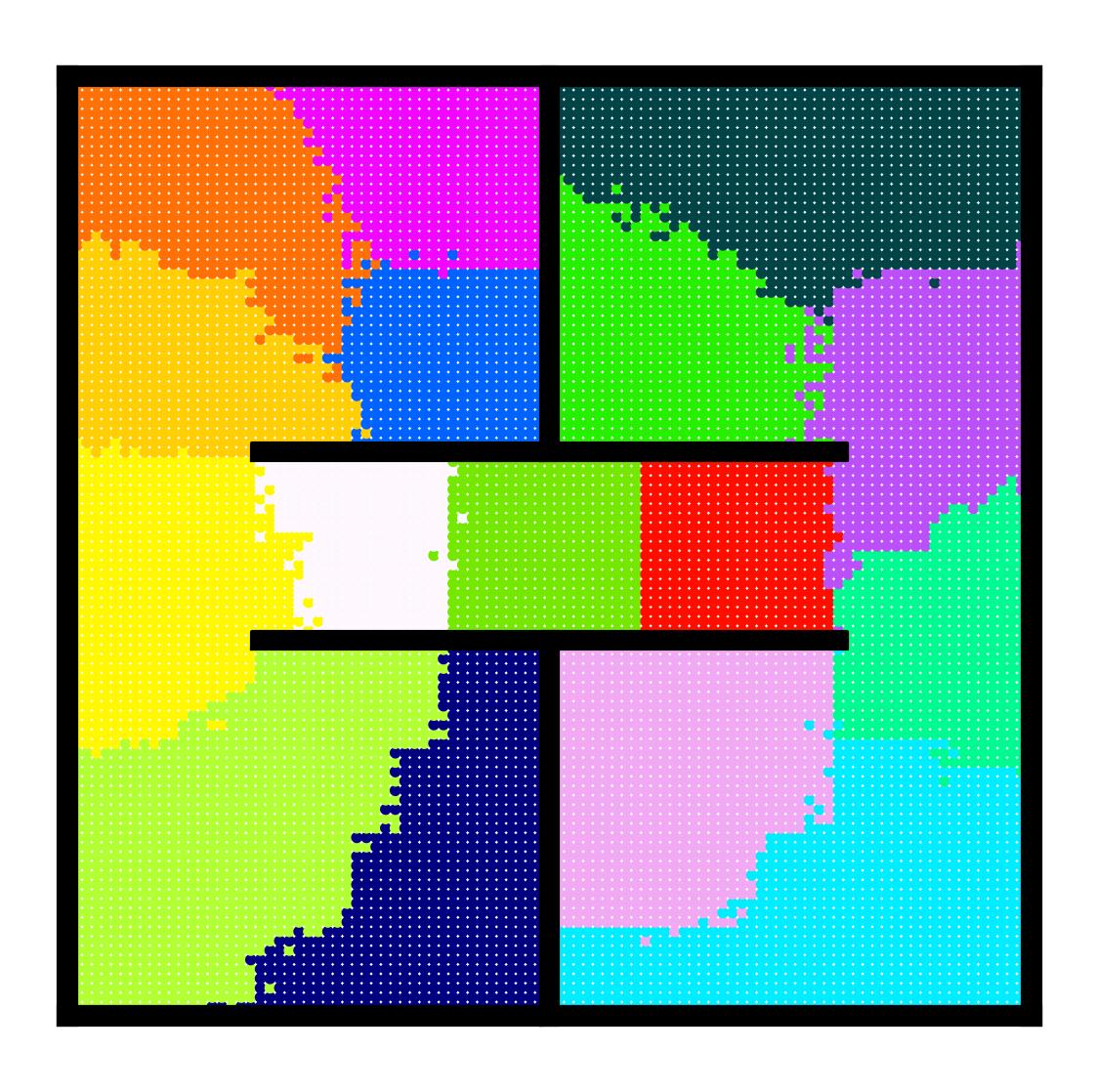} &
     \includegraphics[width=0.28\linewidth]{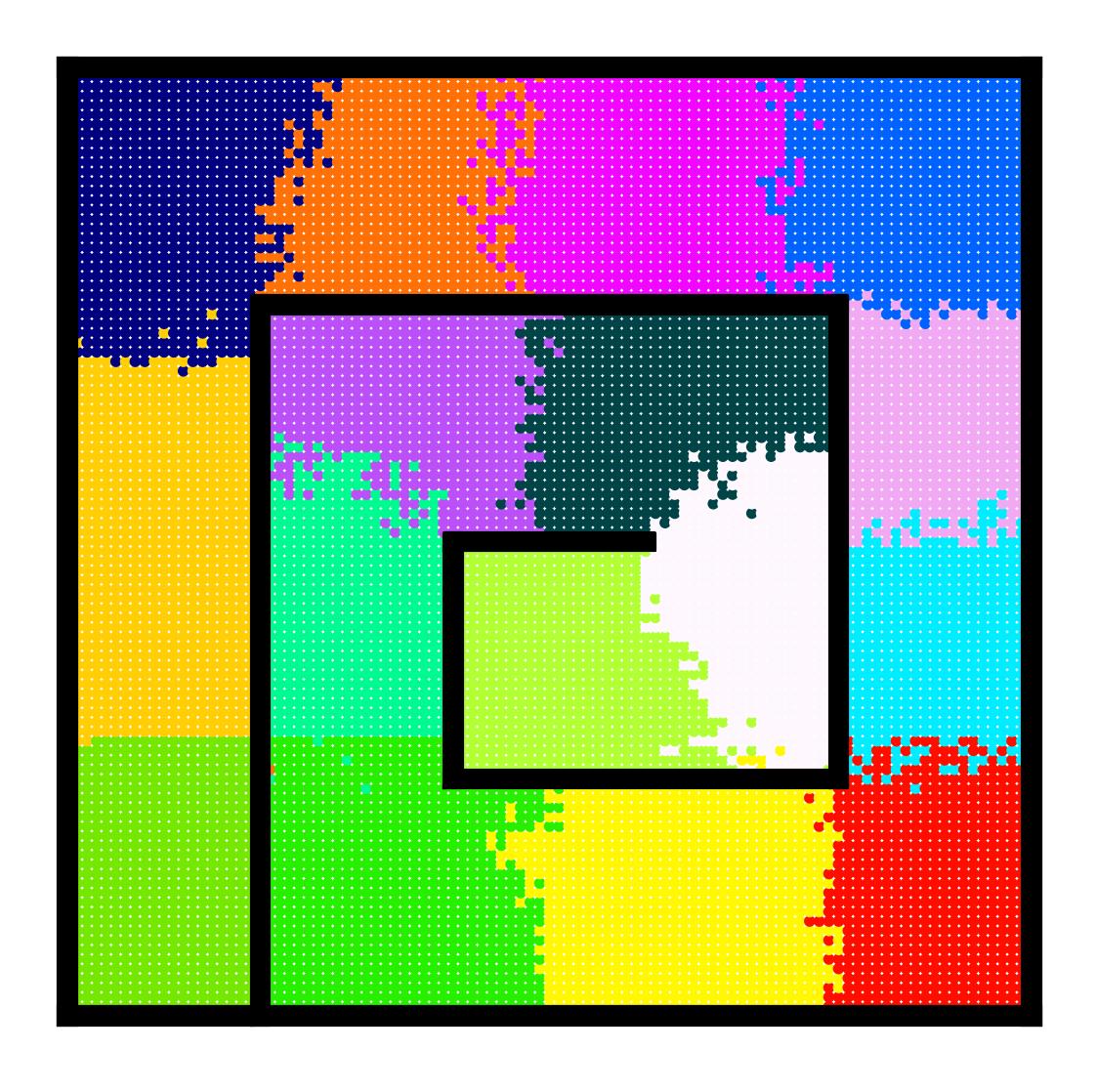} &
     \includegraphics[width=0.28\linewidth]{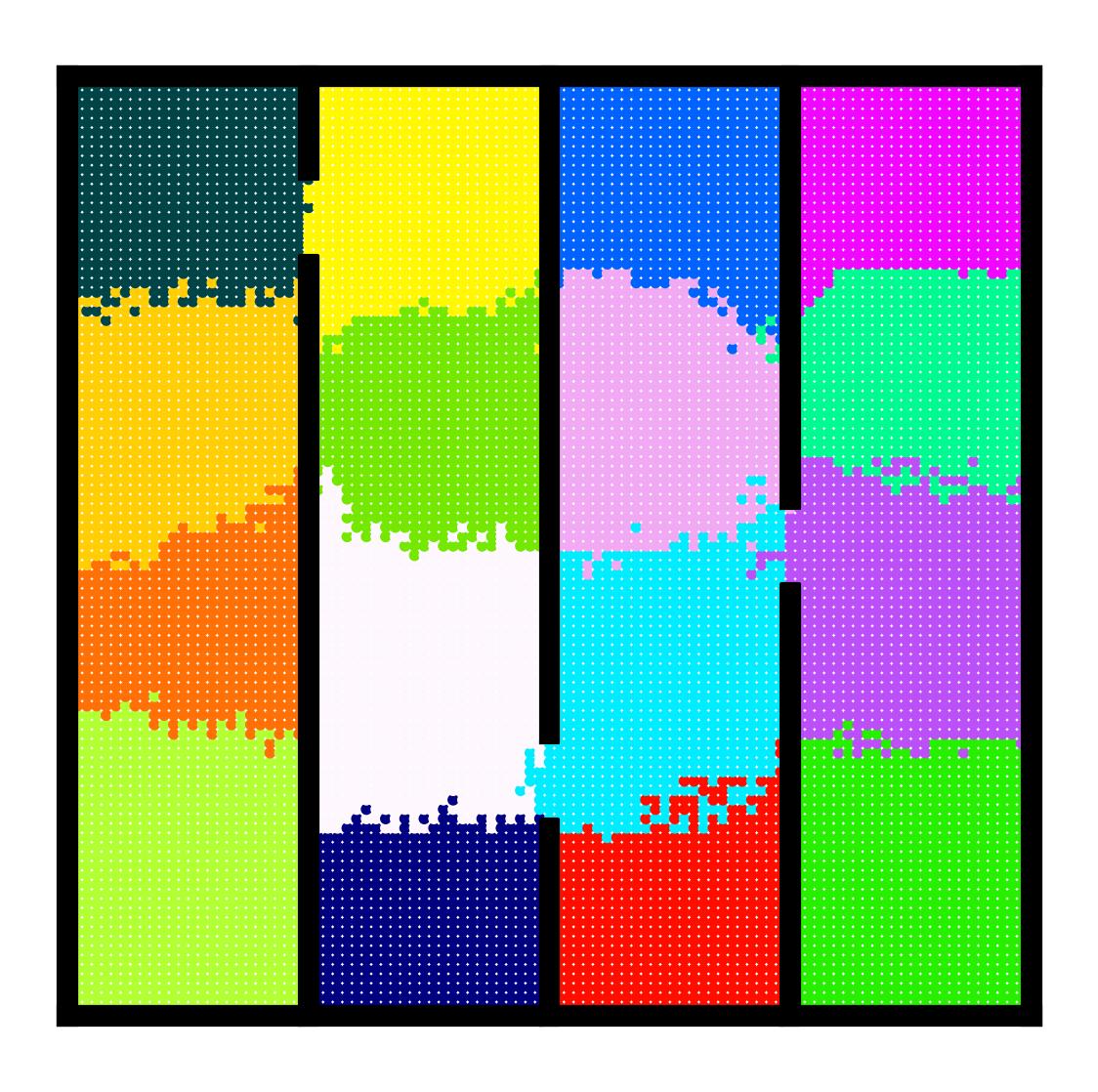} \\
     \raisebox{0.2in}{\rotatebox{90}{\textsf{Baseline}}} & 
     \includegraphics[width=0.28\linewidth]{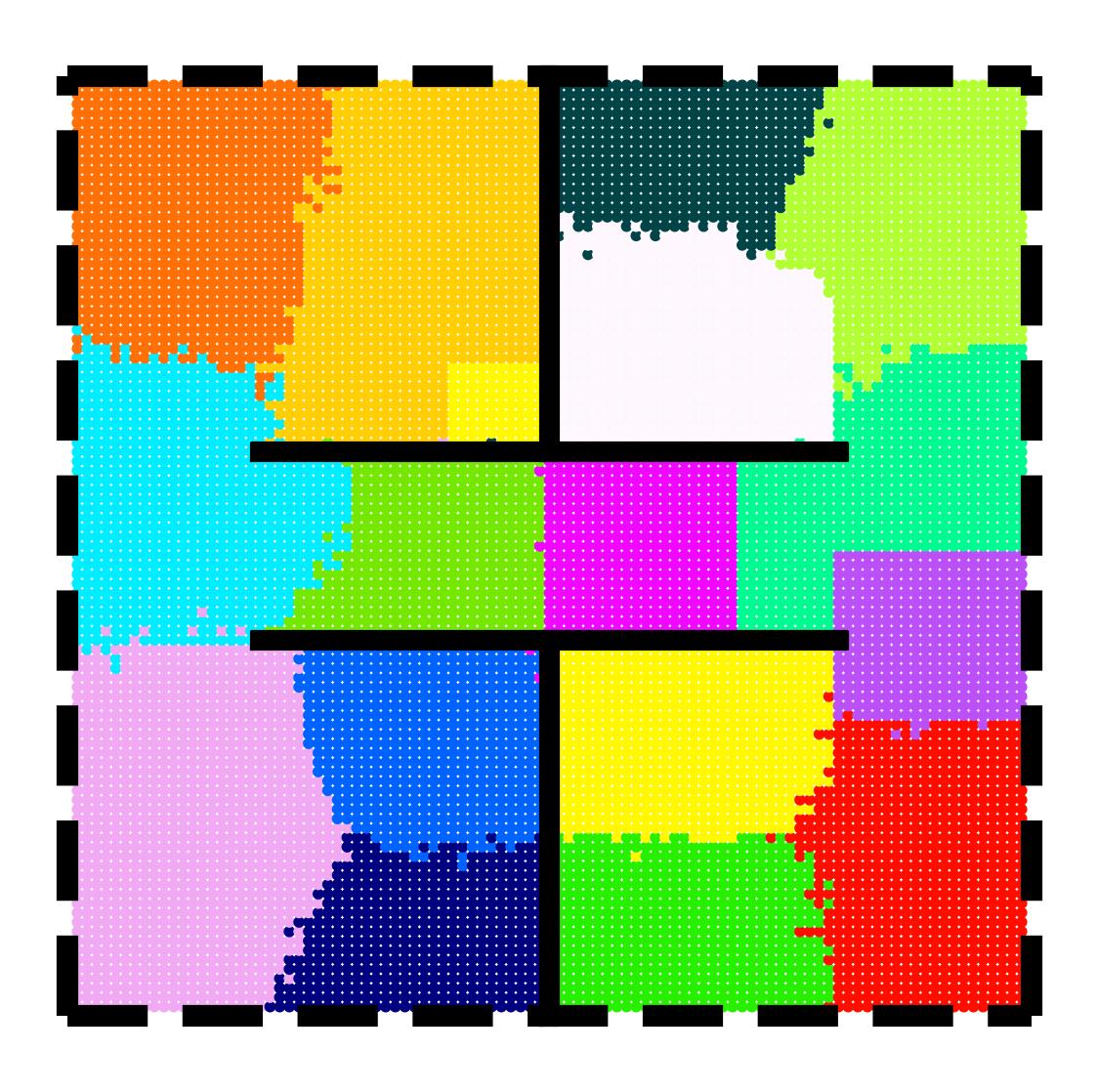} &
     \includegraphics[width=0.28\linewidth]{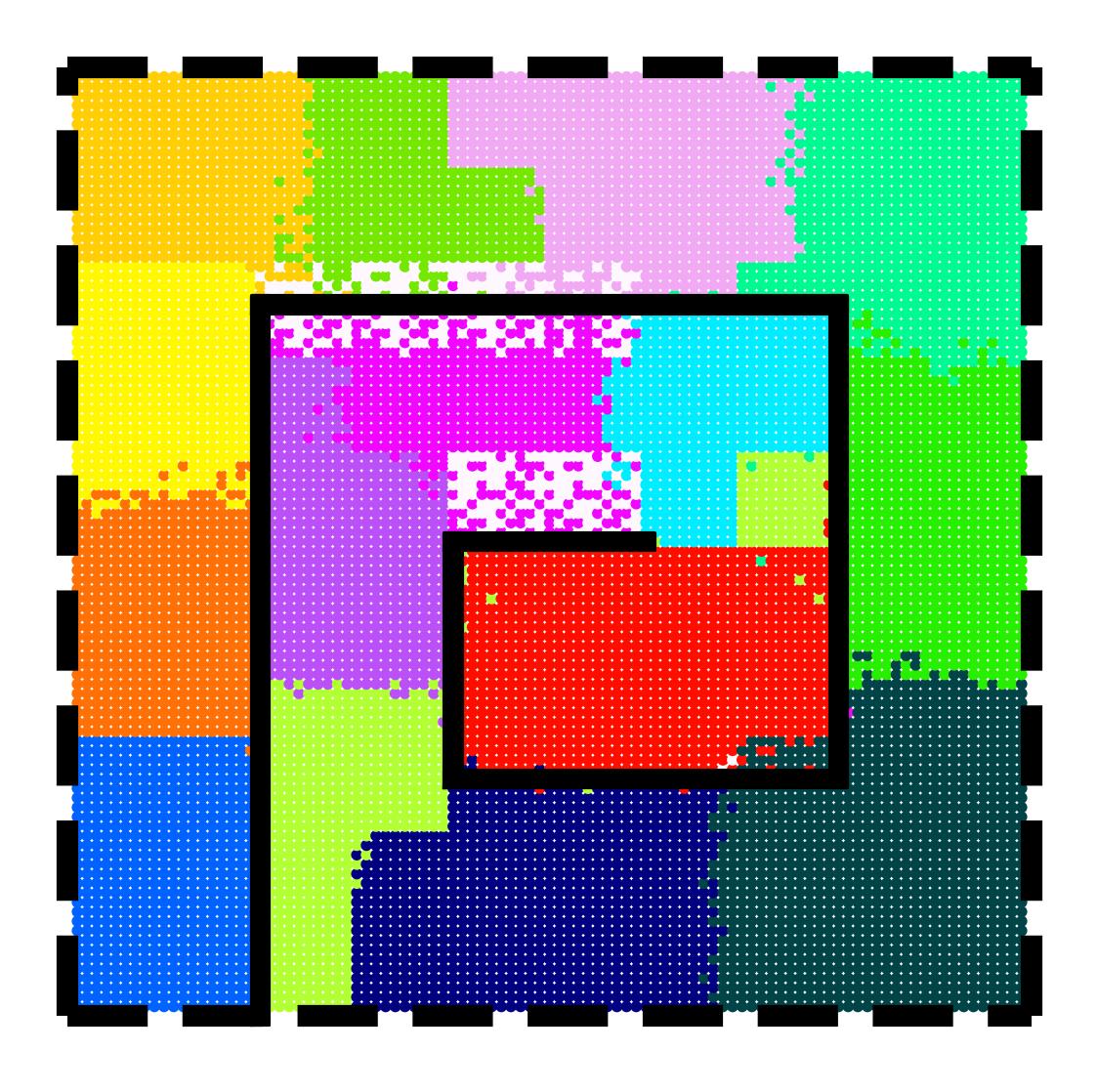} & 
     \includegraphics[width=0.28\linewidth]{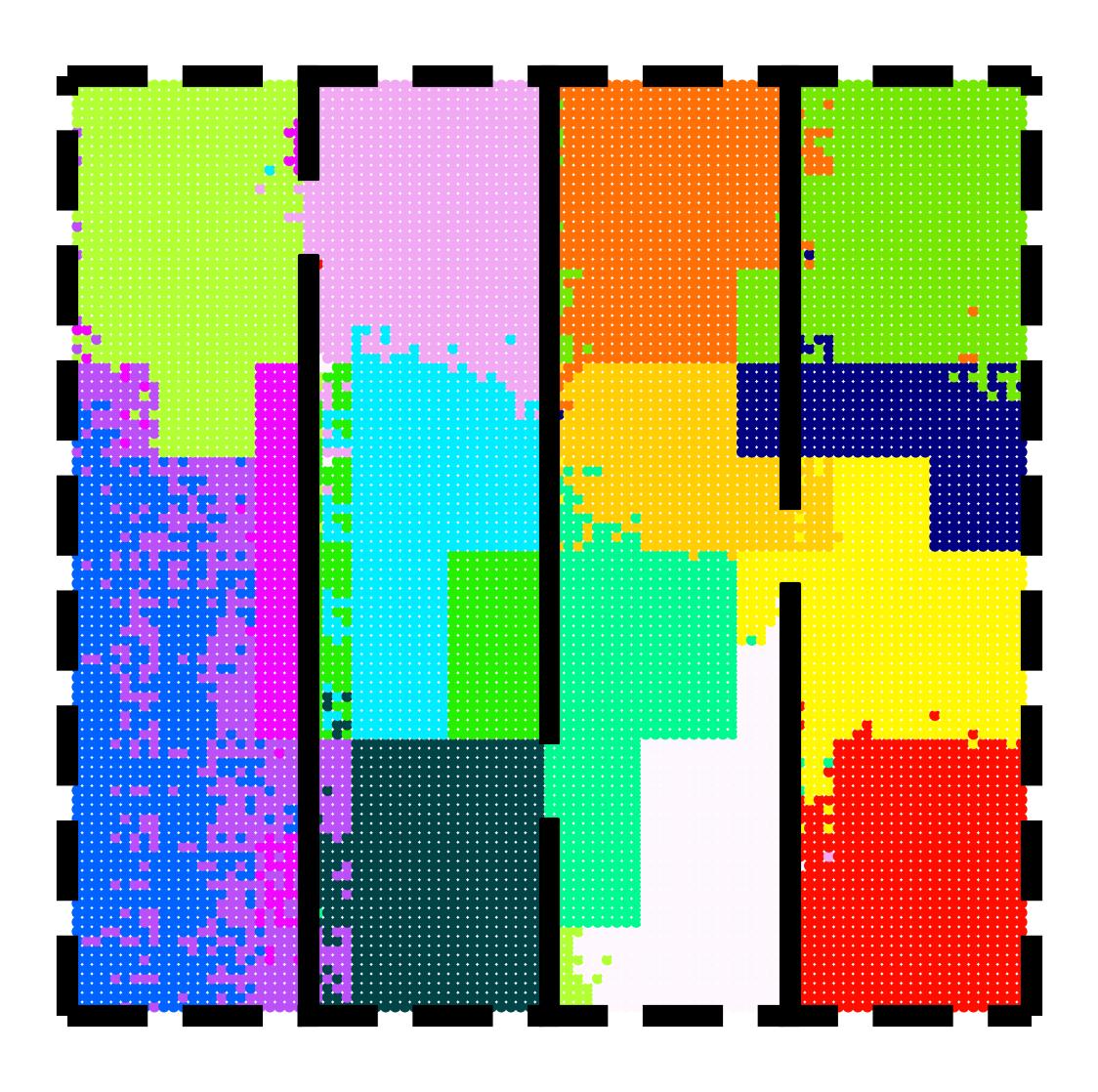} 
    \end{tabular}
    \caption{Results for \replearn on maze environments: $K$-means clusterings ($K=16$) visualize the learned latent space. 
      \textsf{Baseline} does not enforce Lipschitzness of the $g$
      functions in~\cref{eq:rep_learn}.
      \replearn learns to respect walls and boundaries; \textsf{Baseline} does not.
    }%
    \label{fig:kmeans}%
    \icml{\vspace{-0.6cm}}
  \end{figure}
}

\arxiv{
\begin{figure}[t]
  \centering
  \begin{tabular}{cccc}
   & \textsf{Hallway} & \textsf{Spiral}  & \textsf{Room} \\
   \raisebox{0.45in}{\rotatebox[origin=t]{90}{\replearn}} & 
   \includegraphics[width=0.18\linewidth]{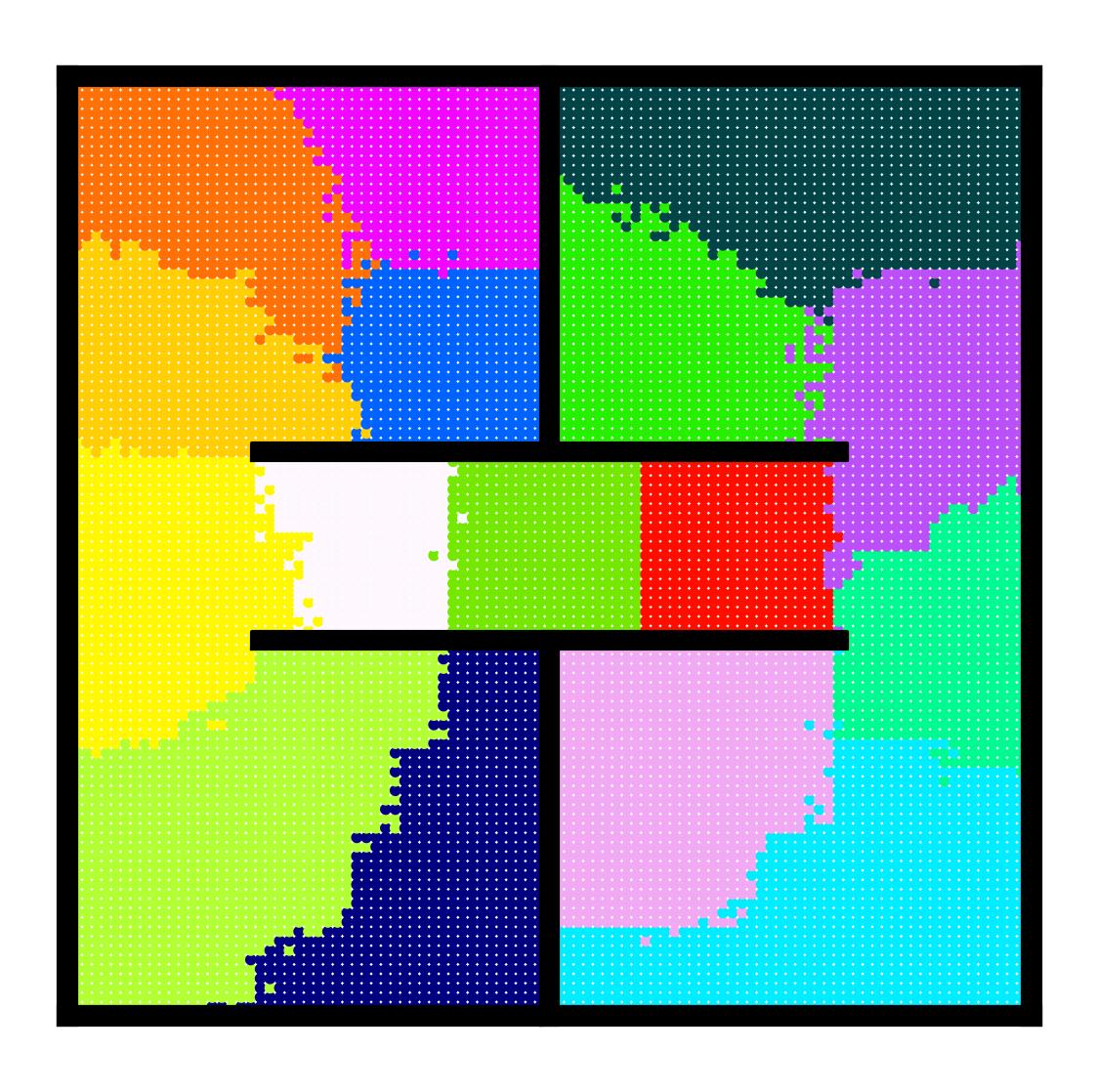} &
   \includegraphics[width=0.18\linewidth]{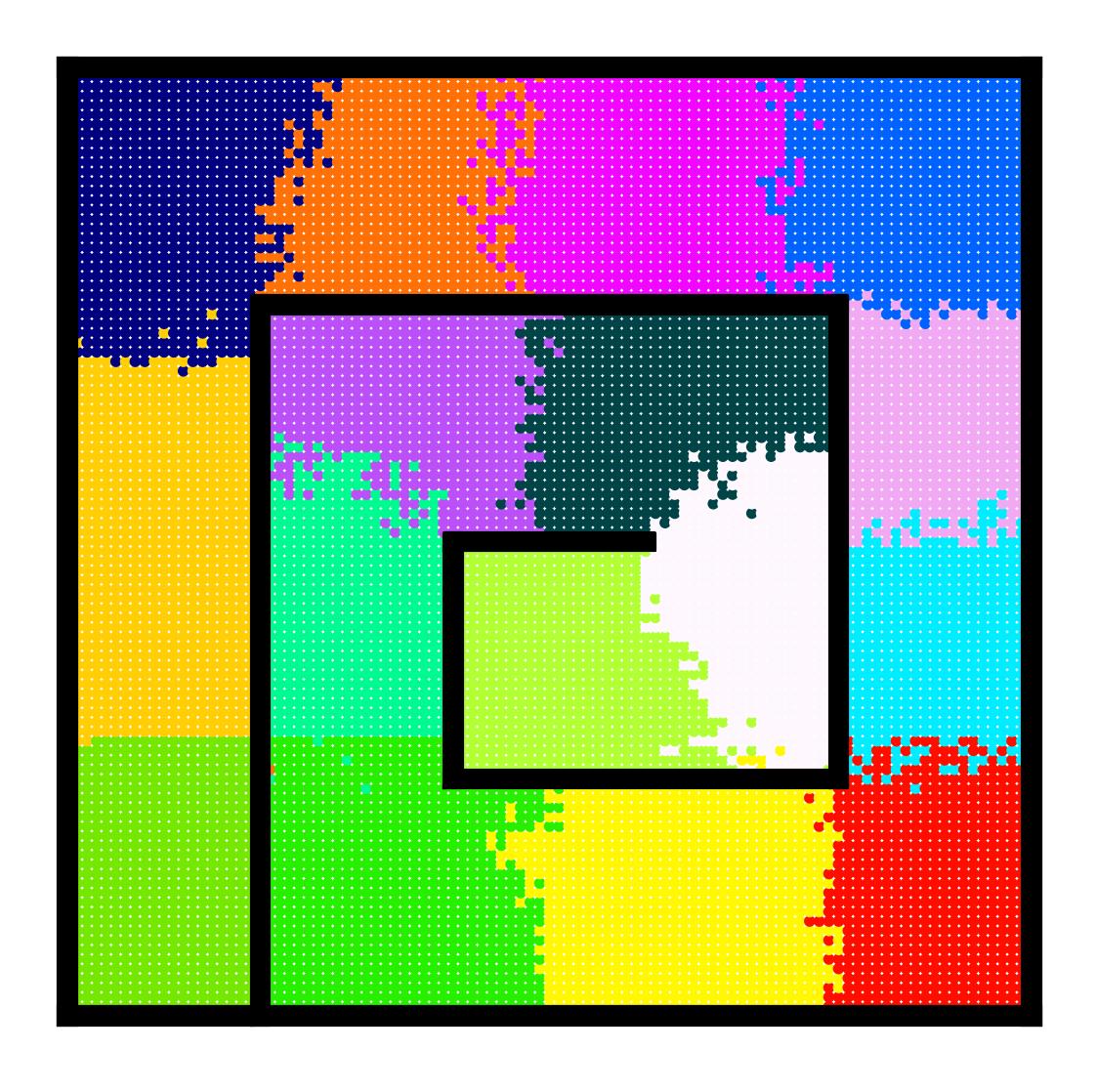} &
   \includegraphics[width=0.18\linewidth]{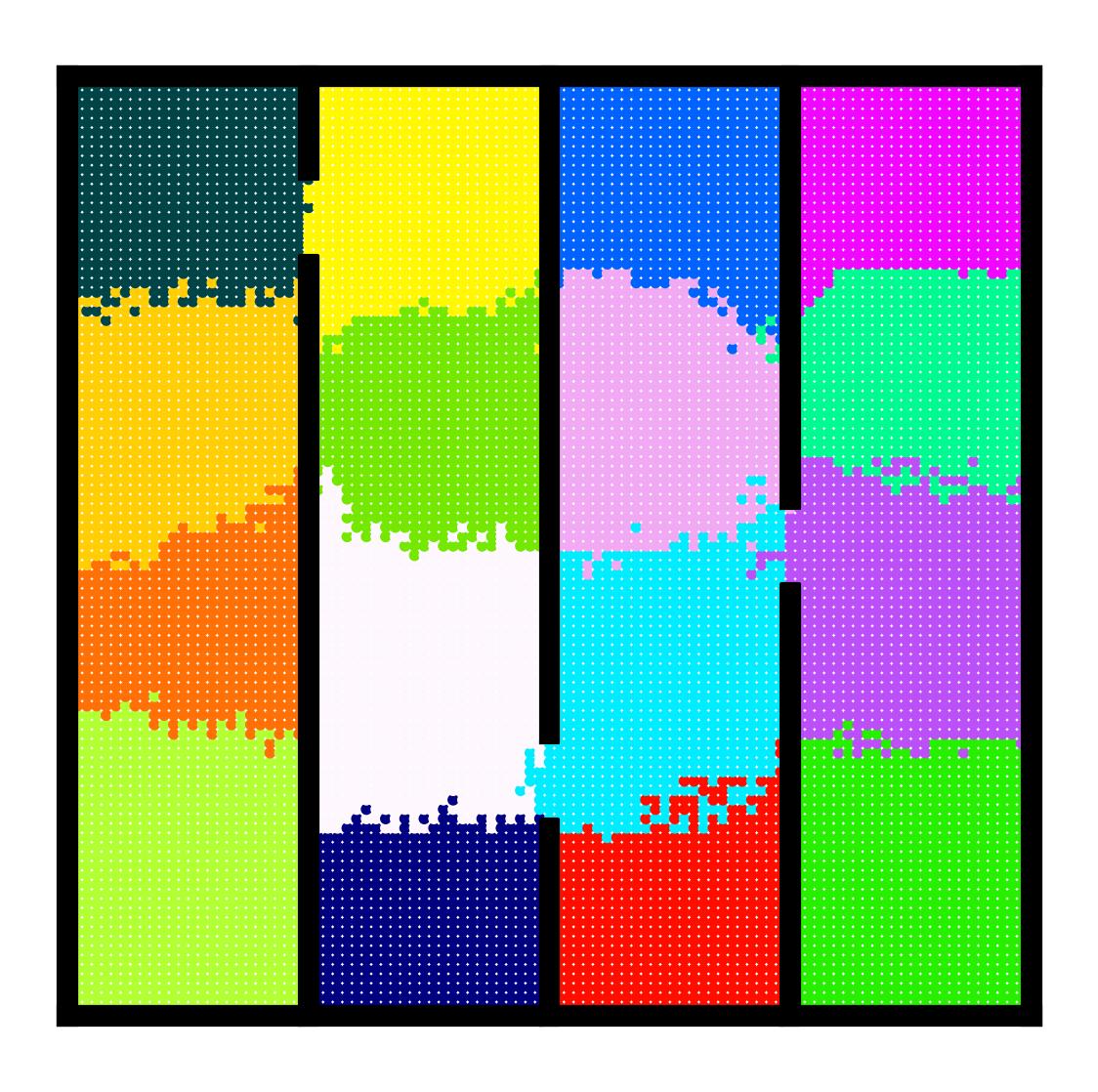} \\
   \raisebox{0.3in}{\rotatebox{90}{\textsf{Baseline}}} & 
   \includegraphics[width=0.18\linewidth]{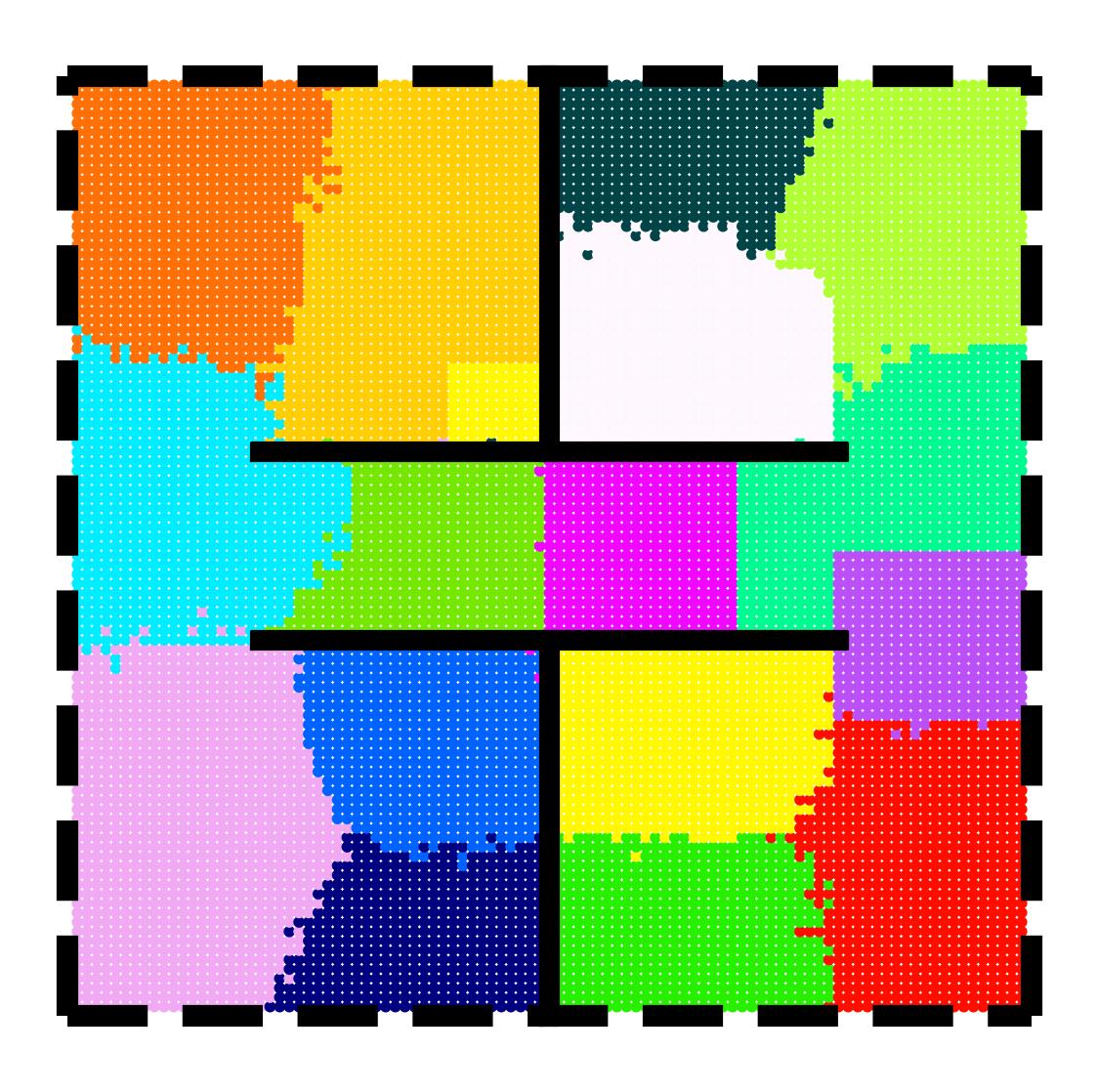} &
   \includegraphics[width=0.18\linewidth]{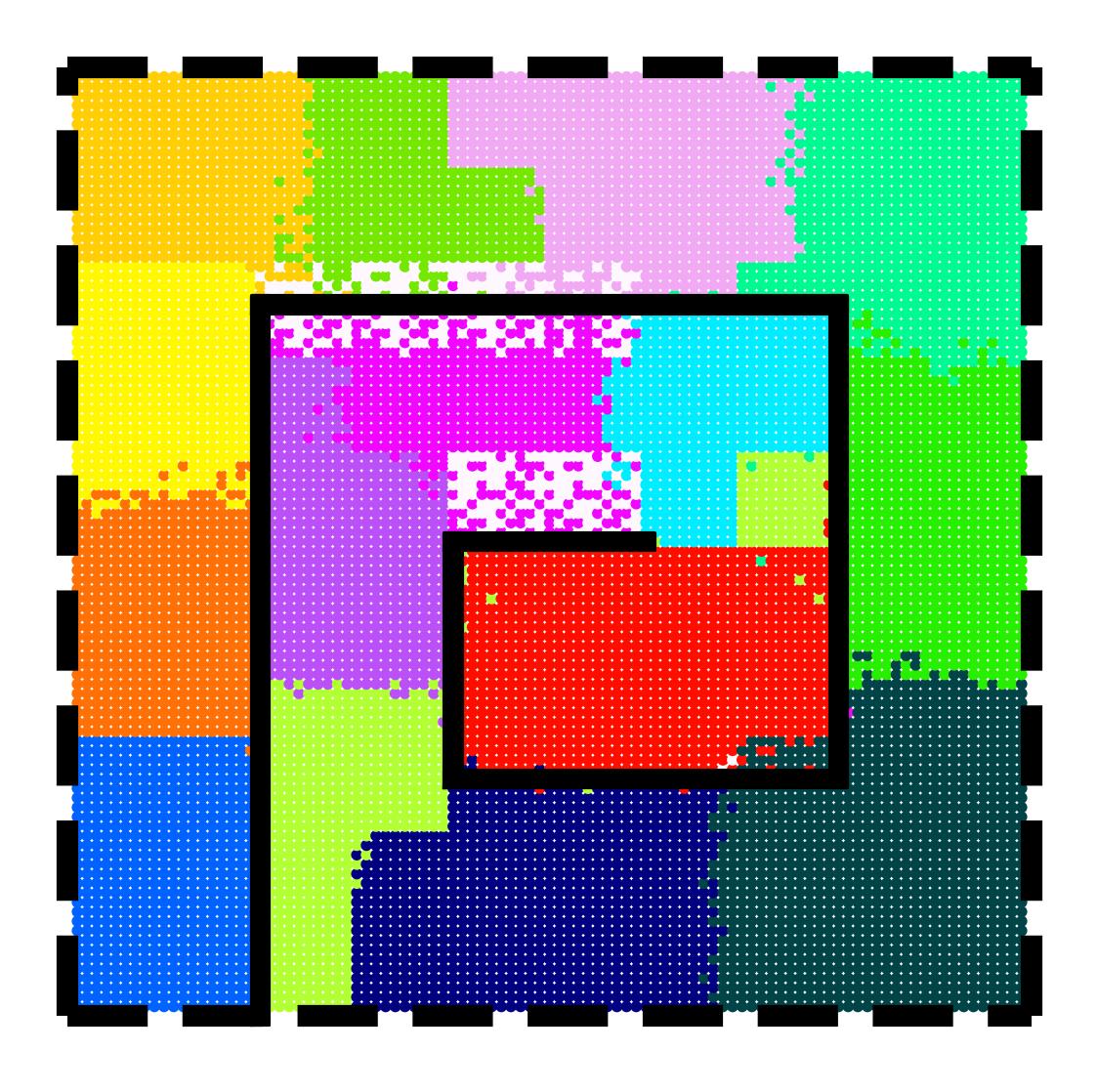} & 
   \includegraphics[width=0.18\linewidth]{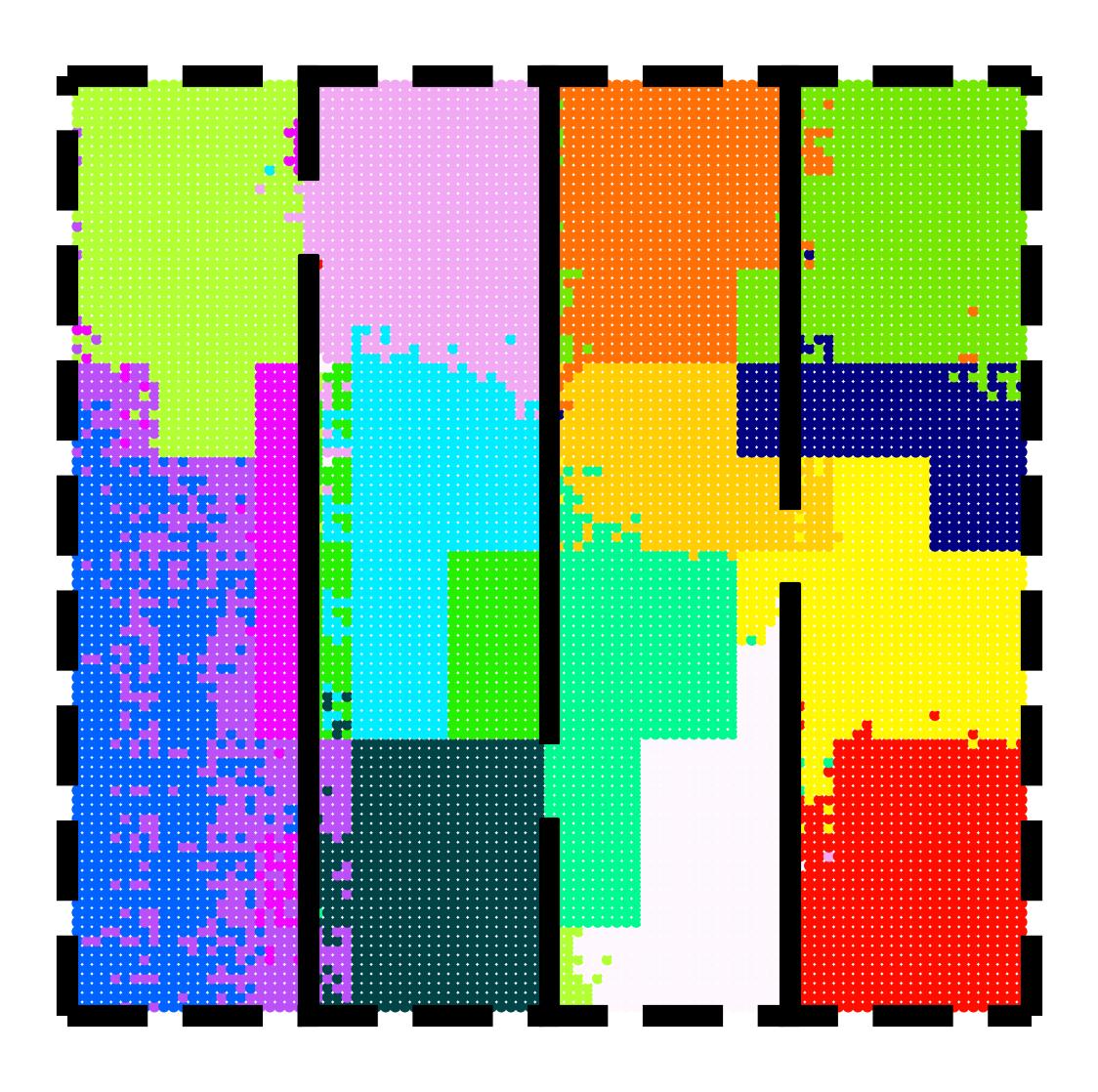} 
  \end{tabular}
  \caption{Results for \replearn on maze environments: $K$-means clusterings ($K=16$) visualize the learned latent space. 
    \textsf{Baseline} does not enforce Lipschitzness of the $g$
    functions in~\cref{eq:rep_learn}.
    \replearn learns to respect walls and boundaries; \textsf{Baseline} does not.
  }%
  \label{fig:kmeans}%
\end{figure}
}

We study questions (1) and (3) above in a set of three synthetic maze
environments \citep{koul2023pclast}, visualized in \cref{fig:kmeans}. In each maze, we control a point mass in a two-dimensional
continuous space ($\dims=2$); the three mazes
correspond to different layouts. Actions correspond to 
continuous displacement
in the environment ($\dima=2$), but are perturbed with noise
through the dynamics. The observation is a $100 \times 100$ pixel image of
the maze with pixel value $1$ in the agent's position, and
pixel value $0$ in all other coordinates. The latent dynamics are
Lipschitz in the sense of \cref{ass:lipschitz}, so the
environment falls into the \framework
framework. 

We perform a \emph{qualitative} evaluation here, analogous
to~\citet{koul2023pclast}. For each maze environment, we train the
decoder via \replearn 
with $500$k samples that are collected from a random policy. %
We visualize the learned representation by performing $K$-means
clustering ($K=16$) in the learned latent space: we uniformly sample $10$k points in
the environment and display their $(x,y)$ position, using color to
represent the cluster assignment in the learned latent space.  The
results are visualized in the top row of~\pref{fig:kmeans}, where we find that the
learned decoder correctly captures the local dynamics of the
environment, respecting the boundaries and the walls. To investigate the role of the Lipschitz constraint (question (3)),
the bottom row of~\pref{fig:kmeans} visualizes the $K$-means clusters
obtained from a decoder trained via the same protocol, but without
spectral normalization ($g\notin\Lip$) so that the prediction heads are not constrained to be
Lipschitz. 
We find that without the Lipschitz constraint, this learned
decoder does not respect the boundaries and walls and is much worse at capturing the local dynamics. \loose

\subsection{Locomotion Benchmark}
For a quantitative evaluation, and to address question (2) above, we
consider two MuJoCo environments from the visual D4RL benchmark~\citep{lu2022challenges}. %
We focus on two environments, walker and cheetah.
We 
use an evaluation protocol from
\citet{islam2022agent}: given offline data, we train decoders
using each representation learning method, then train an agent that takes the learned latent state as input via offline reinforcement learning (specifically \textsf{TD3-BC}~\citep{fujimoto2021minimalist}), and measure the reward obtained by the agent. The only deviation from the setup of \citet{islam2022agent}
is that we remove the exogenous noise from the observations, because filtering exogenous noise is not the main focus of our representation learning algorithm. 
The results of this evaluation are visualized in~\pref{fig:loc}, along with another recently proposed
method for learning continuous representations~\citep{koul2023pclast}, and (ii) a randomly initialized decoder
with the same architecture as that of the other methods. 
Learning
curves are obtained by, for each $t$, taking a checkpoint of the
decoder after $t$ training epochs, running \textsf{TD3-BC} for 1k epochs using
this decoder, and recording the reward obtained by the final policy.
We pick \pclast as a representative baseline because it outperforms other 
representation learning methods (such as contrastive learning \citep{misra2020kinematic}
and multistep inverse kinematics \citep{lamb2023guaranteed})
in the visual D4RL benchmark with exogenous noise \citep{koul2023pclast}.
In both environments, the representations obtained
by \replearn are competitive with those of \pclast and significantly better than the randomly initialized decoder, suggesting that indeed they are useful for
downstream tasks.

\begin{figure}[t]
    \centering
    \includegraphics[width=0.8\linewidth]{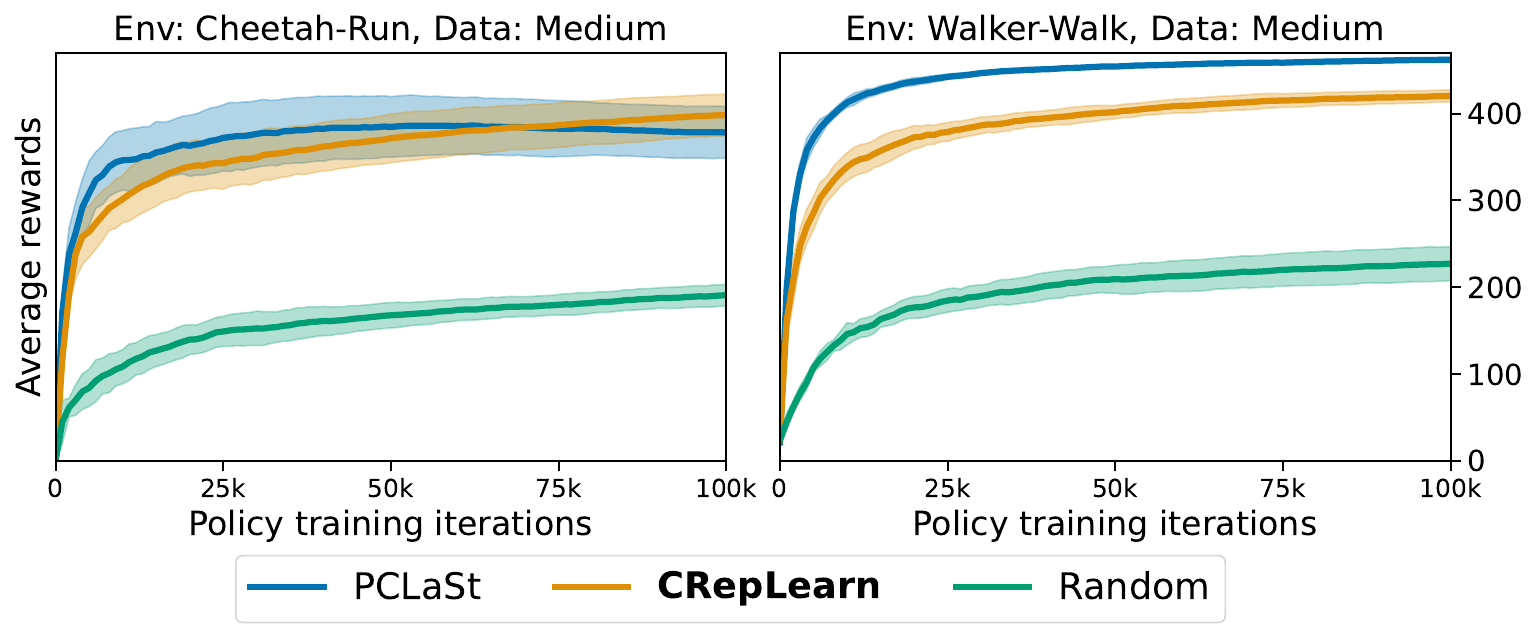}
    \caption{Results on visual D4RL: learning curves of \textsf{TD3-BC} using
      pre-trained (frozen) decoders from each representation learning
      method. Random denotes a randomly initialized decoder.}
    \label{fig:loc}
    \icml{\vspace{-4mm}}
\end{figure}


%% file: section_discussion.tex
Our work lays the foundation for investigation into rich observation
\arxiv{reinforcement learning} \icml{RL} with continuous latent states---both in terms
of algorithm design and statistical complexity---and raises several
open questions and directions for future research.
\icml{\begin{itemize}[noitemsep,topsep=0pt]}
\arxiv{\begin{itemize}}
\item \textbf{Statistical complexity.} 
  Although our guarantees have a similar nonparametric form to the optimal bounds for Lipschitz MDPs
\icml{\citep{ni2019learning,song2019efficient}}
\arxiv{\citep{ni2019learning,song2019efficient,cao2020provably,sinclair2023adaptive}},
  it remains to sharply characterize the minimax sample complexity for
  the \framework framework. It would also be interesting to
  understand whether it is possible to adapt to the intrinsic
  dimension of the latent space to obtain improved guarantees in benign instances~
  \citep{cao2020provably,sinclair2023adaptive}.
\item \textbf{The role of continuity.} Our work
  shows that TV-Lipschitz continuity suffices for learning in the \framework framework, and
  that the $Q^{\star}$-Lipschitz assumption considered in prior work~\citep{ni2019learning,song2019efficient,cao2020provably,sinclair2023adaptive}
  is insufficient. How does the optimal sample complexity in the \framework
  framework change under other notions of Lipschitzness (e.g.,
  Wasserstein) or continuity (e.g., H\"{o}lder). Are there more
  general principles under which assumptions and guarantees for Lipschitz MDPs
  transfer to the \framework framework?
\item \textbf{Exogenous noise.} Developing representation learning and
  exploration schemes that filter exogenous noise and other irrelevant
  information is an important practical problem, with some recent
  progress in the context of discrete latent dynamics \citep{efroni2021provably,lamb2023guaranteed,islam2022agent,koul2023pclast}. Can we lift
  these techniques to the continuous setting?
\item \textbf{General principles for representation learning.} Our
  work shows that existing representation learning schemes such as
  contrastive learning \citep{misra2020kinematic} and inverse dynamics
  \citep{pathak2017curiosity,badia2020agent57,baker2022video,bharadhwaj2022information}
  or multi-step inverse dynamics
  \citep{lamb2023guaranteed,mhammedi2023representation} do not succeed
  as-is with  continuous latent dynamics. Can we 
  adapt these techniques to handle continuous
  dynamics, and are there natural assumptions under which they succeed?
\end{itemize}
\arxiv{We look forward to pursuing these directions in future work. }


%% file: appendix_related.tex
\paragraph{Reinforcement learning in metric spaces}
There is a large body of literature on reinforcement learning in
metric spaces (also known as the Lipschitz MDP framework), either with
generative model access
\citep{kakade2003exploration,shah2018q,henaff2019explicit} or in the
fully online reinforcement learning framework
\citep{sinclair2019adaptive,ni2019learning,song2019efficient,sinclair2020adaptive,cao2020provably,sinclair2023adaptive}. All
of these works assume that the state is observed directly, and thus do
not address the representation learning problem. Compared to our
results, these works often operate under weaker Lipschitz continuity
assumptions; as shown in \cref{sec:golf}, these assumptions do not
readily lift to the rich-observation setting. Additionally, many of
these works establish adaptive (also known as instance-dependent)
guarantees that are analogous to gap-dependent bounds in the tabular
setting~\citep{simchowitz2019non} and that generalize prior zooming
guarantees for Lipschitz
bandits~\citep{slivkins2011contextual,kleinberg2013bandits}. Obtaining
similar adaptive guarantees for \settingname is an interesting
direction for future work.

\paragraph{Reinforcement learning with rich observations}
Reinforcement learning with rich observations has received extensive
investigation in recent years, mostly focused on the \emph{Block MDP}
framework in which the latent state space is tabular/finite
\citep{krishnamurthy2016pac,du2019provably,misra2020kinematic,zhang2022efficient,mhammedi2023representation},
as well as the closely related \emph{Low-Rank MDP} framework
\citep{agarwal2020flambe,zhang2022efficient,uehara2022representation,mhammedi2023efficient,modi2021model}. Notable
works that go beyond tabular latent spaces include
\citet{dean2019robust,mhammedi2020learning,dean2020certainty}, which considers continuous \emph{linear}
dynamics, and \citet{misra2021provable}, which considers factored (but
discrete) latent dynamics. To the best of our knowledge, there are no works from this line of
research that address continuous latent states with \emph{nonlinear} dynamics.

\paragraph{General complexity measures for reinforcement learning}
Another line of research provides general complexity measures that
enable sample-efficient reinforcement learning, including Bellman rank
\citep{jiang2017contextual,sun2019model,du2021bilinear,jin2021bellman},
eluder dimension \citep{russo2013eluder}, coverability
\citep{xie2023role}, and the Decision-Estimation Coefficient (DEC)
\citep{foster2021statistical,foster2023tight,foster2023model}. Naively
applying Bellman rank and variants
\citep{jiang2017contextual,sun2019model,du2021bilinear,jin2021bellman}
to the \framework framework is problematic due to issues around
misspecification. In particular, while the \framework framework can be shown to admit
low \emph{approximate} Bellman rank via discretization, the sample
complexity guarantees for approximate Bellman rank in these works are
too weak to give non-trivial guarantees (in detail, the
misspecification error is typically scaled by the Bellman rank,
which is problematic when the rank depends on the misspecification
error itself). Our analysis of \golf in \cref{sec:golf} can be viewed
as an approximate version of the coverability-based analysis in
\citet{xie2023role}, but it delicately exploits the specific structure
of the \framework framework. We expect that it is possible to bound
the Decision-Estimation Coefficient
\citep{foster2021statistical,foster2023tight,foster2023model} for the
\framework framework, but deriving efficient algorithms using this
framework is non-trivial.\loose

\paragraph{Empirical research on continuous control with
  high-dimensional observations}
There is a large body of empirical research that addresses continuous
control from high-dimensional observations via 
representation learning
\citep{wahlstrom2015pixels,watter2015embed,banijamali2018robust,ha2018recurrent,hafner2019dream,hafner2019learning,gelada2019deepmdp,levine2019prediction,laskin2020curl,shu2020predictive,schrittwieser2020mastering,yarats2021improving,guo2022byol,hafner2023mastering,farebrother2023proto,koul2023pclast},
often through the reconstruction of observations or prediction of future
latent states, though recent work considers more principled objectives
\citep{koul2023pclast}. 
There is a similarly large body of work from the robotics community developing representation learning methods for use in robotics applications~\citep{jonschkowski2015learning,pari2021surprising,kumar2021rma,nair2022r3m,xiao2022masked}.
These works do not provide provable guarantees
or systematically address the exploration problem.\loose

\paragraph{Continuous control}
Control of continuous, nonlinear systems is a classical topic in
control theory (e.g., \citet{slotine1991applied}). With a few exceptions~\citep{mania2020active,kakade2020information,song2021pc,boffi2021regret,pfrommer2022tasil,tu2022sample,pfrommer2023power,tian2023toward,wagenmaker2023optimal}, this
literature does not address the exploration problem when the
underlying system is unknown, nor does it consider the issue of sample complexity.


%% file: appendix_replearn_offline.tex
 In this section, we will show how to leverage the representation learning result 
\pref{thm:rep_learn} under an exploratory distribution in the offline RL setting.
In the offline RL setting, the learner only has access to a fixed dataset,
instead of the online interaction with the environment. We assume we have one 
dataset $\cD_h$ for each time step $h \in [H]$, and each dataset consist of $N$ i.i.d. samples
$\{(x^i_h,a^i_h,x^i_{h+1})\}_{i=1}^N$ generated by $x^i_h,a^i_h \sim \rho_h,
x^i_{h+1} \sim P_h(x^i_h,a^i_h)$. We call the distribution $\rho_h \in \Delta(\cX \times \cA)$ the offline distribution. 

\subsection{A General Result for Offline RL with \replearn}
Suppose for each $h \in [H]$, we have offline dataset $\cD_h$ and
function classes $\{\cV_{h}\}_{h=1}^{H+1} \subset
(\cX \to \bbR)$ that contain the optimal state-value
functions $V_h^\star \in \cV_h$. Recall that $\Phi_h: \cX \to \cS$ is
the decoder class, and we define 
$\Theta_h: \cS \to \bbR$ as the prediction head class (generalizing the 
set of Lipschitz functions that we used in the \framework framework),
such that $Q^\ast_h \in \cF_h := \Theta_h \circ \Phi_h$ for all $h$.
 We use these as inputs to the
representation learning algorithm, to obtain a decoders $\{\phi_h\}_{h=1}^H$ that
satisfy\loose
\begin{align}\label{eq:rep_learn_offline}
  \max_{v \in \cV_{h+1}} \En_{x,a \sim  \rho_h}\brk*{
     \prn*{\psb_{\cD_h, \phi_h}[v](x,a) - \realb_h[v](x,a)}^2 } = \erep(N,\delta),
\end{align}
where 
\begin{align*}
  \psb_{\cD_h, \phi_h}[v] = \min_{f \in \cF_h} \widehat \En_{x,a,x' \sim \cD_h}
  \brk*{ \prn{f(x,a) - v(x')}^2},
\end{align*}
and $\erep(N,\delta)$ is the accuracy guarantee, which we will leave abstract in this section. (In the context of the \framework framework, $\erep(N,\delta)$ is controlled by \pref{thm:rep_learn}). 
Note that for $\erep(N,\delta)$ to decay to $0$ as $N\to\infty$, one requires a form of completeness, in the sense that $\cP_h[v] \in \cF_h$ for all $v \in \cV_{h+1}$.

Then let us consider the following approximate dynamic programming (ADP) algorithm. 
The goal is to construct $f_h \in \cF_h$
to estimate $Q^\ast_h$, the optimal value function. Let $f_{H+1}(x,a) = 0,\; \forall x,a \in \cX \times \cA$, and $\pi_{H+1}(x)$ be an arbitrary action. For each $h$, we define $f^\pi_h(x) = f_h(x,\pi_h(x))$. Then for each $h = H, \ldots, 1$, for each $x,a \in \cX \times \cA$, let 
\begin{align}\label{eq:adp_offline}
  f_{h}(x,a) = R_h(x,a) + \psb_{\cD_h, \phi_h}[f^{\pi}_{h+1}](x,a),
  \quad \pi_h(x) = \argmax_{a \in \cA} f_{h}(x,a).
\end{align}  
We show that, with an exploratory offline distribution, running this
ADP algorithm with the representation learning result
in~\pref{eq:rep_learn_offline} will lead to a policy that is close to
the optimal policy. In this abstract setting, we assume that 
the offline distribution is exploratory in
the sense that it enables transferring of Bellman error:
\begin{definition}\label{def:bellman_error_cov}
The Bellman transfer coefficient of a distribution $\rho$ with respect to value function class $\cF$ is 
\begin{align*}
  \cconc^{\mathsf{be}}(\cF) \ldef \max_{f \in \cF} \max_{g \in \cF} \frac{\sum_{h=1}^H \En_{x,a \sim d^{\pi_g}_h} 
  \brk*{\abs*{f_h(x,a) - \cT_h f_{h+1}(x,a)}}}{\sum_{h=1}^H\sqrt{\En_{x,a \sim \rho_h} 
  \brk*{\prn*{f_h(x,a) - \cT_h f_{h+1}(x,a)}^2}}}.
\end{align*}
\end{definition}
Whenever $\cconc^{\mathsf{be}}(\cF)$ is small, we can show that the ADP algorithm 
in \pref{eq:adp_offline} will return a policy that is close to the optimal policy.

\begin{proposition}\label{prop:offline_be}
  If for all $h \in [H]: V_h^\ast \in \cV_h, Q_h^\star \in \cF_h$, and~\pref{eq:rep_learn_offline} holds, then the policy $\pi$ returned by \pref{eq:adp_offline} satisfies
  \begin{align*} 
    J(\pi^\ast) - J(\pi) \leq 
    H\cdot\cconc^{\mathsf{be}}(\cF)\cdot \sqrt{\erep(N,\delta)}.
  \end{align*}
\end{proposition}

Note that conditioned on the representation learning
result,~\pref{prop:offline_be} does not require any additional
assumptions about the dynamics of the MDPs such as
\pref{ass:lipschitz}. In other words, the Bellman error coverage
assumption is a general coverage notion that can guarantee optimality
as long as the representation learning guarantee holds.

\subsection{Results for Offline RL in the \framework Framework}

We now turn to the \framework framework. The first result follows
immediately from instantiating the above general result to the
\framework framework: let $\cF = \Lip \circ \Phi$, we have the
following result:
\begin{corollary}\label{corollary:offline}
  Suppose \pref{ass:lipschitz} holds. Let $\cF = \Lip \circ \Phi$,
  $\cV = \{ \max_a f(x,a): f \in \cF\}$, and assume \pref{eq:rep_learn_offline} holds for $\cF$, then 
  the policy $\pi$ returned by \pref{eq:adp_offline} satisfies
  \begin{align*} 
    J(\pi^\ast) - J(\pi) \leq 
    H\cdot\cconc^{\mathsf{be}}(\cF)\cdot \sqrt{\erep(N,\delta)},
  \end{align*}
  where $\erep(N,\delta)$ is given in~\pref{thm:rep_learn}.
\end{corollary}
\pref{corollary:offline} follows from \pref{prop:offline_be} and~\pref{thm:rep_learn}, along with the fact that, under~\pref{ass:lipschitz}, we have the required realizability conditions. This is verified in~\pref{lem:realizability}.

We obtain a second result by considering another coverage notion specific to the \framework
framework. The intuition is that, the Bellman transfer coefficient is
defined with respect to the observations, but in the \framework
framework it is more natural to measure coverage in the latent
state. Toward a suitable definition, let us consider several
options. The first is the $\ell_\infty$ notion of coverage $\sup_{s
  \in \cS, h \in [H]}\frac{d^\pi_h(s)}{\rho_h(s)}$; this can be
unbounded since the latent state space $\cS$ is continuous. The second
is the Bellman transfer coefficient in the latent space,
i.e.,~\pref{def:bellman_error_cov} with $s = \phi^\ast(x)$ instead of
$x$; this yields exactly the same guarantee
as~\pref{prop:offline_be}. The third candidate, is $\ell_{\infty}$
coverage over a \emph{discretization} of the latent state space; as we
will see, this yields a different guarantee
than~\pref{prop:offline_be}.

Recall that for any $\eta > 0$, we can construct an $\eta$-cover
$\cS_\eta$ of the latent space $\cS_\eta$ such that for any $s \in
\cS$, there exists $s' \in \cS_\eta$ such that $D_\cS(s, s') \leq
\eta/2$. As in~\pref{sec:mainalg} we can define covering elements
as follow: given a pair $(s,a) \in \cS \times \cA$, we define
$\disc{\eta}(s,a) \in \cS_\eta\times \cA_\eta$ as any covering element
for $(s,a)$ in $\cS_\eta\times\cA_\eta$ such that
$D((s,a),\disc{\eta}(s,a)) \leq \eta$.\footnote{If $(s,a)$ is covered
by more than one element in $\cS_\eta\times\cA_\eta$, we break ties in
an arbitrary but consistent fashion.}  Next, we define
$\ball{\eta}(s,a) := \crl*{ (\tilde{s},\tilde{a}) \in \cS\times\cA :
  \disc{\eta}(s,a) = \disc{\eta}(\tilde{s},\tilde{a})}$ as the
``ball'' of $(s,a)$ pairs that map to the same covering
element. Finally, we define $\cB_\eta :=
\crl*{\ball{\eta}(s_\eta,a_\eta) : s_\eta, a_\eta \in \cS_\eta
  \times\cA_\eta}$; note that the definitions ensure that $\cB_\eta$
is a partition of $\cS\times\cA$. Then for any distribution $\rho$
over $\cS \times \cA$, given any $\ball{\eta} \in \cB_\eta$, we define
$\rho(\ball{\eta})$ as the probability mass of $\rho$ on $\ball{\eta}$
(since there are finitely many balls). 

Given these definitions, we can define the $\ell_{\infty}$ coverage over the discretization. 
\begin{definition}\label{def:latent_ball_cov}
  Given any $\eta > 0$, the $\ell_\infty$-discretized concentrability of the
  offline distribution $\rho$ with respect to value function class
  $\cF$ is defined as
  \begin{align*}
    \cconc^{\mathsf{ball}}(\cF;\eta) \ldef \max_{f \in \cF} \max_{\ball{\eta} \in \cB_\eta, h \in [H]} \frac{d^{\pi_f}_h(\ball{\eta})}{\rho(\ball{\eta})}.
  \end{align*}
\end{definition}

To build intuition, observe that if $\eta$ is small, e.g., if we take
$\eta \to 0$, then we recover $\ell_\infty$ coverage over the
continuous latent space. Thus $\cconc^{\mathsf{ball}}(\cF;\eta)$ will
likely be unbounded as $\eta \to 0$. On the other hand, if we take
$\eta$ to be large, then the coverage will small, e.g, if $\eta = 1$,
then the coverage will trivially be 1 (assuming the diameter of the
latent state-action space is $1$).

As we will see later, we will also require the offline distribution to
provide coverage over the action space, which we formalize via the
following assumption:
\begin{assumption}[Action coverage]\label{assum:offline_action}
   We assume that for all $h \in [H]$, $\eta > 0$, we have: 
\begin{align*}
  \sup_{x\in\cX,a\in\cA,f\in\cF} \frac{\pi_{h;f}(a \mid x)}{\rho_h(\ball{\eta}(a) \mid x)} \leq A_{\mathsf{cov};\eta}, 
\end{align*}
where $\ball{\eta}(a) := \{ \tilde{a}: \disc{\eta}(a) = \disc{\eta}(\tilde{a})\}$, and $\disc{\eta}(a)$ denotes the covering element of $a$ in $\cA_\eta$. 
\end{assumption}
Note that control on conditional probabilities of the actions does not guarantee (tight) control on the discretized concentrability coefficient and vice versa. 
Indeed, the offline data distribution $\rho$ is induced by a policy
that takes all actions uniformly from the covering set $A_\eta$, then
$A_{\mathsf{cov};\eta} \leq |\cA_\eta|$, but the discretized concentrability
coefficient can be larger then $|\cA_\eta|^H$ in general.

\begin{proposition}\label{prop:offline_ball}
  Under \pref{ass:lipschitz},
  if~\pref{eq:rep_learn_offline} holds, then the policy $\pi$ returned by \pref{eq:adp_offline} satisfies
  \begin{align*} 
    J(\pi^\ast) - J(\pi) \leq 
    \inf_{\eta > 0} \crl*{ H  A_{\mathsf{cov};\eta}\cdot \prn*{\cconc^{\textsf{ball}}(\cF;\eta)} \sqrt{\erep(N,\delta)} + 2H\eta}.
  \end{align*}
\end{proposition}
The bound reveals the tradeoff in the discretization level $\eta$; the approximation error term decreases as $\eta \to 0$, but both coverage terms increase.

\subsection{Proofs of the Offline Results}
In this section we provide the proofs of the results 
earlier in this section. We require additional notations 
for the proofs, whose introduction is deferred to \pref{app:technical} for a more streamlined presentation. 

\label{app:offline_proof}
\input{appendix_offline_proof.tex}


%% file: appendix_offline_proof.tex
\begin{proof}[\pfref{prop:offline_be}]
    By the standard decomposition, we have  
    \begin{align*}
      &~~~~~J(\pi^\ast) - J(\pi)\\ &=  \En \brk*{V_1^\ast(x_1) - f_1^{\pi}(x_1)} + \En \brk*{f_1^{\pi}(x) - V_1^{\pi}(x)} \\
      &\leq  \En \brk*{V_1^\ast(x_1) - f_1^{\pi^\ast}(x_1)} + \En \brk*{f_1^{\pi}(x) - V_1^{\pi}(x)} \\
      &\leq \sum_{h=1}^H \En_{x,a \sim d^{\pi^\ast}_h} \brk*{-\prn{f_h(x,a) - \cT_h f_h(x,a)}} + \sum_{h=1}^H \En_{x,a \sim d^{\pi}_h} \brk*{f_h(x,a) - \cT_h f_{h+1}(x,a)} \\
      &\leq \sum_{h=1}^H \En_{x,a \sim d^{\pi^\ast}_h} \brk*{\abs*{f_h(x,a) - \cT_h f_h(x,a)}} + \sum_{h=1}^H \En_{x,a \sim d^{\pi}_h} \brk*{\abs*{f_h(x,a) - \cT_h f_{h+1}(x,a)}} \\
      &= \sum_{h=1}^H \En_{x,a \sim d^{\pi^\ast}_h} \brk*{-\prn*{ \psb_{\cD_h, \phi_h}[f_{h+1}](x,a)- \realb_h[f_{h+1}](x,a)}} + \sum_{h=1}^H \En_{x,a \sim d^{\pi}_h} \brk*{\psb_{\cD_h, \phi_h}[f_{h+1}](x,a)- \realb_h[f_{h+1}](x,a)} \\
      &\leq \sum_{h=1}^H \En_{x,a \sim d^{\pi^\ast}_h} \brk*{\abs*{ \psb_{\cD_h, \phi_h}[f_{h+1}](x,a)- \realb_h[f_{h+1}](x,a)}} + \sum_{h=1}^H \En_{x,a \sim d^{\pi}_h} \brk*{ \abs*{\psb_{\cD_h, \phi_h}[f_{h+1}](x,a)- \realb_h[f_{h+1}](x,a)}}.
    \end{align*}
    
    then we can show that let $\pi^\ast$ be the optimal policy, and let $\pi$ be the policy 
    returned by our ADP procedure, then 
    by \pref{def:bellman_error_cov},
    \begin{align*}
      &J(\pi^\ast) - J(\pi) \\ \leq& \cconc^{\mathsf{be}}(\cF) \sum_{h=1}^H\sqrt{\En_{x,a \sim \rho_h} \brk*{\prn*{f_h(x,a) - \cT_h f_{h+1}(x,a)}^2}} + \cconc^{\mathsf{be}}(\cF) \sum_{h=1}^H\sqrt{\En_{x,a \sim \rho_h} \brk*{\prn*{f_h(x,a) - \cT_h f_{h+1}(x,a)}^2}} \\
      =& 2\cconc^{\mathsf{be}}(\cF) \sum_{h=1}^H\sqrt{\En_{x,a \sim d^{\pi^\ast}_h} \brk*{\prn*{\psb_{\cD_h, \phi_h}[f_{h+1}](x,a)- \realb_h[f_{h+1}](x,a)}^2}} \\
      =& 2 \cconc^{\mathsf{be}}(\cF)H \sqrt{\erep(N,\delta)}.
    \end{align*}
    The last line is because $f_h$ is Lipschitz and thus the representation learning guarantee holds for all $f_h$ defined using pseudobackups in the ADP procedure. 
    \end{proof}
    
    \begin{proof}[\pfref{prop:offline_ball}]
      Again by the same decomposition as the previous proof, we have  
    \begin{align*}
      &~~~~~J(\pi^\ast) - J(\pi)\\ 
      &\leq \sum_{h=1}^H \En_{x,a \sim d^{\pi^\ast}_h} \brk*{\abs*{ \psb_{\cD_h, \phi_h}[f_{h+1}](x,a)- \realb_h[f_{h+1}](x,a)}} + \sum_{h=1}^H \En_{x,a \sim d^{\pi}_h} \brk*{ \abs*{\psb_{\cD_h, \phi_h}[f_{h+1}](x,a)- \realb_h[f_{h+1}](x,a)}} \\
      &= \sum_{h=1}^H \int   d^{\pi^\ast}_h(x,a) \delta_h(x,a) \bm{x,a} + \sum_{h=1}^H \int   d^{\pi}_h(x,a) \delta_h(x,a) \bm{x,a},
    \end{align*}
    where we use the shorthand $\delta_h(x,a) \ldef \abs*{\psb_{\cD_h, \phi_h}[f_{h+1}](x,a)- \realb_h[f_{h+1}](x,a)}$. 
    
    Next, for any $\pi$, we can rewrite 
    \begin{align*}
      &\sum_{h=1}^H \int   d^{\pi}_h(x,a) \delta_h(x,a) \bm{x,a} \\
      =& \sum_{h=2}^{H} \int   d^{\pi}_{h-1}(x,a) 
      \crl*{\int   P_h(\widetilde x \mid x,a) \pi_h(\widetilde a \mid \widetilde x)\delta_h(\widetilde x,\widetilde a) \bm{\widetilde x,\widetilde a}} \bm{x,a} + \int   d^{\pi}_{1}(x,a) \delta_1(x,a) \bm{x,a}\\
      \leq& A_{\mathsf{cov};\eta} \sum_{h=2}^{H} \int   d^{\pi}_{h-1}(x,a)  
      \crl*{\int   P_h(\widetilde x \mid x,a) \rho_h(\widetilde a \mid \widetilde x)\delta_h(\widetilde x,\widetilde a) \bm{\widetilde x,\widetilde a}} \bm{x,a} \\ &+ A_{\mathsf{cov};\eta} \int   \rho_{1}(x,a) \delta_1(x,a) \bm{x,a} + H\eta,
    \end{align*} 
    where the last line is due to \pref{assum:offline_action} and the fact that 
    $\delta_h$ is Lipschitz with respect to action (but not necessarily Lipschitz in observation). Then let us denote another shorthand $\widetilde \delta_{h-1}(x,a) \ldef \int   P_h(\widetilde x \mid x,a) \rho_h(\widetilde a \mid x)\delta_h(\widetilde x,\widetilde a) \bm{\widetilde x,\widetilde a}$, and due to the Lipschitz assumption on the latent dynamics (\pref{ass:lipschitz}), we can observe that $\tilde \delta$ is Lipschitz in both latent state ($\phi^\ast(x)$) and action. 
    
    Now for any distribution $p$ over $\cX \times \cA$, we denote the probability of visiting a state action pair $x,a$ conditioned on visiting a ball $b$: $p(x,a \mid b)$ such that $p(x,a) = p(b) \cdot p(x,a \mid b)$, where $p(b)$ is the probability over a ball that we defined above. Now for any $\pi$ and $h \geq 2$ and $\eta > 0$, we have 
    \begin{align*}
      \int   d^{\pi}_{h-1}(x,a) \widetilde \delta_{h-1}(x,a) \bm{x,a} &= \sum_{\ball{\eta} \in \cB_\eta} d^{\pi}_{h-1}(\ball{\eta}) \int   d^{\pi}_{h-1}(x,a \mid \ball{\eta})\widetilde \delta_{h-1}(x,a) \bm{x,a} \\
      &\leq \sum_{\ball{\eta} \in \cB_\eta} d^{\pi}_{h-1}(\ball{\eta}) \int   \rho_{h-1}(x,a \mid \ball{\eta})\widetilde \delta_{h-1}(x,a) \bm{x,a} + \eta \\
      &\leq \cconc^{\mathsf{ball}}(\cF;\eta) \sum_{\ball{\eta} \in \cB_\eta} \rho_{h-1}(\ball{\eta}) \int   \rho_{h-1}(x,a \mid \ball{\eta})\widetilde \delta_{h-1}(x,a) \bm{x,a} + \eta \\
      &= \cconc^{\mathsf{ball}}(\cF;\eta)  \int   \rho_{h-1}(x,a)\widetilde \delta_{h-1}(x,a) \bm{x,a} + \eta \\
      &= \cconc^{\mathsf{ball}}(\cF;\eta)  \int   \rho_{h}(x,a) \delta_{h}(x,a) \bm{x,a} + \eta,
    \end{align*}
    where the second line is due to $\widetilde \delta$ differs at most $\delta$ 
   within observations from the same $\delta$-ball, and the conditional probability marginalizes to 1. Line 3 is due to the definition of the ball coverage and importance weighting (note that $\widetilde \delta$ and $\delta$ are always non-negative). The last line is due to the construction of $\widetilde \delta$. 
    
    Then putting everything together we have:
    \begin{align*}
      &~~~~~J(\pi^\ast) - J(\pi)\\ 
      &\leq \sum_{h=1}^H \int   d^{\pi^\ast}_h(x,a) \delta_h(x,a) \bm{x,a} + \sum_{h=1}^H \int   d^{\pi}_h(x,a) \delta_h(x,a) \bm{x,a} \\
      &\leq 2A_{\mathsf{cov};\eta} \cdot \cconc^{\mathsf{ball}}(\cF;\eta) \sum_{h=1}^H \int   \rho_{h}(x,a) \delta_{h}(x,a) \bm{x,a} + 2H\eta \\
      &= 2H  A_{\mathsf{cov};\eta}\cdot \cconc^{\mathsf{ball}}(\cF;\eta) \sqrt{\erep(N,\delta)} + 2H\eta.
    \end{align*}
\end{proof}

%% file: appendix_experiments.tex
\subsection{Pseudocode of \textsf{Iter-}\replearn}

\begin{algorithm}[htp] 
  \caption{\textsf{Iter-}\replearn}
  \begin{algorithmic}[1]  
  \Require Dataset $\cD_h = \{x,a,x'\}$, decoder class $\Phi$, discriminator 
  class $\cF_{h+1}$, latent Lipschitz class $\Lip$, number of iterations $T$, 
  termination threshold $\beta$.
  \State Randomly initialize $\phi_h^1 \in \Phi$.
  \State Denote loss function 
  \begin{align*}
    \ell_{\cD_h}(\phi, g, f) = \En_{(x_h,a_h,x_{h+1})\sim \cD_h} \brk*{ (g(\phi(x_h,a_h)) - f(x_{h+1}))^2 }.
  \end{align*} 
  \For{$t = 1, \dots, T$}  
  \State Select the adversarial discriminator for the current decoder: 
  \begin{align}\label{eq:discriminator_selection}
    f^t = \argmax_{f \in \cF_{h+1}} \brk*{\min_{g \in \Lip}\ell_{\cD_h}(\phi_h^t, g, f) - 
    \min_{\widetilde \phi_h \in \Phi, \widetilde g \in \Lip} \ell_{\cD_h}(\widetilde \phi, \widetilde g, f)}.
  \end{align}
  \If {$f^t$ induces a loss smaller than $\beta$}
  \State \textbf{return} $\phi^t$.
  \EndIf
  \State Update the decoder to minimize Bellman error for all discriminators: 
  \begin{align}\label{eq:decoder_update}
    \phi^{t+1}_h = \argmin_{\phi_h \in \Phi} \brk*{\min_{\{g^i \in \Lip\}_{i=1}^t} 
    \sum_{i=1}^t \ell_{\cD_h}(\phi_h, g^i, f^i)}.     
  \end{align}
  \EndFor
  \end{algorithmic}
  \label{alg:rep_learn_iter}
  \end{algorithm}

\subsection{Maze Environments}

\paragraph{Environments} Here we give more details about the maze environment. 
Each maze is a two-dimensional point-mass control problem, and the latent 
state is the $x,y$ coordinates of the point mass. Specifically, we consider
$\cS = [0,1] \times [0,1]$ and $\dims = 2$. The observation space is a 
$100 \times 100$ pixel image of the maze, and the value of the pixel is $1$ if
the point mass is in the corresponding position, and $0$ otherwise. The action
space is $\cA = [-0.2,0.2] \times [-0.2,0.2]$, and $\dima = 2$, which corresponds
to the intended displacement of the point mass. The latent dynamics proceed as:
given state $s$ and action $a$,
we first sample a uniform action noise from $\xi \sim \mathsf{Unif}([0,0.01] \times [0,0.01])$,
and then proceed to $s' = w(s+a+\xi)$, where $w$ is the projection of the point mass
according to the wall dynamics: for example, the point mass cannot move through walls.
Note that the latent dynamics and the emission distribution are identical across 
all time steps. The three mazes in our experiments share the same $\cS,\cA,\cX$, 
emission distribution and action noise,
and the only difference is the configuration of the walls.

Note that in the maze environment, the horizon does not affect the
representation learning algorithm due to the stationary dynamics, and
we do not perform policy optimization in this environment.

For the collection of offline data, we run a random policy (i.e., $a
\sim \mathsf{Unif}([-0.2,0.2] \times [-0.2,0.2])$) for 500k steps in
each maze. Specifically, every 2000 steps, we reset the point
mass to a random position in the maze and proceed rolling out the
random policy for 2000 steps. The collected data is then used to train
the representation learning algorithm. Note that, for these environments, this offline data-generating process
is exploratory, as can be seen from the example in
\pref{fig:maze}.
\begin{figure}[h]
  \centering
  \includegraphics[width=0.4\textwidth]{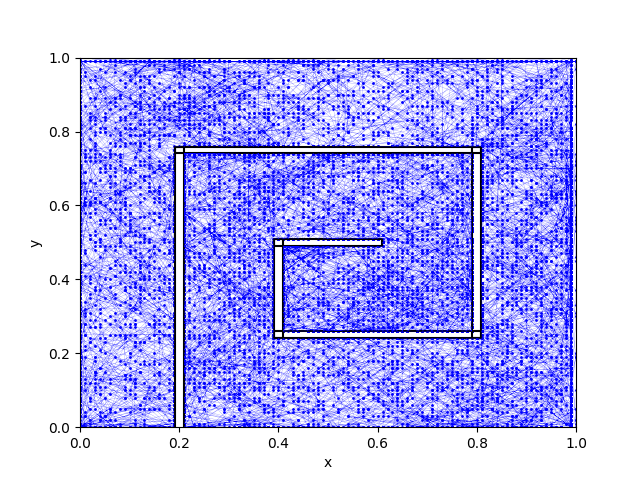}
  \caption{Visualization of the coverage of the offline dataset in the
    latent space for the Spiral Maze environment. Each 
  dot denotes one sample and each line denotes the transition between the states. Here 
  we subsample 20000 samples out of the 500k samples collected in the maze environment.}
  \label{fig:maze}
\end{figure}

\paragraph{\replearn implementation} 
We start with a concrete description of the way in which we implement
the iterative version of \replearn given in \pref{alg:rep_learn_iter}:
in each iteration $i$, we first fix the previous decoder $\phi^{i-1}$, and then
we solve for the discriminator $f^i$ that maximize the prediction error of $\phi^{i-1}$
(with the double sampling correction, i.e., solve \pref{eq:discriminator_selection}). 
We solve the minimax problem in~\pref{eq:discriminator_selection} 
by alternating between 
performing several gradient descent steps on $g$ and $\tilde g$, and one gradient ascent step on $f$ and 
$\tilde \phi$ (This is a form of two-time scale gradient-descent-ascent). 

After we solve for $f^i$, to solve \pref{eq:decoder_update},
 we fix the discriminators from all 
iterations $f^1, \ldots, f^i$, and we find the decoder $\phi^i$ that minimize the
backup error for all $f^1, \ldots, f^i$ at the same time (with a different 
Lipschitz prediction head for each $f^i$). In the following hyperparameter
table, we denote the number of gradient descent steps on $\phi$ as ``decoder steps'',
and the number of gradient ascent steps on $f$ and $\tilde \phi$ as ``discriminator steps'', and gradient descent steps on $g$ and $\tilde g$ as ``prediction head steps''.

\paragraph{Baseline} For the baseline displayed in \cref{fig:kmeans}, 
we consider a variant of \replearn without the Lipschitz constraint.
That is, for the \replearn implementation above, we use spectral normalization for each weight matrix in the
prediction head, but in the baseline we do not add spectral
normalization (the baseline is otherwise identical).\loose

\paragraph{Architecture} In maze, we use  MLP-Mixer 
\citep{tolstikhin2021mlp} to parameterize the decoder, we use a two-layer 
network with spectral normalization \citep{miyato2018spectral} to parameterize the
Lipschitz prediction head ($\Lip$), and the discriminator is the composition of the 
decoder and the Lipschitz prediction head. More hyperparameters are given in \pref{table:maze}.

\begin{table}[htp]
    \caption{Hyperparameters for Maze}
  \centering
  \begin{tabular}{cccc}
    \toprule
    &\; \replearn Value &\; Baseline Value &\; Values considered \\
    \hline
    Offline sample size                         &\; 500000 &\; 500000 &\; 500000  \\
    Decoder steps                               &\; 200    &\; 200    &\; 200     \\
    Discriminator steps                         &\; 50     &\; 50     &\; 50      \\
    Prediction head steps                       &\; 10     &\; 10     &\; 10      \\
    Minibatch size                              &\; 128    &\; 128    &\; 128     \\
    Latent representation dimension             &\; 32     &\; 128    &\; 32;128  \\
    MLP hidden layer size                       &\; 128    &\; 128    &\; 128     \\
    MLP-mixer number of layers                  &\; 2      &\; 2      &\; 2       \\
    MLP-mixer number of channels                &\; 32     &\; 32     &\; 32      \\
    MLP-mixer patch size                        &\; 10     &\; 10     &\; 10      \\
    MLP-mixer hidden layer size                 &\; 32     &\; 32     &\; 32      \\
    \toprule
\end{tabular}\label{table:maze}
\end{table}

\subsection{Locomotion Environments}

\paragraph{Environments} We consider two MuJoCo environments, walker-walk and
cheetah-run. We train representations on the environments using the 
visual D4RL benchmark datasets
\citep{lu2022challenges}.\footnote{The datasets that we use can be downloaded in \href{https://drive.google.com/drive/folders/1BCPzlgkKW4T7Do9xtV1N4ykP5_ZdvTVP}{data source: cheetah run medium} and 
 \href{https://drive.google.com/drive/folders/1KLoXtyqTkhT-n63n10xeDBJf57oq9AQg}{data
   source: walker walk medium}.}
 
The action space is 6 dimensional where each dimension has a real value ranging from -1 to 1. Each frame consists of 3 channels resulting in the size $3 \times 84 \times 84$,
and each observation $x_h$ in a batch consists of 3 sequential frames resulting in the size $9 \times 84 \times 84$ where 
each value is the pixel value range from 0 to 255, as shown in \cref{fig:mujoco_seq} for visual illustration.

\begin{figure}[ht!]
  \centering
      \includegraphics[width=0.4\linewidth]{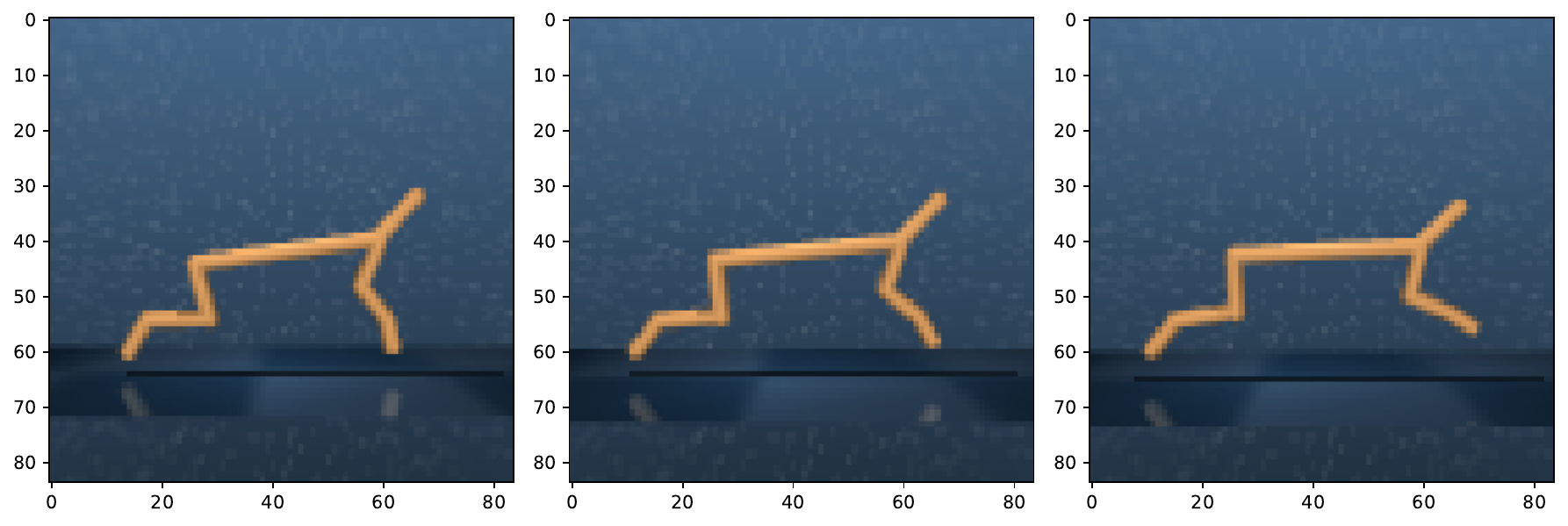}
      \includegraphics[width=0.4\linewidth]{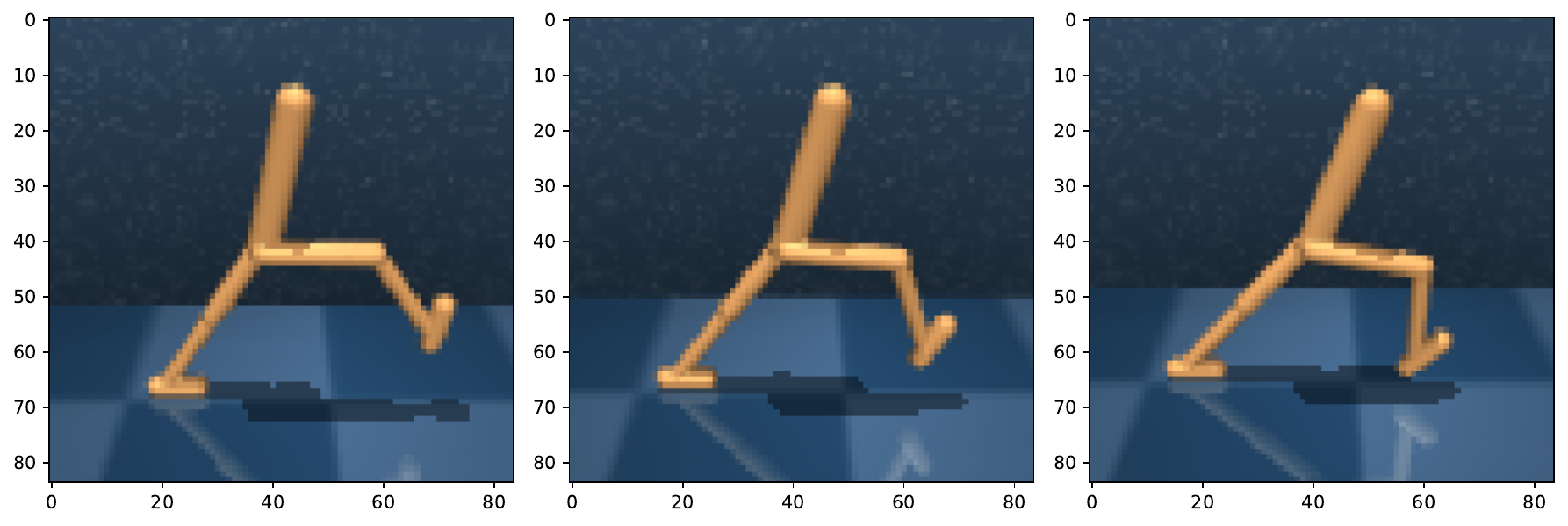}
  \caption{Left: 3-frames Cheetah Run; Right: 3-frames Walker Walk.}
  \label{fig:mujoco_seq}
\end{figure}

\paragraph{Baselines}
For baselines, we consider (i), the \pclast method
\citep{koul2023pclast}, which was shown to have strong performances in
the D4RL benchmark experiments among various baselines,
 and (ii) a randomly initialized decoder.
\pclast learns the decoder by combining a multi-step inverse kinematics (predicting the current action from the current and future state, the same as the \music objective in \eqref{eq:musik}) 
which helps filter out exogenous noises 
and a temporal contrastive objective on top of the learned representation from the multi-step inverse kinematics objective. The goal of the temporal contrastive objective is to cluster the states that are temporally close to each other together, and push the states that are temporally far away from each other apart, in the $\ell_2$ metric. Here the temporal relationship is determined by the dynamics of the environment. 
For (ii) the random initialized decoder, we use the same encoder architecture as that of the other methods, but it is randomly initialized without further training, i.e., we directly run the offline RL algorithm with the randomly initialized decoder frozen through the entire training process.

\paragraph{Architecture} We parameterize \replearn with an
architecture similar to the one used for the Maze environments, except that we use a four-layer convolutional neural network rather than MLP-Mixer to parameterize the decoder. 
The configuration of the four-layer convolutional neural network is presented in \cref{table:mujoco}.

\paragraph{Evaluation protocol}
We first learn the decoders via the three representation learning methods, and then we use 
an offline RL method, TD3-BC \citep{fujimoto2021minimalist} with frozen decoders to evaluate performance. Given a frozen decoder, we run TD3-BC for 100k iterations. Every 1k iterations, we evaluate the learned policy in the environment by running 10 episodes online and recording the average 
episodic rewards over the 10 episodes.
Finally we plot the average rewards over the iterations. We plot the mean and the shaded area denotes 2 standard
errors across 5 replicates as shown in \cref{fig:loc}. In each replicate we re-train the representation and the TD3-BC agent, using different random seeds.

For computational efficiency and GPU memory usage consideration, we bound the set of
discriminators to have size $M=125$, and discard one at
random when the set reaches this
size. 

We use the same parameter configuration as in \citet{koul2023pclast} with two differences.
The first concerns data augmentation, where we add the same random noise to both $x_h$ and $x_{h+1}$ while~\citet{koul2023pclast}
  add different random noise to all the observations. The second is that we train the \pclast decoder for 100k iterations
   but train the \replearn decoder for 1450 iterations. However, note that 
   for \replearn, each training iteration involves iteratively solving the min-max-min problem, so the number of iterations is not directly comparable. 
  The \replearn-related hyperparameters are presented in \cref{table:mujoco}.
\begin{table}[htp]
  \caption{Hyperparameters for Locomotion Environments}
  \centering
  \begin{tabular}{cc}
    \toprule
    &\; Value  \\
    \hline
    The number of the offline frames            &\; 100000   \\
    Decoder steps                               &\; 200      \\
    Discriminator steps                         &\; 50       \\
    Prediction head steps                       &\; 10       \\
    Batch size                                  &\; 256      \\
    Latent representation dimension             &\; 512      \\
    Number of the decoder pretraining steps     &\; 1450     \\
    Convolutional Layer 1:  \\
     \boldmath{$\cdot$} \small{Input channels}             &\; 9x84x84 \\
     \boldmath{$\cdot$} \small{Output channels}            &\;  32 \\
     \boldmath{$\cdot$} \small{Kernel size}                &\;  3x3 \\
     \boldmath{$\cdot$} \small{Stride}                     &\;  2 \\
     \boldmath{$\cdot$} \small{Activation function}        &\;  ReLU \\
    Convolutional Layer 2--4: \\
    \boldmath{$\cdot$} \small{Input and output channels both} &\;  32 \\
    \boldmath{$\cdot$} \small{Kernel size}                 &\;  3x3 \\
    \boldmath{$\cdot$} \small{Stride}                      &\;  1 \\
    \boldmath{$\cdot$} \small{Activation function}         &\;  ReLU \\
    \toprule

\end{tabular}\label{table:mujoco}\end{table}


%% file: appendix_auxiliary.tex
\subsection{Measure-theoretic Notation}
\label{app:measure}
For all the proofs in the appendix, we will adopt the following measure-theoretic notation.
We assume that there exists a $\sigma$-finite measure $\nu_{\cX}$ for the state space $\cX$ 
such that for every $h \in [H], x \in \cX, a \in \cA$, 
we have $P_h(x' \mid x,a)$ as a probability density function with respect to $\nu_\cX$,
i.e., 
\begin{align*}
  \forall \cX'\subset \cX: ~~ \bbP(\cX' \mid x_h = x,a_h = a) = \int_{\cX'} P_h(x' \mid x,a) \mathrm{d}\nu_\cX(x').
\end{align*}
Similarly, we assume there is a $\sigma$-finite measure $\nu_{\cA}$
for the action space $\cA$, such that for any policy $\pi$ and $x \in
\cX$, $\pi(x)$ is a probability density function with respect to
$\nu_\cA$, i.e., 
\begin{align*}
  \forall \pi \in \Pi, x \in \cX, \cA' \subset \cA: ~~ \pi(\cA' \mid x) = \int_{\cA'} \pi(a' \mid x) \mathrm{d}\nu_\cA(a').
\end{align*}
We denote the joint measure of the state-action space as $\nu(x,a) = \nu_\cX(x) \times \nu_\cA(a)$.

We typically use the same notation for a probability distribution and
density function; the object we refer to will be clear from context. 
For example, recall that $d_h^\pi\in \Delta(\cX\times\cA)$ is the
occupancy of policy $\pi$ at time $h$. Then given a function $f: \cX
\times \cA \to \bbR$, we can define the expectation of $f$ under
$d^\pi_h$ as
\begin{align*}
    \En_{x,a \sim d^\pi_h} \brk{ f(x,a) } = \int_{\cX\times\cA} f(x,a) d^\pi_h(x,a) \bm{x,a},
\end{align*}
where it is clear that $d^\pi_h$ on the left hand side denotes the
probability distribution, and $d_h^\pi$ on the right hand side denotes
the probability density function.

Unless otherwise specified, all integrals are over the entire domain
of the variable being integrated.

\subsection{Auxiliary Lemmas}
\begin{lemma}[Freedman's inequality \citep{Freedman1975Tail}]\label{lem:bernstein}
  Let $X_1, \dots, X_n\in\bbR$ be a martingale difference sequence with $|X_i| 
    \leq L$ almost surely for all $i$. Then for all $\eps > 0$, we have
   \begin{align*}
     \bbP\left[\sum_{i=1}^{n}X_i \geq \epsilon \right] \leq 
     \exp\left(-\frac{\epsilon^2}{\sum_{i=1}^n \En\brk*{X_i^2\mid{}X_{1:i-1}} + L\epsilon/3}\right).
   \end{align*}
 \end{lemma}
 
 \begin{lemma}[Lemma G.2 of \citet{agarwal2020pc}]\label{lem:trace_to_det}
   For $t = 1, \dots, T$, consider a sequence of symmetric matrices
   $M_t = M_{t-1} + G_t\in\bbR^{d\times{}d}$, with $M_0 \ldef
   \lambda_0 I$ for some $\lambda_0>0$, $G_t \succeq 0$, and $\nrm{G_t}_{\op} \leq 1$. Then we have that
   \begin{align*}
     2 \ln \det(M_T) - 2 \ln \det(M_0) \geq \sum_{t=1}^T \tr(G_t M_{t-1}^{-1}).
   \end{align*}
 \end{lemma}
 
 \begin{lemma}[Potential function lemma (Lemma 20 of \citet{zhang2022efficient})]\label{lem:potential}
   Consider the sequence of matrices defined in \pref{lem:trace_to_det}. If
   $\tr(G_t) \leq B^2$ for all $t\in\brk{T}$, then
   \begin{align*}
     2 \ln \det(M_T) - 2 \ln \det(M_0) \leq d \ln \prn*{1 + \frac{TB^2}{d\lambda_0}}.
   \end{align*}
 \end{lemma}


%% file: appendix_golf.tex
\begin{algorithm}[htp] 
  \caption{\golfdbr (variant of \citet{amortila2024mitigating})}
  \begin{algorithmic}[1] 
  \State \textbf{input:} Function class $\cF$, estimation policy $\psi$,
  confidence radius $\beta>0$, discretization scale $\eta>0$, filter level $\kappa$.
  \State Initialize $\cF^1 \gets \cF$, $\cD^1 \gets \emptyset$.
  \For{$t = 1, 2, \ldots, T$}
  \State $f^t \gets \argmax_{f \in \cF^t} f_1(x_1, \pi_{1;f}(x_1))$,
  where $\pi_f := \crl*{\pi_{h;f}(x_h) = \argmax_{a_h \in \cA}
    f_h(x_h,a_h)}_{h=1}^H$. 
  \State $\pi^t \gets \crl[\big]{\pi_h(x_h) = \argmax_{a_h \in \cA_\eta} f^t_h(x_h,a_h)}_{h=1}^H$.
  \State For each $h\in\brk{H}$, collect
    $(x^t_h,a^t_h,r_h\ind{t}) \sim \pi^t \circ_h \psi$ and update $\cD_h^{t+1} \gets \cD_h^t \cup \crl*{x_h^t,a_h^t,r_h^t,x_{h+1}^t}$.
  \State Compute version space:
  \begin{align*}
    &\cF^{t+1} \gets \crl*{f \in \cF^t \mid \max_{g_h \in \cF_h} \ell^t_h(f_h,f_{h+1},g_h) - 
       \ell^t_h(g_h,f_{h+1},f_h) \leq \beta,\; \forall h \in [H]},\\
    &\text{where} \quad 
    \ell^t_h(f_h,f_{h+1},f'_h) \ldef \\&\quad \quad \sum_{(x_h,a_h,r_h,x_{h+1})  \in \cD_h^{t+1}} W^{\kappa}_{f_h,f'_h}(x_h) \prn*{f_h(x_h,a_h) -
        r_h - \max_{a_{h+1} \in \cA}f_{h+1}(x_{h+1}, a_{h+1})}^2,\\
    &\text{and} \quad W^{\kappa}_{f_h,f'_h}(x_h) \ldef \indic\crl*{\abs*{f_h(x_h)-f'_h(x_h)} \geq \kappa}.
  \end{align*}
  \EndFor
  \State \textbf{return:} $\widehat \pi = \frac{1}{T} \sum_{t=1}^T \pi^t$.
  \end{algorithmic}\label{alg:golf}
  \end{algorithm}

In this section we discuss the statistical results for the \framework
framework. We begin with a description of the \golfdbr algorithm and the
formal upper bound result statement in \cref{app:golf_details}.  Then
we discuss the discrepancy in sample complexity
between~\pref{thm:golf} and the optimal sample complexity for the
classical Lipschitz MDP setting, without rich observations.
In~\pref{app:golf_proof}, we provide the proof of~\pref{thm:golf}, begining by introducing the key concept of approximate coverability (\pref{def:approximate_coverability}), a general tool for capturing the complexity of settings with misspecification. 
We prove that \pref{alg:golf}, with an appropriate choice of parameters, has low sample complexity in any MDP with bounded approximate coverability.
Finally, we prove that \framework framework indeed has bounded approximate coverability and satisfies all the other conditions for the general analysis, allowing us to derive~\pref{thm:golf} as a corollary. 

\subsection{Formal Version of \creftitle{thm:golf} and Discussion}
\label{app:golf_details}

Our upper bound result is achieved by a variant of the \golf
algorithm, \golfdbr \citep{amortila2024mitigating}. Pseudocode is
displayed in \pref{alg:golf}. \golf is a version space algorithm based
on the principle of optimistic in the face of uncertainty: in each
iteration $t$, for each $h \in [H]$, we first select the most
optimistic value function $f^t$ from the version space $\cF^t$. Then
we collect data from the policy $\pi^t$, the greedy policy with
respect to $f^t$ (restricted to taking actions in the discretized set
$\cA_\eta$): for each $h \in [H]$, we first follow $\pi^t$ for $h-1$
steps, and then we take \golfold{a random action with $\piunif$ which
  samples actions uniformly in $\cA_\eta$}\golfnew{an action with the
  estimation policy $\psi$ (similar to~\citet{du2021bilinear},
  discussed in the sequel)}, and collect the transition
  $(x_h,a_h,r_h,x_{h+1})$.  Finally, we update the version space
  $\cF^{t+1}$ by removing all the value functions that are
  inconsistent with the collected data.

  Note that the first difference between our variant and the original
  \golf algorithm is that we use an exploration policy $\psi$ to take
  a single action in the data collection process. We take $\psi$
  to be uniform over $\cA_\eta$, an $\eta$-cover of the action space,
  and discuss the need for this in detail subsequently.  The second
  difference is that all deployed policies choose actions from the
  discretized action set $\cA_\eta$.  The third difference is in how
  we update the version space, the elimination condition is based on
  the ``disagreement-based regression'' loss introduced in
  \citet{amortila2024mitigating}, which as we show later,
  carefully controls misspecification terms that arise in the
  \framework framework.

We start by stating the formal result of our upper bound.

\begin{thmmod}{thm:golf}{$'$}[PAC upper bound for \framework; formal version of
  \cref{thm:golf}]
  \label{thm:golf_formal}
Suppose \cref{ass:lipschitz,assum:decoder_realizability} hold. For any $\delta \in (0,1)$ and $\veps\in(0,1)$, with an
    appropriate choice of parameters $\beta$, $\eta$, function class $\cF$ and estimation policy $\psi$, 
  \golfdbr (\pref{alg:golf}), 
   outputs a policy $\widehat \pi$ satisfying
  $J(\pi^\ast) - J(\widehat \pi) \leq \epsilon$ with probability at least $1-\delta$, using at most 
  \begin{align*}
    \cO\prn*{\frac{H^{\widetilde d+1} \log \prn*{H|\Phi| / \delta \epsilon}} 
    {\epsilon^{\widetilde d}}},
  \end{align*}
  samples, where $\widetilde d \ldef 2\dimsa +  \dima + 2$. 
\end{thmmod}

We prove this result in \cref{app:golf_proof}. We first discuss the sample complexity guarantee.

\paragraph{Discussion}
The dependence on $\dimsa$ and $\dima$ in \cref{thm:golf} is
significantly worse than what can be achieved for Lipschitz MDPs
\citep{sinclair2019adaptive,song2019efficient, cao2020provably}. Here,
we briefly explain these differences through the lens of
pseudo-regret, defined via
\begin{align*}
  \reg(T) \ldef \sum_{t=1}^T J(\pi^\ast) - J(\pi^t).
\end{align*}
Pseudoregret is a central quantity in the analysis of \cref{alg:golf},
as the policy $\pihat$ returned has
\begin{align}
J(\pistar)-J(\pihat) =
\frac{1}{T}\reg(T);\label{eq:pac_regret}
\end{align}
 it is similar to standard regret, except it is defined with respect
 to the policy sequence $\pi\ind{1},\ldots,\pi\ind{T}$ instead of the
 data collection policies. In what follows, we highlight three factors
 that result in differences between the \framework framework and
 Lipschitz MDPs.

\paragraphi{Dependence on state covering number}
In the Lipschitz MDP setting, previous works \citep{sinclair2019adaptive,song2019efficient,
cao2020provably} have shown the optimal regret bound (ignoring
dependence on the horizon $H$ and terms that are not essential to this discussion) is 
\begin{align}
  \label{eq:nonparametric_opt}
  \reg(T) \leq \bigoht
  \prn*{\inf_{\eta>0}\crl*{\sqrt{\prn*{\frac{1}{\eta}}^{\dimsa}T}+T\eta}}
  = \bigoht\prn*{T^{1-\frac{1}{\dimsa+2}}}.
\end{align}
Above, $\eta$ corresponds to discretization level, and the optimal
choice is $\eta = T^{-\frac{1}{\dimsa+2}}$. Through
\cref{eq:pac_regret}, this also yields the optimal PAC sample
complexity $\bigoht(\veps^{-(\dimsa+2)})$. Note that
\cref{eq:nonparametric_opt} generalizes the minimax optimal rate
$\bigoht(\sqrt{SAT})$ \citep{azar2017minimax} for tabular MDPs with
$S$ states and $A$ actions, with the dependence on $SA$ replaced by
the covering number $\prn*{1/\eta}^{\dimsa}$. 

Now let us consider the rich-observation setting, and draw an analogy
to tabular rich-observation MDPs. Even though the minimax regret for
tabular MDPs with $S$ states scales with $\bigoht(\sqrt{ST})$, to the best of our knowledge the best
existing upper bound for tabular rich-observation 
MDPs (i.e., Block MDPs) is $\bigoht(\sqrt{S^2A^2T})$, which is obtained through \golf.
In particular, the two bounds of \citet{jin2021bellman,xie2023role}
both give $\reg(T) \leq \bigoht \prn[\big]{\sqrt{\prn*{SAT\beta}}}$, which has an
$SA$ factor arising from distribution shift (either Bellman-eluder
dimension or coverability). An additional $SA$ factor arises
from the confidence radius $\beta = \cO(\log |\cF|) = \bigoht(SA)$
(ignoring dependence on $\log\abs{\Phi}$),
which gives the $\bigoht(\sqrt{S^2A^2T})$ rate.

Observing that Lipschitz MDPs generalize the standard tabular setting
with $\prn*{1/\eta}^{\dimsa}$ replacing $SA$, and extending to the rich-observation setting, existing Block MDP results would suggest the bound
\begin{align*}
  \reg(T) \leq \bigoht \prn*{\inf_{\eta>0}\crl*{\sqrt{\prn*{\frac{1}{\eta}}^{2\dimsa}T}+T\eta}},
\end{align*}
for the \framework framework. This already degrades the PAC bound (when compared with Lipschitz MDPs) to 
$\bigoht(\veps^{-(2\dimsa+2)})$, with optimal tuning. Thus, it does not seem to be possible
to match the Lipschitz MDP rate (\cref{eq:nonparametric_opt}) in the \framework
setting without first closing the $(SA)^2$ vs. $SA$ gap for tabular Block MDPs.

\paragraphi{Dependence on action covering number} Another unique
property of rich-observation RL that further degrades the rate is the
dependence on the number of actions (or action space covering
number). Again, consider tabular Block MDPs. Recall that we sample a
single action from the uniform policy within each trajectory during
the data collection process. This form of randomization appears in
most existing algorithms for rich-observation RL (e.g.,
\cite{agarwal2020flambe}), and is used in \emph{every} algorithm we
are aware of for Low-Rank MDPs.
The only exception we are aware of is \golf, where the
coverability-based analysis in \cite{xie2023role} avoids uniform sampling by leveraging the structure of the emission process.
This yields a regret bound of $\bigoht(\sqrt{S^2A^2T})$, but if random actions are taken (as we do), one obtains  
\begin{align*}
  \reg(T) \leq \bigoht (\sqrt{S^2A^3T}),
\end{align*}

where the extra $\sqrt{A}$ factor arises from a standard importance weighting argument. 

For the \framework, it is unclear if one can avoid this additional factor in the same vein as~\cite{xie2023role}. 
 
As a result, we incur regret
\begin{align*}
  \reg(T) \leq \bigoht \prn*{\inf_{\eta>0}\crl*{\sqrt{ \prn*{\frac{1}{\eta}}^{2\dimsa+\dima}T}+T\eta}},
\end{align*} 
where taking $\eta = T^{-\frac{1}{2\dimsa + \dima + 2}}$ gives us the PAC bound of
$\bigoht(\veps^{-(2\dimsa + \dima + 2)})$. 

\paragraphi{Dependence on misspecification} Finally, let us remark on
the role of misspecification and highlight the improvement obtained
from disagreement-based regression~\citep{amortila2024mitigating},
which is incorporated into~\pref{thm:golf}. In~\pref{thm:golf}, we
take $\cF = \Lip \circ \Phi$. Obtaining uniform convergence with this
function class requires a discretization argument, and if we
discretize at level $\gamma>0$, we can set $\beta =
(1/\gamma)^{\dimsa} + T\gamma^2$. Unfortunately, as identified
by~\citet{amortila2024mitigating}, the standard \golf algorithm
manifests \emph{misspecification amplifciation}, an undesirable
interaction between the distribution shift parameter (approximate
coverability) and the misspecification of the function
class. Misspecification amplification implies that we cannot set
$\gamma=\eta$, as doing so results in a trivial regret of
$\bigoht\prn*{T \prn*{1/\eta}^{\dimsa} + T\eta} = \Omega(T)$.
Instead, we must choose $\gamma \neq \eta$ and this gives:

\begin{align*}
  \reg(T) \leq \bigoht \prn*{\sqrt{\prn*{\frac{1}{\eta}}^{\dimsa+\dima}\prn*{\frac{1}{\gamma}}^{\dimsa}T + T^2\gamma^2}+T\eta}.
\end{align*}
Tuning $\gamma = T^{-\frac{1}{(\dimsa+2)}}$, $\eta = T^{-\frac{2}{((\dima + \dimsa + 2)(\dimsa + 2))}}$, 
we obtain a $\bigoht(\veps^{-\dimsa^2})$ dependence.

However, this degradation from $\dimsa$ to $\dimsa^2$ due to
misspecification amplification is not
fundamental. Indeed,~\citet{amortila2024mitigating} introduce
``disagreement-based regression'' precisely to avoid this phenomenon, and using it here results in the guarantee in~\pref{thm:golf_formal}.

\subsection{Proof of
  \creftitle{thm:golf_formal}}\label{app:golf_proof}

In this section we prove \cref{thm:golf_formal}. The basic idea behind
our proof is to discretize the state and action space, then carefully
proceed with an analysis similar to the finite state/action setting. To this end, we introduce 
\emph{approximate coverage}, a new complexity measure that builds on the coverability
framework from \citet{xie2023role}, but handles the misspecification error induced 
by discretizing in the nonparametric setting.
Similarly, we introduce the approximate version of the 
completeness assumption, which is required by previous analysis 
\citep{jin2021bellman,xie2023role}.
We first introduce the definition of approximate
coverability, then show that any MDP with bounded approximate coverability is
PAC-learnable via running \pref{alg:golf} with an approximately 
complete function class. Finally, we show that the
\framework framework admits bounded approximate coverability.

\icml{In the remaining part of the appendix, when using a decoder $\phi: \cX \to \cS$ 
sometimes we
write $\phi(x,a) := (\phi(x),a)$ to avoid superfluous
parenthesis.}

\subsubsection{Approximate Coverability}
We introduce \emph{approximate coverability} as a structural property
that generalizes the coverability property from \citet{xie2023role} to
exploit the structure of a class of MDPs. Approximate coverability
involves two important concepts: the \emph{one-step-back operator} and
\emph{approximate occupancy measures}. These are introduced below.

The one-step-back operator is defined as follows:
\begin{definition}[One-step-back operator]\label{def:osb_cov}
  Fix $h\in [H]$. Given function class $\cG: \cX \times \cA \to [0,L]$ and policy class $\Pi: \cX \to \Delta(\cA)$, an operator $\osb_{h,\psi}: \cG \times \Pi \to ((\cX\times\cA) \to [0,L])$ is called a \emph{one-step-back operator} for an estimation policy $\psi$ and time step $h$ with coefficient $A$ if it satisfies the following property:
  \begin{align*}
    \forall g \in \cG,\pi \in \Pi: ~~ \En^{\pi}\brk*{\mathsf{conv}^{+}(\osb_{h,\psi}[g,\pi](x_{h-1},a_{h-1}))} \leq \En^{\pi \circ_h \psi}\brk*{A\cdot\mathsf{conv}^{+}( g_h(x_h,a_h)) } .
  \end{align*}
  where we use the notation $\mathsf{conv}^+(\cdot)$ to indicate
  that the inequality holds for all non-negative convex
  functions. 
\end{definition}

In the following, we often instantiate $\psi$ with $\piunif_\eta$, and $A = A_\eta$. 

\pref{def:osb_cov} is an abstract definition. When a one-step back operator exists, we can compare the expected value of any function from 
timestep $h$ to a related quantity at time step $h-1$. Unsurprisingly,
the construction of such an operator will rely on the dynamics of the MDP. This one-step-back
operator can be seen as a generalization of a central technique used in prior works~\citep{agarwal2020flambe,jin2021bellman,uehara2022representation}.

Next, we introduce the notion of approximate occupancy measure.
\begin{definition}[Approximate occupancy]\label{def:approximate_occupancy}
Fix $h \in [H]$, function class $\cG$, policy class
$\Pi$ and let $\osb_{h,\psi}$ be a one step back operator with
estimation policy $\psi$ and
coefficient $A$. We say that $\Xi_\eta \subset \Delta(\cX\times\cA)$
is an $\eta$-approximate occupancy class for $\Pi$ if for all $\pi$
there exists $\xi_\eta \in \Xi_\eta$ such that for all $g \in \cG$
with $\widetilde{g}_\pi = \osb_{h,\psi}[g,\pi]$:
\begin{itemize}
  \item \textbf{Average distribution shift: $d^\pi_h \to \xi_{\eta}$.}
        \begin{align*}
          \En^{\pi} \brk*{g(x_h,a_h)} \leq \En^{\xi_{\eta}} \brk[\big]{\widetilde g_{\pi}(x_h ,a_h)} + 
          \eta,
        \end{align*}
        where $\En^{\xi_\eta}[\cdot] \ldef \En_{(x,a) \sim \xi_{\eta}} [\cdot]$.
  \item \textbf{Pointwise distribution shift: $\xi_{\eta} \to d^\pi_{h}$.}
    There exists a coupling\footnote{Given two probability spaces $(\Omega_1, \cF_1, \bbP)$ and 
    $(\Omega_2, \cF_2, \bbQ)$, a coupling of $\bbP$ and $\bbQ$, $\sigma \in \Gamma(\bbP,\bbQ)$, is a joint probability measure over $(\Omega_1 \times \Omega_2)$ such that the marginal distributions are $\bbP$ and $\bbQ$.
    I.e., $\sigma(A,\Omega_2) = \bbP(A),\; \forall A \in \cF_1$, and 
    $\sigma(\Omega_1, B) = \bbQ(B),\; \forall B \in \cF_2$.} 
    $\sigma \in \Gamma(\xi_\eta,d_{h}^\pi)$ such that for all $x,a \in \cX\times\cA$ we have
    \begin{align*}
      \xi_\eta(x,a)\widetilde{g}_\pi^2(x,a) \leq \int\sigma(x,a,\widetilde{x},\widetilde{a})\widetilde{g}_\pi^2(\widetilde{x},\widetilde{a}) \bm{\widetilde{x},\widetilde{a}} + \xi_\eta(x,a)\eta^2.
    \end{align*}
\end{itemize}
\end{definition}

Approximate occupancy measures formalize discretization over the
continuous state-action space, allowing for a controlled
discretization error. Introducing approximate occupancy measures
allows us to appropriately control distribution shift, via the
following notion of approximate coverability.
\begin{definition}[Approximate coverability]\label{def:approximate_coverability}
    The approximate coverability of an MDP 
    with respect to $\Xi_\eta$ that satisfies \pref{def:approximate_occupancy}
    (with any $\osb$, $\cG$, $\Pi$, $\psi$, and $A$) is defined as
    \begin{align*}
        \Ccov(\Xi_\eta) := \inf_{\mu_1, \ldots, \mu_H \in \Delta(\cX \times \cA)}
        \sup_{\xi_\eta \in \Xi_\eta, h \in [H]} 
        \nrm*{\frac{\xi_{\eta;h}}{\mu_h}}_{\infty}.
    \end{align*}
  \end{definition}
Above, the function class $\cG,\Pi$ as well as the exploration policy
$\psi$ and coefficient $A$ are implicit. We instantiate these objects in
the sequel.  The essential difference between approximate coverability
and coverability \citep{xie2023role} is that in our definition, the
supremum is taken over approximate occupancy measures
(\pref{def:approximate_occupancy}), instead of occupancy
measures. This is a natural generalization of the coverability
property, as it allows us to handle misspecification error induced
by discretization in the nonparametric setting. 

\subsubsection{Analysis under Approximate Coverability}

We now give a general analysis of the \golfdbr algorithm under
approximate coverability. A basic assumption that is required by all
previous analyses is the realizability and completeness of the value
function class $\cF$.  Below we state a more general misspecified
version of this assumption.  As we show later, with the class of
Lipschitz function, the \framework framework satisfies these
properties.

First, for any function class $\cF: \cX \times \cA \to \bbR$,
we define the reward-based Bellman backup operator 
$\cT_h: \cF_{h+1} \to \cF_h$, where
$\cT_h[f_{h+1}](x,a) = R_h(x,a) + 
  \En \brk*{\max_{a_{h+1}} f_{h+1}(x_{h+1},a_{h+1}) \mid x_h = x, a_h = a}.$

\begin{assumption}[Approximate realizability and completeness]\label{assum:completeness}
  We assume the function class $\cF_\gamma$ (given as an argument to
  \golfdbr) is $\gamma$ misspecified in the $\ell_\infty$ norm: For
  all $h \in [H]$, we have $f^\ast_{\infty;h} \in \cF_{\gamma;h}$ such
  that $\nrm*{f^\ast_{\infty;h} - Q^\ast_h}_{\infty} \leq
  \gamma$. Also, we have for any $f_{h+1} \in \cF_{\gamma;h+1}$, there
  exists $f_{h} \in \cF_{\gamma;h+1}$ such that $\nrm*{f_h - \cT_h
    \brk*{f_{h+1}}}_{\infty} \leq \gamma$. Furthermore, we assume
  there exists a finite subset of actions $\cA_\gamma$ such that for
  all $h \in [H]$, for any $x \in \cX$ and $f_h \in \cF_h$, $\max_{a
    \in \cA}f_h(x,a) - \max_{a \in \cA_\gamma}f_h(x,a) \leq \gamma$.
\end{assumption}

\golfold{
\begin{assumption}[Realizability and completeness]\label{assum:completeness}
  We assume the function class $\cF$ that \golfold{\golf}\golfnew{\golfdbr} takes
  satisfies realizability and completeness. That is, 
  For all $h \in [H]$, we have $Q^\ast_h \in \cF_h$. Furthermore,
  we have $\forall f_{h+1} \in \cF_{h+1}$, $\cT_h [f_{h+1}] \in \cF_h$.
\end{assumption}
}

The next results gives a sample complexity bound for \golfdbr for any MDP that satisfies \pref{def:approximate_coverability} and any function that satisfies \pref{assum:completeness}.

\begin{lemma}\label{lem:golf_cov}
Consider any MDP that satisfies \pref{def:approximate_coverability} 
with approximate occupancy $\Xi_\eta$, and parameter $\eta \in (0,1)$, $\psi$ 
and $A$ for classes $\cF_\gamma,\Pi=\{\pi_f: f \in \cF_\gamma\}$. Then running \golfdbr (\pref{alg:golf}) with function $\cF_\eta$ that 
satisfies \pref{assum:completeness} and 
  $\beta = c \log\prn*{TH \abs{\cF_{\eta}} /\delta}$,
outputs a sequence of policies $\pi^1, \ldots, \pi^T$, such that with probability at least $1-\delta$, 
    \begin{align*}
      \reg(T) \ldef\sum_{t=1}^{T}J(\pistar)-J(\pi\ind{t}) \leq \cO \prn*{H\sqrt{ A \ccov(\Xi_\eta) T\beta \log(T)}+TH\eta},
    \end{align*}
where $\ccov(\Xi_\eta)$ is the approximate coverability parameter defined in \pref{def:approximate_coverability}.
\end{lemma}
\begin{proof}[\pfref{lem:golf_cov}]
  For each $f^t$ induced by the algorithm, for each $h \in [H]$, let 
  $\apx\brk*{f^t_{h+1}} \in \cF_{\eta,h}$ be the $\ell_\infty$ approximation 
  of $\cT_h \brk*{f^t_{h+1}}$ such that 
  $\nrm*{\apx\brk*{f^t_{h+1}} - \cT_h \brk*{f^t_{h+1}}}_{\infty} \leq \eta$. 
  Note that this is always possible due to \pref{assum:completeness}. 
  Then we define
  \begin{align*}
    \delta^t_{h} \ldef f^t_h - \apx\brk*{f^t_{h+1}} \mathand \delta'^{t}_h \ldef \apx \brk*{f^t_{h+1}} - \cT_h\brk*{f^t_{h+1}}.
  \end{align*}
  Under completeness assumption (\pref{assum:completeness}), we have
  \begin{align*}
      \reg(T) &= \sum_{t=1}^T J(\pistar)-J(\pi\ind{t})\\
      &\leq \sum_{t=1}^T \abs*{f^t_1(x_1, \pi^t_{1}(x_1))-J(\pi\ind{t})} + 2T\eta 
      \tag{Almost optimism (\pref{lem:golf_optimism}) with $\gamma=\eta$} \\
      &\leq \sum_{t=1}^T \sum_{h=1}^H \abs*{\En^{\pi^t_h}\brk*{ f_h - \cT_h \brk*{f_{h+1}}}} + 3T\eta \tag{\pref{lem:golf_decomposition}} \\
      &= \sum_{t=1}^T \sum_{h=1}^H \abs*{\En^{\pi^t_h}\brk*{ \delta^t_h + \delta'^t_h }} + 3T\eta \\
      &\leq \sum_{t=1}^T \sum_{h=1}^H \abs*{\En^{\pi^t_h}\brk*{ \delta^t_h}} + 4T\eta \\
      &\leq \sum_{t=1}^T \sum_{h=1}^H \En^{\pi^t_h}\brk*{ \abs*{\indic\crl*{\delta^t_h \geq 3\eta}\delta^t_h}} + 7T\eta.
  \end{align*}
  Now define $\cG_h = \abs*{\cF_{\eta;h} - \cT_h \brk*{\cF_{\eta;h+1}}}$, we see first that
  $\cG_h$ is bounded by 2, and second $\abs*{\indic\crl*{\delta^t_h \geq 3\eta}\delta^t_h} \in \cG_h$, then by \pref{def:osb_cov}, let
 $\widetilde \delta^t_{h-1} = \osb_h(\abs*{\indic\crl*{\delta^t_h \geq 3\eta}\delta^t_h}, \pi^t_h)$ be defined with respect to the test policy $\psi$, we have 
 \begin{align*}
  \sum_{t=1}^T \sum_{h=1}^H \En^{\pi^t_h}\brk*{ \abs*{\indic\crl*{\delta^t_h \geq 3\eta}\delta^t_h}} \leq \sum_{t=1}^T \sum_{h=1}^H \En^{\xi^t_{\eta;h-1}}\brk*{\widetilde \delta^t_{h-1}} + TH \eta,
 \end{align*}
  where $\xi^t_{\eta;h-1}$ is the distribution over $(\cX \times \cA)$ from the family of distributions $\Xi_\eta$ 
  defined in \pref{def:approximate_occupancy} such that the average distribution shift
  property holds for $\pi^t_h$. 
  We can now
  proceed by adapting the coverability analysis in \citet{xie2023role}. We first
  define 
  \begin{align}\label{eq:accumulated_distribution}
      \bar \xi^t_{\eta;h}(x,a) := \sum_{i=1}^{t-1} \xi^i_{\eta;h}(x,a), \mathand
      \mu^\ast_h = \argmin_{\mu_h \in \Delta(\cX \times \cA)} \sup_{\xi_\eta \in \Xi_\eta}
      \nrm*{\frac{\xi_{\eta;h}}{\mu_h}}_{\infty}, 
  \end{align}
  and define the ``burn-in'' time
  \begin{align*}
      \tau_h(x,a) = \min\crl*{t \mid \bar \xi^t_{\eta;h}(x,a) \geq \Ccov(\Xi_\eta) \cdot \mu^\ast_h(x,a)}.
  \end{align*}
Similar to the definition of $\bar \xi^t_{\eta;h}$, we also define 
\begin{align*}
\bar d^{\pi^t}_{h}(x,a) = \sum_{i=1}^{t-1} d^{\pi^i}_{h}(x,a).
\end{align*}
Finally we will use the following shorthand to denote the data collection distribution
in \pref{alg:golf}: at iteration $t$, recall for each $h \in H$, the data collection policy is $\pi^t \circ_h \psi$, and denote the data collection distribution as 
$d^t_h(x,a) = d^{\pi^t \circ_h \psi}_h(x,a)$, then we denote the sum of historical data collection distribution as
\begin{align}\label{eq:data_collection_dist}
\bar d^t_h(x,a) = \sum_{i=1}^{t-1} d^i_h(x,a).
\end{align}

We can then decompose the regret for each layer $h$ as 
  \begin{align*}
      \sum_{t=1}^T \En^{\xi^t_{\eta;h-1}} \brk*{ \widetilde \delta^t_{h-1}} = 
      \underbrace{\sum_{t=1}^T \En^{\xi^t_{\eta;h-1}} \brk*{ \widetilde \delta^t_{h-1} \indic\crl*{t < \tau_{h-1}}}}_{\text{burn-in term}} + 
      \underbrace{\sum_{t=1}^T \En^{\xi^t_{\eta;h-1}} \brk*{\widetilde \delta^t_{h-1} \indic\crl*{t \geq \tau_{h-1}}}}_{\text{stable term}}.
  \end{align*}
  For the burn-in term, we have by the construction that $|\widetilde \delta^t_{h-1}| \leq 1$, then 
  \begin{align*}
      \sum_{t=1}^T \En^{\xi^t_{\eta;h-1}} \brk*{ \widetilde \delta^t_{h-1} \indic\crl*{t < \tau_{h-1}}} \leq 
      \int \sum_{t < \tau_{h-1}(x,a)} \xi^t_{\eta;h-1}(x,a) \bm{x,a} =
      \int \bar \xi^{\tau_{h-1}(x,a)}_{\eta;h-1}(x,a) \bm{x,a}.
  \end{align*}
  Then by the definition of accumulated distribution (\pref{eq:accumulated_distribution}), we have
  \begin{align*}
    \int \bar \xi^{\tau_{h-1}(x,a)}_{\eta;h-1}(x,a)  \bm{x,a}&= 
    \int \bar \xi^{\tau_{h-1}(x,a)-1}_{\eta;h-1}(x,a)  \bm{x,a}+ 
    \int \xi^{\tau_{h-1}(x,a)-1}_{\eta;h-1}(x,a)  \bm{x,a}\\
    &\leq 2\Ccov(\Xi_\eta) \sum_{x,a} \mu^\ast_h(x,a) = 2\Ccov(\Xi_\eta),
  \end{align*}
  where the last inequality is due to the definition of $\tau_h(x,a)$ and $\Ccov(\Xi_\eta)$.
  Thus we proved that the burn-in term is bounded by $2\Ccov(\Xi)$. For the
  stable term, recall the definition of $\sigma_{\eta;h}$ the coupling 
  between $\xi_{\eta;h}$ and $d^{\pi}_h$ that satisfies the pointwise 
  distribution shift property, we have 
\begin{align*}
&~~~\sum_{t=1}^T \En^{\xi^t_{\eta;h-1}} \brk*{\widetilde \delta^t_{h-1} \indic\crl*{t \geq \tau_{h-1}}}\\
&= \sum_{t=1}^T \int \xi^t_{\eta;h-1}(x,a) 
\prn*{\frac{\bar \xi^t_{\eta;h-1}(x,a)}{\bar \xi^t_{\eta;h-1}(x,a)}}^{\frac{1}{2}}
 \widetilde \delta^t_{h-1}(x,a)\indic\crl*{t \geq \tau_{h-1}}  \bm{x,a}\\
&= \sum_{t=1}^T \int \sqrt{\frac{\prn*{\xi^t_{\eta;h-1}(x,a) 
\indic\crl*{t \geq \tau_{h-1}}}^2}{\bar \xi^t_{\eta;h-1}(x,a)}}
\sqrt{\bar \xi^t_{\eta;h-1}(x,a) \widetilde  \delta^{t^2}_{f;\pi;h-1}(x,a)}  \bm{x,a} \\
&\leq \sum_{t=1}^T \int \sqrt{\frac{\prn*{\xi^t_{\eta;h-1}(x,a) 
\indic\crl*{t \geq \tau_{h-1}}}^2}{\bar \xi^t_{\eta;h-1}(x,a)}}
\sqrt{2\int\bar \sigma^t_{h-1}(\widetilde x,\widetilde a,x,a) 
\widetilde  \delta^{t^2}_{h-1}(\widetilde x,\widetilde a)
+ 2\xi^t_{\eta;h-1}(x,a) \eta^2 \bm{\widetilde x,\widetilde a}} \bm{x,a}  \tag{Point-wise distribution shift (\cref{def:approximate_occupancy})}\\
&\leq \sum_{t=1}^T \int\sqrt{\frac{\prn*{\xi^t_{\eta;h-1}(x,a) 
\indic\crl*{t \geq \tau_{h-1}}}^2}{\bar \xi^t_{\eta;h-1}(x,a)}}
\sqrt{2\int\bar \sigma^t_{h-1}(\widetilde x,\widetilde a,x,a) 
\widetilde  \delta^{t^2}_{h-1}(\widetilde x,\widetilde a) \bm{\widetilde x,\widetilde a}} \bm{x,a} 
 + 2T\eta  \tag{Normalization}\\
&\leq \sqrt{\sum_{t=1}^T \int \frac{\prn*{\xi^t_{\eta;h-1}(x,a) 
\indic\crl*{t \geq \tau_{h-1}}}^2}{\bar \xi^t_{\eta;h-1}(x,a)} \bm{x,a}}
\sqrt{\sum_{t=1}^T \int
2\bar \sigma^t_{h-1}(\widetilde x,\widetilde a,x,a) 
\widetilde  \delta^{t^2}_{h-1}(\widetilde x,\widetilde a) 
\bm{\widetilde x,\widetilde a}\bm{x,a}}  + 2T\eta \tag{Cauchy-Schwarz}\\
&= \sqrt{\sum_{t=1}^T \int \frac{\prn*{\xi^t_{\eta;h-1}(x,a) 
\indic\crl*{t \geq \tau_{h-1}}}^2}{\bar \xi^t_{\eta;h-1}(x,a)} \bm{x,a}}
\sqrt{\sum_{t=1}^T \int 2\bar d^{\pi^{t}}_{h-1}(x,a)
\widetilde  \delta^{t^2}_{f;\pi;h-1}( x, a) \bm{x,a}} 
+ 2T\eta \tag{Marginalization of $\sigma_{\eta,h}$}.
\end{align*}

Now define 
\begin{align*}
  \err^t_h = \ldef \indic\crl*{\abs*{f^t_h - \apx(f^t_{h+1})} \geq 3\gamma}
  \cdot \crl*{\prn*{f^t_h - \cT_h f^t_{h+1}}^2 - \prn*{\apx[f^t_{h+1}] - \cT_h \brk*{f^t_{h+1}}}^2},
\end{align*}
the term inside the second square root can be bounded by 
\begin{align*}
  \sum_{t=1}^T \int 2\bar d^{\pi^{t}}_{h-1}(x,a)
  \widetilde  \delta^{t^2}_{h-1}( x, a) \bm{x,a} &\leq 
  \sum_{t=1}^T \int 2\bar d^{t}_{h}(x,a) A\crl*{
  \indic^2\crl*{{\delta^t(h) \geq 3\eta}} \cdot \delta^{t^2}_{h}(x, a)}\bm{x,a} \\
  &\leq 6A\sum_{t=1}^T \int \bar d^{t}_{h}(x,a) \err^t_h(x,a)\bm{x,a},
\end{align*}
where the first line is by the property of the one-step-back operator (\pref{def:osb_cov}). Plugging it back into the above we have:
\begin{align*}
  &\sum_{t=1}^T \En^{\xi^t_{\eta;h-1}} \brk*{\widetilde \delta^t_{h-1} \indic\crl*{t \geq \tau_{h-1}}} \leq\\
   & \underbrace{\sqrt{\sum_{t=1}^T \int \frac{\prn*{\xi^t_{\eta;h-1}(x,a)
  \indic\crl*{t \geq \tau_{h-1}}}^2}{\bar \xi^t_{\eta;h-1}(x,a)} \bm{x,a}}}_{\text{A: Distribution shift}} 
  \underbrace{\sqrt{\sum_{t=1}^T 6A \int \bar d^t_h(x,a)
  \err^{t}_{h}(x,a) \bm{x,a}}}_{\text{B: In-sample error}} 
  + 2T\eta.
\end{align*}
The distribution shift term is bounded by $O(\sqrt{\Ccov(\Xi)
\log(T)})$, which follows from an identical argument to Theorem 1 of
\citet{xie2023role}, and the
in-sample error term is bounded by $O(\sqrt{\beta T})$ by invoking \pref{lem:golf_concentration} with our choice of $\beta$. Thus, we have that 
\begin{align*}
  \reg(T) \leq O\prn*{H\sqrt{A \Ccov(\Xi_\eta) \beta T \log(T)} + TH\eta}.
\end{align*}
\end{proof}

\subsubsection{Instantiating the Regret Bound for \framework}

Having proven \cref{lem:golf_cov}, it remains to show that 
we can construct a function class that satisfies the approximate 
realizability and completeness assumption (\pref{assum:completeness}) for the \framework framework
and 
\framework framework with Lipschitz dynamics (\pref{ass:lipschitz}) has bounded approximate 
coverability.

Before we prove the realizability and completeness result, we first 
introduce the notation for the covering set of a function class. 
Given a function class $\cF$, we use $\cN_{\infty}(\cF;\gamma)$ to 
denote the $\gamma$-covering set of the function class $\cF$, that 
is, for any $f \in \cF$, there exist $f' \in \cN_{\infty}(\cF;\gamma)$ such that $\nrm*{f-f'}_{\infty} \leq \gamma$. 
We denote the size of the covering set $N_{\infty}(\cF;\gamma)
\ldef \abs*{\cN_{\infty}(\cF;\gamma)}$.

\begin{lemma}[Realizability and completeness]\label{lem:realizability}
Consider the \framework framework and suppose \cref{ass:lipschitz,assum:decoder_realizability} hold.
  Let $\Lip_2: \cS \times \cA \to [0,1]$ be the set of all
  $\brk{0,1}$-bounded 2-Lipschitz functions. Define $\cF := \Lip_2 \circ \Phi$, and then $\cN_{\infty}(\cF,\gamma)$ satisfies
  the approximate realizability and completeness properties
  in \pref{assum:completeness}.
\end{lemma}
\begin{proof}[\pfref{lem:realizability}]
All the results follows if we prove $\cF = \Lip_2 \circ \Phi$
satisfies exact realizability and completeness.
The realizability result is just a combination of \pref{assum:decoder_realizability} 
and a standard fact from the non-parametric RL literature that Lipschitzness 
in dynamics implies Lipschitzness in $Q^\ast$ (see, e.g., Lemma 2.4 of 
\citet{sinclair2019adaptive}). For the completeness results, we recall that
\begin{align*}
  \cT_h[f_{h+1}](x,a) &= R_h(x,a) + \En \brk*{
    \max_{a_{h+1}}f_{h+1}(x_{h+1},a_{h+1}) \mid
                        x_h = x, a_h = a}\\
  &= R_h(\phistar_h(x),a) + \En \brk*{\max_{a_{h+1}}
  f_{h+1}(x_{h+1},a_{h+1}) \mid s_h = \phistar_h(x), a_h = a}.
\end{align*}
This implies that we
just require Lipschitzness of the function
\begin{align*}
  (s,a)\mapsto{}\cT_h[f_{h+1}](s,a) \ldef R_h(s,a) + 
  \En \brk*{\max_{a_{h+1}}f_{h+1}(x_{h+1},a_{h+1}) \mid s_h = s, a_h = a}.
\end{align*}
By \pref{ass:lipschitz}, $R_h$ is $1$-Lipschitz in $(s,a)$, and since $f_{h+1}$ is bounded by 1,
$(s,a)\mapsto\En [\max_{a_{h+1}}f_{h+1}(x_{h+1},a_{h+1}) \mid s_h = s, a_h = a]$ is $1$-Lipschitz,
so that $\cT_h[f_{h+1}](s,a)$ is 2-Lipschitz. By \pref{assum:decoder_realizability}
and our construction of $\cF$, this proves the completeness.

Finally, by realizing any function in $\cF$ is Lipschitz with 
respect to the action, then
for all $h \in [H]$, for any $x \in \cX$ and $f_h \in \cF_{\gamma,h}$, we have $\max_{a \in \cA}f_h(x,a) - \max_{a' \in \cA_\gamma}f_h(x,a') \leq \gamma$, where $\cA_\gamma$ is a $\gamma$-covering set of the action space.
\end{proof}

Next we will prove that the \framework framework indeed has 
a bounded approximate coverability.

\begin{lemma}\label{prop:coverability}
  For the \settingname framework, there exists one-step-back functions 
  $\{\osb_h\}_{h=1}^H$ that satisfies \pref{def:osb_cov}, and  
  there exists a distribution family $\Xi_\eta$
   which satisfies the conditions of \cref{def:approximate_occupancy} and 
   $\Xi_\eta$ has finite approximate coverability:
    $\Ccov(\Xi_\eta) \leq \dimeta := \prn*{\frac{1}{\eta}}^{\dimsa}$.
\end{lemma}

\begin{proof}[\pfref{prop:coverability}] 
The proof consists of three major parts:
we first construct the mapping $\osb_h$ that satisfies \pref{def:osb_cov},
and then the construction of the family of 
distributions $\Xi_\eta$ that satisfies all the properties in \pref{def:approximate_occupancy},
and finally we show with the constructed $\Xi_\eta$, 
 $\Ccov(\Xi_\eta)$ is bounded for any $\eta$.
\paragraph{Construction of $\osb$}
For all $h \in [H]$, take $\cG_h: \cX \times \cA \to [0,L]$ to be any positive bounded function class and for any $g_h \in \cG$, we have 
\begin{align*}
  \En^{\pi} \brk*{g_h}
  &\leq \int  d^\pi_h(x_h,a_h) g_h (x_h,a_h) \bm{x_h,a_h} \\
  &= \int   d^\pi_{h-1}(x_{h-1},a_{h-1}) 
  \underbrace{\int  P_{h-1}(x_h \mid x_{h-1},a_{h-1}) \pi_h(a_h \mid x_h) g_h (x_h,a_h) \bm{x_h,a_h}
  }_{\rdef\widetilde g_{\pi;h-1}(x_{h-1},a_{h-1})} \bm{x_{h-1},a_{h-1}},
\end{align*}
thus we have the construction that 
\begin{align*}
  \osb_h(g_{h}, \pi_h) = \int  P_{h-1}(x_h \mid x_{h-1},a_{h-1}) \pi_h(a_h \mid x_h) g_h (x_h,a_h) \bm{x_h,a_h}.
\end{align*}
By construction, we can easily verify that this choice satisfies
\begin{align*}
    \En^{\pi} \brk[\big]{\mathsf{conv^+}(\widetilde g_{\pi;h-1})} \leq \En^{\pi \circ_h \unifpi} 
    \brk*{A_\eta \cdot \mathsf{conv^+}(g_{h})}
  \end{align*}
by Jensen's inequality and importance weighting.
We can also this inequality can be extended to any policy $\psi$ such that 
$\sup_{x,a}\frac{1}{\psi(a \mid x)} = A < \infty$, and we have 
\begin{align*}
  \En^{\pi} \brk[\big]{\mathsf{conv^+}(\widetilde g_{\pi;h-1})} \leq \En^{\pi \circ_h \psi} 
  \brk*{A \cdot \mathsf{conv^+}(g_{h})}.
\end{align*}

\paragraph{Construction of $\Xi_\eta$}
Now let us consider any fixed $\eta > 0$, $f \in \cF$ and 
$\pi \in \Pi$. We now observe that our construction for $\widetilde g$ that satisfies
the following property: if the ground-truth latents for two observations $x$ and 
$x'$ are close to each other, i.e., $D_{\cS}(\phi^\ast_h(x), \phi^\ast_h(x'))\leq \eta$,
then by to the Lipschitz continuity of $P_h$ with respect to $\cS$, we have 
$\abs*{\widetilde g_{\pi;h}(x,a) - \widetilde g_{\pi;h}(x',a)} \leq \eta$
for any $f \in \cF$, $\pi \in \Pi$. Note that this property is independent of the 
original function class $\cF$ or $\cG$, but instead it is a property inherited from
the dynamics of the MDPs from the \settingname framework.
This property hints at a natural construction of $\xi$: instead of focusing 
on each observation, there should be a distribution that ``groups''
the support (probability of being visited by $\pi$)
where the latent states are close to each other without changing the value of 
the expectation of $\widetilde g_{\pi;h}$ under $d^\pi$  by too much.
Recall the definition of $\ball{\eta}[\phi^\ast_h](x,a)$,
\begin{align*}
    \ball{\eta}[\phi^\ast_h](x,a) := \{\widetilde x, \widetilde a \in \cX \times \cA 
  \mid \disc{\eta}[\phi^\ast_h](\widetilde x, \widetilde a) = \disc{\eta}[\phi^\ast_h](x,a)\},
\end{align*}
which is the set of observations whose corresponding latent states are $\eta$-close to each other 
in the latent space. Given a policy $\pi$ and an observation-space ball 
$\ball{\eta}[\phi^\ast_h](x,a)$, we can define 
$d^{\pi}_h(\ball{\eta}[\phi^\ast_h](x,a))$ be the probability of $\pi$ visiting
the ball $\ball{\eta}[\phi^\ast_h](x,a)$ at timestep $h$, and define 
$\ballunif{x,a}$ be the uniform density over the $\ball{\eta}[\phi^\ast_h](x,a)$.
Then we can define the following distribution $\xi_{\eta,h}$ and
  coupling $\sigma_{\eta,h}$:
\begin{align*}
  &\sigma_{\eta,h}(x,a,\widetilde x, \widetilde a) := 
  \indic\crl*{\widetilde x, \widetilde a \in \ball{\eta}[\phi^\ast_h](x,a)} 
  \ballunif{x,a} d^\pi_h(\widetilde x, \widetilde a) \\
  &\xi_{\eta,h}(x,a) := \int   \sigma_{\eta,h}(x,a,\widetilde x, \widetilde a)
  \bm{\widetilde x, \widetilde a}
  = \ballunif{x,a} d^\pi_h(\ball{\eta}[\phi^\ast_h](x,a)).
\end{align*}
That is, 
$\xi$ assigns equal density to all the observations within each observation ball.
We can also verify that $d^\pi_h(\widetilde x, \widetilde a) = \int
\sigma_{\eta,h}(x,a,\widetilde x, \widetilde a) \bm{x,a}$ as required. This concludes the construction
of $\xi$.

\paragraphi{Average distribution shift: $d^\pi_h \to \xi_{\eta,h}$} From the
construction above, it follows that
\begin{align*}
  & \int d^\pi_h(x,a) \widetilde g_{\pi;h}(x,a) \bm{x,a}\\
= &\int  \ballunif{x,a}
\int_{\widetilde x, \widetilde a \in \ball{\eta}[\phi^\ast_h](x,a)} d^\pi_h(\widetilde x, \widetilde a) 
\widetilde g_{\pi;h}(\widetilde x, \widetilde a)
\bm{\widetilde x, \widetilde a} \bm{x,a}\\
\leq & \int  \ballunif{x,a}
\int_{\widetilde x, \widetilde a \in\ball{\eta}[\phi^\ast_h](x,a)} d^\pi_h(\widetilde x, \widetilde a)
\widetilde g_{\pi;h}( x, a)\bm{\widetilde x, \widetilde a} \bm{x,a}  +  \eta\\
=&  \int \xi_{\eta,h}(x,a) \widetilde g_{\pi;h}(x,a) \bm{x,a} + \eta.
\end{align*}

\paragraphi{Pointwise distribution shift: $\xi_{\eta,h} \to d^\pi_h$}
To prove the second distribution shift property, we have for all $(x,a) \in \cX \times \cA$,
\begin{align*}
  \xi_{\eta,h}(x,a) g_{\pi;h}^2(x,a) &= \ballunif{x,a}
  \int_{\widetilde x, \widetilde a \in\ball{\eta}[\phi^\ast_h](x,a)} d^\pi_h(\widetilde x, \widetilde a) 
  \widetilde g_{\pi;h}^2(x, a) \bm{\widetilde x, \widetilde a} \\
  &\leq \ballunif{x,a}
  \int_{\widetilde x, \widetilde a \in\ball{\eta}[\phi^\ast_h](x,a)} d^\pi_h(\widetilde x, \widetilde a) 
  2\prn*{\prn*{\widetilde g_{\pi;h}(\widetilde x, \widetilde a) - \widetilde  g_{\pi;h}( x,  a)}^2
  +  \widetilde g_{\pi;h}^2(\widetilde x, \widetilde a)} \bm{\widetilde x, \widetilde a} \tag{AM-GM} \\
  &\leq 2\ballunif{x,a} \int_{\widetilde x, \widetilde a \in \ball{\eta}[\phi^\ast_h](x,a)}
  d^\pi_h(\widetilde x, \widetilde a) \widetilde g_{\pi;h}^2(\widetilde x, \widetilde a) \bm{\widetilde x, \widetilde a} 
  + 2\xi_{\eta,h}(x,a) \eta^2\\
  &= 2 \int   \sigma_{\eta,h}(\widetilde x, \widetilde a, x,a) 
  \widetilde g_{\pi;h}^2(\widetilde x, \widetilde a) \bm{\widetilde x, \widetilde a}
   + 2\xi_{\eta,h}(x,a) \eta^2.
\end{align*}

This finishes the construction of $\Xi$. 
\paragraph{Bounding $\Ccov(\Xi)$}
Recall the definition of $\Ccov(\Xi)$:
\begin{align*}
  \Ccov(\Xi_\eta) := \inf_{\mu_1, \ldots, \mu_H \in \Delta(\cX \times \cA)}
  \sup_{\xi_\eta \in \Xi_\eta, h \in [H]} 
  \nrm*{\frac{\xi_{\eta;h}}{\mu_h}}_{\infty}.
\end{align*}
Our goal is to construct $\mu_h$ for each $h \in [H]$ such that $\Ccov(\Xi_\eta)$
is bounded.
We first make the following observation: by the construction of $\xi$, we can see that it is the same as 
the distribution of a tabular Block MDP where the states are the collections of states 
in each ball on the level of discretization $\eta$. 
Follow this intuition, recall the notation $\cB_\eta :=
\crl*{\ball{\eta}(s_\eta,a_\eta) : s_\eta, a_\eta \in \cS_\eta
\times\cA_\eta}$, which defines the set of all $\eta$-balls in the 
latent space. Now fix $h \in [H]$, for each ball $b_\eta \in \cB_\eta$, 
we define the policy that maximizes its probability of visiting this ball:
\begin{align*}
  \pi_{b_\eta} = \argmax_{\pi \in \Pi} \bbP^\pi(b_\eta).
\end{align*}
Then we can construct the following distribution $\mu_h$:
\begin{align*}
  \mu_h(x,a) = \frac{1}{S_\eta A_\eta} \sum_{b_\eta \in \cB_\eta} \upsilon(x \mid b_\eta) d^{\pi_{b_\eta}}_h(b_\eta),
\end{align*}
that is, the distribution $\mu_h$ takes the uniform distribution over the grid actions and for each observation, it takes the average of the ``ball-reaching'' policy's occupancy measure over each ball (similar to how we design $\xi$), and take the uniform distribution over all the ``ball-reaching'' policy's occupancy measure. Now given any $x \in \cX$, and policy $\pi$ and its corresponding $\xi_{\eta,h}$, we define the approximate occupancy measure on observation:
\begin{align*}
  \xi_{\eta,h}(x,a) = \upsilon(x \mid \ball{\eta}[\phi^\ast_h](x)) d^\pi_h(\ball{\eta}[\phi^\ast_h](x)) \pi_h(a \mid x),
\end{align*}
i.e., the average occupancy measure of the ball $\ball{\eta}[\phi^\ast_h](x)$ under policy $\pi$.
Then denote $b_\eta = \ball{\eta}[\phi^\ast_h](x)$, we have that 
\begin{align*}
  \frac{\xi_{\eta,h}(x,a)}{\mu_h(x,a)} \leq S_\eta A_\eta \cdot \frac{\upsilon(x \mid b_\eta) d^\pi_h(b_\eta)}{\upsilon(x \mid b_\eta) d^{\pi_{b_\eta}}_h(b_\eta)} = S_\eta A_\eta \cdot \frac{d^\pi_h(b_\eta)}{d^{\pi_{b_\eta}}_h(b_\eta)} \leq S_\eta A_\eta,
\end{align*}
where the last inequality follows the definition of $\pi_{b_\eta}$ such that 
$\pi_{b_\eta} = \argmax_{\pi \in \Pi} \bbP^\pi(s_h \in b_\eta)$. And the proof is completed by the fact that $S_\eta A_\eta \leq \prn*{\frac{1}{\eta}}^{\dimsa}$.
\end{proof}

With \pref{lem:golf_cov} and \pref{prop:coverability}, 
we are ready to prove \pref{thm:golf_formal}.

\golfold{
\begin{proof}[Proof of \pref{thm:golf_formal}]
By \pref{lem:golf_cov}, we have that \golf satisfies that 
\begin{align*}
    \reg(T) \leq \cO\prn*{H\sqrt{\Ccov(\Xi_\eta)A_\eta \beta T \log(T)} + TH\eta},
\end{align*}
as long as we take
\begin{align*}
  \beta = T^{\frac{\dimsa}{\dimsa+2}} \log(\abs*{\Phi}  TH/\delta),
\end{align*} 
and choose function class 
$\cF = \Lip \circ \Phi$ such that completeness holds (\pref{lem:realizability}), 
and thus the event in \pref{lem:golf_concentration} holds with probability at least
$1-\delta$ for all $t\in\brk{T}$. 

Plugging in this choice for 
$\beta$ and the bound on $\Xi_\eta$ from \pref{prop:coverability}, we have with probability as least $1-\delta$,
\begin{align*}
    \reg(T) \leq O\prn*{H\sqrt{(1/\eta)^{\dimsa+\dima}  T^{\frac{\dimsa}{\dimsa+2}} \log(TH|\Phi|(1/\gamma)^{\dimsa}/\delta)} + TH\eta}.
\end{align*}
To conclude, we choose $\eta = T^{-\frac{2}{(\dima + \dimsa + 2)(\dimsa + 2)}}$,
which gives
\begin{align*}
    \reg(T) \leq \cO\prn*{HT^{\frac{\widetilde d}{\widetilde d + 2}} \sqrt{\log(TH|\Phi|/\delta)}},
\end{align*}
for $\widetilde d \ldef \dimsa^2 + \dimsa \dima + 4\dimsa + 2\dima + 2$.
Finally, we convert this to a PAC bound. We have that with 
\begin{align*}
    T = \bigoh\prn*{H^{\frac{\widetilde d + 2}{2}} \log \prn*{TH|\Phi| / \delta \epsilon} / 
    \epsilon^{\frac{\widetilde d + 2}{2}}}
\end{align*}
sufficiently large, the policy $\pihat$ satisfies
\begin{align*}
    J(\pi^\ast) - J(\widehat \pi) \leq \epsilon.
\end{align*}
\end{proof}
}

\golfnew{
  \begin{proof}[Proof of \pref{thm:golf_formal}]
    By \pref{lem:golf_cov}, we have that \golfdbr satisfies that 
    \begin{align*}
        \reg(T) \leq \cO\prn*{H\sqrt{\Ccov(\Xi_\eta)A_\eta \beta T \log(T)} + TH\eta},
    \end{align*}
    as long as we take
    \begin{align*}
      \beta = c \log(N_\infty(\cF,\eta) TH/\delta),
    \end{align*} 
    where we choose function class $\cN_\infty(\cF,\eta)$ where 
    $\cF = \Lip \circ \Phi$ such that approximate completeness holds (\pref{lem:realizability}), 
    and thus the event in \pref{lem:golf_concentration} holds with probability at least
    $1-\delta$ for all $t\in\brk{T}$, 
    and choose estimation policy $\psi$ to the be uniform policy 
    over the covering action space $\cA_\eta$.
    
    Plugging in this choice for 
    $\beta$ and the bound on $\Xi_\eta$ from \pref{prop:coverability}, we have with probability as least $1-\delta$,
    \begin{align*}
        \reg(T) \leq O\prn*{H\sqrt{(1/\eta)^{\dimsa+\dima} T \log(TH N_\infty(\cF,\eta)/\delta)} + TH\eta}.
    \end{align*}
    Then to calculate the covering number of $\cF$, note that 
    the covering number of $\Lip$ is $\prn*{\frac{1}{\eta}}^{\frac{1}{\eta}^{\dimsa}}$ \citep{wainwright2019high}, and $\Phi$ is finite, we have 
    \begin{align*}
      \reg(T) \leq O\prn*{H\sqrt{(1/\eta)^{2\dimsa+\dima} T \log(TH \abs*{\Phi}/\eta \delta)} + TH\eta}.
  \end{align*}
    To conclude, we choose $\eta = T^{-\frac{1}{2\dimsa + \dima + 2}}$,
    which gives
    \begin{align*}
        \reg(T) \leq \cO\prn*{HT^{\frac{2\dimsa + \dima + 1}{2\dimsa + \dima + 2}} \sqrt{\log(TH|\Phi|/\eta\delta)}}.
    \end{align*}
    Finally, we convert this to a PAC bound. We have that with 
    number of total samples 
    \begin{align*}
        H \cdot T = \bigoh\prn*{H^{(2\dimsa + \dima + 3)} \log \prn*{TH|\Phi| / \delta \epsilon} / 
        \epsilon^{(2\dimsa + \dima + 2)}}
    \end{align*}
    sufficiently large, the policy $\pihat$ satisfies
    \begin{align*}
        J(\pi^\ast) - J(\widehat \pi) \leq \epsilon.
    \end{align*}
    \end{proof}
}

\subsection{Auxiliary Lemmas}
\begin{lemma}\label{lem:golf_decomposition}
   Given any value function $\{f_h\}_{h=1}^H$ that satisfies the last condition in \pref{assum:completeness}, and let $\pi_f := \crl*{\pi_{f;h}(x_h) = \argmax_{a_h \in \cA_\eta}
  f_h(x_h,a_h)}_{h=1}^H$, we have 
\begin{align*}
  \En \brk*{f_1(x_1, \pi_{f;1}(x_1)) - J(\pi_f)} \leq 
  \sum_{h=1}^H \En_{x_h,a_h \sim d^{\pi_f}_h} \brk*{\abs*{f_h(x_h,a_h) - \cT_h 
  \brk*{f_{h+1}}(x_h,a_h)}} + H\eta.
\end{align*}
\end{lemma}

\begin{proof}[Proof of \pref{lem:golf_decomposition}] The proof is mostly the same as the standard results such as Lemma 1 of 
  \citet{jiang2017contextual}, and the only difference is that we need to handle the 
  misspecification due to the continuous action space. By the standard decomposition, we have
  \begin{align*}
    &\En \brk*{f_1(x_1, \pi_{f;1}(x_1)) - J(\pi_f)}\\ =&  
    \En_{a_1 \sim \pi_{f;1}(x_1)} \brk*{ f_1(x_1, a_1) - \prn*{r_1(x_1, a_1) + \En_{x_2 \sim P_1(x_1,a_1)} \brk*{V^{\pi_f}_2(x_2)}} } \\
    =& \En_{a_1 \sim \pi_{f;1}(x_1)} \brk*{ f_1(x_1, a_1) - \prn*{r_1(x_1, a_1) + \En_{x_2 \sim P_1(x_1,a_1)} \brk*{f_2(x_2,\pi_{f;2}(x_2))}} + \En_{x_2 \sim P_1(x_1,a_1)} \brk*{f_2(x_2,\pi_{f;2}(x_2)) - V^{\pi_f}_2(x_2)} }.
  \end{align*}
  Note that the second term is the recursion term, and for the first term, we have 
  \begin{align*}
    &\En_{a_1 \sim \pi_{f;1}(x_1)} \brk*{ f_1(x_1, a_1) - \prn*{r_1(x_1, a_1) + \En_{x_2 \sim P_1(x_1,a_1)} \brk*{f_2(x_2,\pi_{f;2}(x_2))}}}  \\
    =&\En_{a_1 \sim \pi_{f;1}(x_1)} \brk*{ \abs*{f_1(x_1, a_1) - \prn*{r_1(x_1, a_1) + \En_{x_2 \sim P_1(x_1,a_1)} \brk*{\max_{a_2 \in \cA_\eta} f_2(x_2,a_2)}}}} \\
    \leq & \En_{a_1 \sim \pi_{f;1}(x_1)} \brk*{ \abs*{f_1(x_1, a_1) - \prn*{r_1(x_1, a_1) + \En_{x_2 \sim P_1(x_1,a_1)} \brk*{\max_{a_2 \in \cA} f_2(x_2,a_2)}}}} + \eta \tag{last property in \pref{assum:completeness}}\\
    =& \En_{a_1 \sim \pi_{f;1}(x_1)} \brk*{ \abs*{ f_1(x_1,a_1) - \cT_1 \brk*{f_2 (x_1,a_1)} }} + \eta.
  \end{align*}
  Then by applying the recursion on the second term we complete the proof. 
\end{proof}

\golfnew{
\begin{lemma}[Theorem 2.1 and Eq. (13) of \citet{amortila2024mitigating}]\label{lem:golf_concentration}
  Suppose \pref{assum:completeness} holds for $\cF_\eta$. For 
  any $f \in \cF_\eta$, denote $\apx[f_{h+1}] \in \cN_{\infty}(\cF,\gamma)$ as 
  the $\ell_\infty$-approximation of $\cT_h \brk*{f_{h+1}}$. Then for each $t \in [T]$,
  denote 
  \begin{align*}
    &\delta^t_{f;h} \ldef f^t_h - \cT_h \brk*{f^t_{h+1}}, \mathand\\
    &\err^t_{f;h} \ldef \indic\crl*{\abs*{f^t_h - \apx(f^t_{h+1})} \geq 3\gamma}
    \cdot \crl*{\prn*{f^t_h - \cT_h f^t_{h+1}}^2 - \prn*{\apx[f^t_{h+1}] - \cT_h \brk*{f^t_{h+1}}}^2}.
  \end{align*}
  Then running \pref{alg:golf} with $\cF_\gamma$ and 
\begin{align*}
 \beta = c \log \prn*{TH \abs{\cF_\eta}/\delta},
\end{align*}
with probability at least $1-\delta$, we have that for all $t \in [T], h \in [H]$,
\begin{align*}
 f^\ast_\infty \in \cF\ind{t} \mathand \sum_{i < t} \En_{x,a \sim d^i_h} \err^t_{f^t;h}(x,a) \leq \cO(\beta),
\end{align*}
where $d^i_h$ is the data collection distribution of \pref{alg:golf} defined in \pref{eq:data_collection_dist}, and the $f^\ast_\infty$ is the $\ell_\infty$ 
approximation of $Q^\ast$, i.e., $\nrm*{f^\ast_\infty - Q^\ast}_{\infty} = \gamma$.
\end{lemma}
}

\golfold{
\begin{lemma}[Lemma 39 and 40 of \citet{jin2021bellman}]\label{lem:golf_concentration}
   Suppose \pref{lem:realizability} holds. Then running \pref{alg:golf} with $\cF$ defined in 
\pref{lem:realizability} and take
\begin{align*}
  \beta = T^{\frac{\dimsa}{\dimsa+2}} \log(\abs*{\Phi}  TH/\delta),
\end{align*}
them with probability at least $1-\delta$, we have that for all $t \in [T]$,
\begin{align*}
  Q^\ast \in \cF\ind{t} \mathand \sum_{i < t} \En_{x,a \sim d^i_h} \brk*{ \prn*{f_h(x,a)-\cT_h f_{h+1}(x,a)}^2} \leq \cO(\beta),
\end{align*}
where $d^i_h$ is the data collection distribution of \pref{alg:golf} defined in \pref{eq:data_collection_dist}.
\end{lemma}

\begin{proof}
  The proof is identical to the proof of Lemma 39 and 40 of \citet{jin2021bellman}, but here we show how we obtain the non-parametric 
  confidence interval width $\beta$. We first define the notation for 
  the covering number of a function class. 
  For any function class $\cF$, we denote a $\gamma$-cover in the 
  $\ell_\infty$ norm by $\cN_\infty(\cF, \gamma)$; i.e., for any $f \in \cF$, there
  exists $\tilde f \in \cN_\infty(\cF, \gamma)$ such that
  $\nrm[\big]{f - \tilde f}_\infty \leq \gamma$. We define $N_\infty(\cF,\gamma)=\abs*{\cN_\infty(\cF,\gamma)}$.
  In \citet{jin2021bellman}, the result 
  holds for any $\gamma \in (0,1)$, as long as we take 
  \begin{align*}
    \beta = c \prn*{\log \prn*{tH N_{\infty}\prn*{{\cF;\gamma}}} + T\gamma^2}.
  \end{align*}
  Since $\cF = \Lip \circ \Phi$, and the $\gamma$-covering number of 
  a Lipschitz class is $\prn*{\frac{1}{\gamma}}^{\prn*{\frac{1}{\gamma}}^{\dimsa}}$ \citep{wainwright2019high},
  we have 
  \begin{align*}
    \log \prn*{N_\infty(\cF,\gamma)} = \prn*{\frac{1}{\gamma}}^{\dimsa} \log \prn*{\abs*{\Phi}/\gamma},
  \end{align*}
  and taking $\gamma = T^{-\frac{1}{\dimsa+2}}$ we complete the proof. 
\end{proof}
}

\begin{lemma}[Almost optimism]\label{lem:golf_optimism}
  Suppose \pref{lem:golf_concentration} holds. Then for each iteration $t \in [T]$, we have
  \begin{align*}
    f^t_1(x_1,\pi^t_{1}(x_1)) - J(\pi^\ast) \geq -(\eta + \gamma).
  \end{align*}
\end{lemma}

\begin{proof}[Proof of \pref{lem:golf_optimism}]
  Conditioned on the event that \pref{lem:golf_concentration} holds, by the 
  construction of $f^t$, we have 
  \begin{align*}
    \max_{a \in \cA} f^t_1(x_1,a) \geq f^\ast_{\infty;1}(x_1,\pi^\ast(x_1)) \geq J(\pi^\ast) - \gamma.
  \end{align*}
  By the the last property of \pref{assum:completeness}, we have 
  \begin{align*}
    \max_{a \in \cA} f^t_1(x_1,a) - f^t_1(x_1,\pi^t_{1}(x_1)) = 
    \max_{a \in \cA} f^t_1(x_1,a) - \max_{a \in \cA_\eta }f^t_1(x_1,a) \leq \eta.
  \end{align*}
  Then put everything together we complete the proof. 
\end{proof}

\subsection{Proof of \creftitle{thm:lower} (Lower Bound for
  $Q^\ast/V^\ast$-Lipschitz Latent Dynamics)} \label{app:lower_bound}

The following result gives the formal version of \cref{thm:lower}.

\begin{thmmod}{thm:lower}{$'$}[Lipschitz $Q^\star/V^\star$ lower bound; formal
  version of \cref{thm:lower}]
  \label{thm:lower_formal}
  For every $N \in \bbN$, there exists a decoder class $\Phi$ with
  $|\Phi| = N$ and a family of rich observation MDPs $\cM$ with (i)
  $H\leq{}\bigoh(\log(N))$, (ii) $\abs*{\cS_\eta}\leq{}4H\leq\bigoh(\log(N))$ for all $\eta\geq{}0$,
  (iii) $|\cA| = 2$, (iv) $\abs{\cX}\leq{}N^2$, (v) 
  $Q^\star/V^\star$-Lipschitz latent dynamics, and (vi)
  decoder realizability, such that any
  algorithm requires $\Omega(N/\log(N))$ episodes to learn an c-optimal
  policy for a worst-case MDP in $\cM$, where $c>0$ is an absolute constant.
\end{thmmod}


      \begin{proof}[\pfref{thm:lower_formal}]%
        Let $N\geq{}4$ be given, and assume without loss of generality
        that it is a power of $2$. We first construct the class of
        rich-observation MDPs, then verify the Lipschitz assumption
        and prove a sample complexity lower bound.

        \paragraph{Latent MDP}
        Our construction has a single ``known'' latent MDP $\Mlat$; that is,
        the only uncertainty in the family of rich-observation MDPs we
        construct arises from the emission processes. Set
        $H=\log_2(N)+1$ and $\cA=\crl{0,1}$. We define the state space
        and latent transition dynamics as follows.
        \begin{itemize}
        \item The state space can be partitioned as
          $\cS=\cS\ind{1},\ldots,\cS\ind{N}$.
        \item Each block $\cS\ind{i}$ corresponds to a standard
          depth-$H$ binary
          tree MDP with deterministic dynamics (e.g.,
          \citet{osband2016lower,domingues2021episodic}). There is a
          single ``root'' node at layer $h=1$, which we denote by
          $\sinit\ind{i}$, and $N$ ``leaf'' nodes at layer $H$, which
          we denote by $\crl[\big]{\sfin\ind{i,j}}_{j\in\brk{N}}$. For
          each $h=1,\ldots,H-1$, choosing action $0$ leads to the left
          successor of the current state deterministically, and
          choosing action $1$ leads to the right sucessor; this
          process continues until we reach a leaf node at layer $H$.
        \item The initial state distribution is
          $\Platent_1(\emptyset)=\unif(\sinit\ind{1},\ldots,\sinit\ind{N})$.
        \item There are no rewards for layers $1,\ldots,H-1$. For
          layer $H$, the reward is
          \begin{align}
            \label{eq:lower_rewar}
            R_H(\sfin\ind{i,j},\cdot) = \indic\crl*{j=i}.
          \end{align}
        \end{itemize}
Informally, this construction can summarized as follows. At layer $1$,
we draw the index of one of $N$ binary trees uniformly at random, and
initialize into the root of the tree. From here, we receive a reward
of $1$ if we successfully navigate to the leaf node whose index agrees with the
index of the tree itself, and receive a reward of $0$ otherwise.

Note that the total number of latent states in this construction is
$\abs{\cS}=N\cdot\abs{\cS_1}=N(2N-1)$; we will now choose a metric for
which the covering number is
significantly smaller.

\paragraph{Verifying the Lipschitz assumption}
Let $\Qstar$ denote the optimal Q-function for the MDP $\Mlat$. Since
$\Mlat$ is known/fixed in our construction, we can choose the state
metric as
\begin{align*}
  \DS(s,s') = \indic\crl*{\Qstar_h(s,\cdot)\neq{}\Qstar_h(s',\cdot)},
\end{align*}
for states $s$ and $s'$ that are both reachable at layer
$h\in\brk{H}$, and set $\DS(s,s')=1$ if $s$ and $s'$ are reachable at
different layers. We also choose
$\DA(a,a')=\indic\crl{a\neq{}a'}$. With this construction, it is
immediate that for all $\eta\geq{}0$, the latent space covering number
with respect to $\DS$ is bounded as $\abs{\cS_\eta}\leq{}4H=4(\log_2(N)+1)$. Indeed, our
reward construction ensures that $\Qstar$ is $0/1$-valued, so there
are only four possible profiles ($(0,0)$, $(0,1)$, $(1,0)$, $(1,1)$)
that $\Qstar_h(s,\cdot)$ can take on. We form the cover by merging all
states in a given layer whose profiles agree.\loose

\paragraph{Observation space and decoder class}
Let us introduce some additional notation. For each block
$\cS\ind{i}$, let
$\cS_h\ind{i}\ldef{}\crl{s_h\ind{i,j}}_{j\in\brk{2^{h-1}}}$ denote the
states in block $i$ that are reachable at layer $h$, so that
$\cS\ind{i}_1=\crl*{\sinit\ind{i}}$ and
$\cS\ind{i}_H=\crl{\sfin\ind{i,j}}_{j\in\brk{N}}$. We define
$\cX=\cS$ so that $\abs{\cX}\leq{}4N^2$, and consider a class of emission processes corresponding to
deterministic maps. Let
\begin{align*}
  \Psi = \crl*{\psi_i}_{i\in\brk{N}}
\end{align*}
denote the set of cyclic permutations on $N$ elements, excluding the
identity permutation. That is, each $\psi\in\Psi$ takes the form
  $\psi_i:k\mapsto{}k+i\mod{}N$ for $i\in\crl{1,\ldots,N}$. For each
$\psi\in\Psi$, we consider the emission process
\[
E^{\psi}_h(\cdot\mid{}s_h\ind{i,j}) = \indic_{s_h\ind{\psi(i),j}}.
\]
That is, $E^{\psi}$ shifts the index of the binary tree containing
$s_h\ind{i,j}$ according to $\psi$. We consider the class of
rich-observation MDPs given by
\begin{align}
  \label{eq:lower_bound_class}
  \cM \ldef{} \crl*{M\ind{i}\ldef{}E^{\psi_i}\circ\Mlat\mid{}\psi_i\in\Psi},
\end{align}
where $M^{\psi}\ldef{}E^{\psi}\circ\Mlat$ denotes the
rich-observation MDP obtained by equipping $\Mlat$ with the
emission process $E^{\psi}$.
It is clear that this class of rich-observation MDPs satisfies the
decodability assumption for the class
\begin{align*}
  \Phi = \crl[\big]{s_h\ind{i,j}\mapsto{}s_h\ind{\psi^{-1}(i),j}\mid{}\psi\in\Psi},
\end{align*}
which has $\abs{\Phi}=N$.

\paragraph{Sample complexity lower bound}%
To lower bound, the sample
complexity, we prove a lower bound on the constrained PAC
Decision-Estimation Coefficient (DEC) of
\citet{foster2023tight}. For an arbitrary MDP $\Mbar$ (defined over
the space $\cX$) and $\veps\in\brk{0,2^{1/2}}$, define\footnote{For
  measures $\bbP$ and $\bbQ$, we define squared Hellinger distance by $\Dhels{\bbP}{\bbQ}=\int(\sqrt{d\bbP}-\sqrt{d\bbQ})^2$.}
\begin{align*}
  \dec(\cM,\Mbar)
  = \inf_{p,q\in\Delta(\Pi)}\sup_{M\in\cM}
  \crl*{\En_{\pi\sim{}p}\brk*{\Jm(\pim)-\Jm(\pi)}
  \mid{} \En_{\pi\sim{}q}\brk*{\Dhels{M(\pi)}{\Mbar(\pi)}}\leq\veps^2},
\end{align*}
where $M(\pi)$ denotes the law over trajectories
$(x_1,a_1,r_1),\ldots,(x_H,a_H,r_H)$ induced by executing
the policy $\pi$ in the MDP $M$, $\Jm(\pi)$ denotes the expected
reward for policy $\pi$ under $M$, and $\pim$ denotes the optimal
policy for $M$. We further define
\begin{align*}
  \dec(\cM,\Mbar) = \sup_{\Mbar}\dec(\cM,\Mbar),
\end{align*}
where the supremum ranges over all MDPs defined over $\cX$ and
$\cA$. We now appeal to the following lemma.\loose
\begin{lemma}
  \label{lem:dec_lower}
  For all $\veps^2\geq{}4/N$, we have that
  $\sup_{\Mbar}\dec(\cM,\Mbar) \geq \frac{1}{2}$.
\end{lemma}
In light of \cref{lem:dec_lower}, it follows from Theorem 2.1 in \citet{foster2023tight}\footnote{Theorem 2.1 in
  \citet{foster2023tight} is stated with respect to
  $\sup_{\Mbar\in\mathrm{conv}(\cM)}\dec(\cM,\Mbar)$, but the actual proof (Section
  2.2) gives a stronger result that scales with $\sup_{\Mbar}\dec(\cM,\Mbar)$.}
that any PAC RL algorithm that uses $T$ episodes
of interaction for $T\log(T)\leq{}c\cdot{}N$ must have
$\En\brk*{\Jm(\pim)-\Jm(\pihat)}\geq{}c'$ for a worst-case MDP in
$\cM$, where $c,c'>0$ are absolute constants.
\end{proof}

\begin{proof}[\pfref{lem:dec_lower}]
  Define $\Mbarlat$ as the latent-space MDP that has identical
  dynamics to $\Mlat$ but, has zero reward in every state, and define
  $\Mbar \ldef \id\circ\Mbarlat$ as the rich-observation MDP obtained
  by composing $\Mbarlatent$ with the identity emission process that
  sets $x_h=s_h$. Observe that $\Mbar$ and $M_i$, induce identical dynamics
  in observation space if rewards are ignored: For all policies $\pi$, 
  \begin{equation}
    \label{eq:identical_dynamics}
    \bbP^{\sss{\Mbar},\pi}\brk*{(x_1,a_1),\ldots,(x_H,a_H)=\cdot}
    = \bbP^{\sss{M\ind{i}},\pi}\brk*{(x_1,a_1),\ldots,(x_H,a_H)=\cdot}.
  \end{equation}
  This is because, e.g., under $M\ind{i}$ the probability that $x_1 = j$ is $1/N$ and this is true for all choices of the cyclic permutation. 
It follows that for each $i$, for all policies $\pi$, we have
  \begin{align}
    \Dhels{M\ind{i}(\pi)}{\Mbar(\pi)}\notag
    &= \Dhels{(E^{\psi_i}\circ\Mlatent)(\pi)}{(\id\circ\Mbarlatent)(\pi)} \notag\\
    &= \sum_{j=1}^{N}
      \bbP^{\sMbar,\pi}\brk*{x_H=\sfin\ind{\psi_i(j),j}}\cdot\Dhels{\indic_{1}}{\indic_{0}}\notag\\
    &= 2\sum_{j=1}^{N}
      \bbP^{\sMbar,\pi}\brk*{x_H=\sfin\ind{\psi_i(j),j}}\notag\\
    &= \frac{2}{N}\sum_{j=1}^{N}
      \bbP^{\sMbar,\pi}\brk*{x_H=\sfin\ind{\psi_i(j),j}\mid{}x_1=\sinit\ind{\psi_i(j)}},\\
    &= \frac{2}{N}\sum_{j=1}^{N}
      \bbP^{\sMbar,\pi}\brk*{x_H=\sfin\ind{j,\psi_i^{-1}(j)}\mid{}x_1=\sinit\ind{j}}.\label{eq:lower_step1}
  \end{align}
  The second identity is the conditioning rule for Hellinger distance, which states that:
  \begin{align*}
    \Dhels{\bbP_{Y\mid X}\circ \bbP_X}{\bbQ_{Y\mid X}\circ \bbP_X} = \bbE_{X \sim \bbP_X}\brk*{\Dhels{\bbP_{Y\mid X}}{\bbQ_{Y\mid X}}}.
  \end{align*}
  We instantiate this with $X$ as the trajectory, i.e., $(x_1,a_1),\ldots,(x_H,a_H)$, and $Y$ as the reward, and use the fact that the law of the trajectory is the same for $M\ind{i}(\pi)$ and $\Mbar(\pi)$. 
  For the rest of the calculation, we use that the learner receives identical feedback in the MDPs
  $M\ind{i}$ and $\Mbar$ unless they reach the observation
  $x_H=\sfin\ind{\psi_i(j),j}$ for some $j$ (corresponding to latent
  state $\sfin\ind{j,j}$ in $M\ind{i}$), in which case they receiver
  reward $1$ in $M\ind{i}$ but reward $0$ in $\Mbar$.
  We now claim that for any $q\in\Delta(\Pi)$, there exists a set of
  at least $N/2$ indices $\cI_q\subset\brk{N}$ such that
  \begin{align}
    \En_{\pi\sim{}q}\brk*{\Dhels{M\ind{i}(\pi)}{\Mbar(\pi)} } \leq
    \frac{4}{N}
    \label{eq:lower_step2}
  \end{align}
  for all $i\in\cI_q$. To see this, note that by
  \cref{eq:lower_step1}, we have
  \begin{align*}
    \En_{i\sim\unif(\brk{N})}\En_{\pi\sim{}q}\brk*{
    \Dhels{M\ind{i}(\pi)}{\Mbar(\pi)}
    }
    &\leq{}
      \En_{\pi\sim{}q}\brk*{\frac{2}{N}\sum_{j=1}^{N}
      \frac{1}{N}\sum_{i=1}^{N}\bbP^{\sMbar,\pi}\brk*{x_H=\sfin\ind{j,\psi_i^{-1}(j)}\mid{}x_1=\sinit\ind{j}}
      }\\
    &\leq{}
      \En_{\pi\sim{}q}\brk*{\frac{2}{N}\sum_{j=1}^{N}
      \frac{1}{N}
      }
      = \frac{2}{N},
  \end{align*}
  where the second inequality uses that
  $\sum_{i=1}^{N}\bbP^{\sMbar,\pi}\brk*{x_H=\sfin\ind{j,\psi_i^{-1}(j)}\mid{}x_1=\sinit\ind{j}}\leq{}1$,
  as the events in the sum are mutually exclusive (and the event we
  condition on does not depend on $i$). We conclude by Markov's
  inequality that
  $ \bbP_{i\sim\unif(\brk{N})}\brk*{\En_{\pi\sim{}q}\brk*{
      \Dhels{M\ind{i}(\pi)}{\Mbar(\pi)}}\geq{}4/N } \leq{} 1/2$,
  giving $\cI_q\geq{}N/2$.

  From \cref{eq:lower_step2}, we conclude that for all
  $\veps^2\geq{}4/N$,
  \begin{align*}
    \dec(\cM,\Mbar)
    \geq{} \inf_{q\in\Delta(\Pi)}\inf_{p\in\Delta(\Pi)}\sup_{i\in\cI_q}
    \crl*{\En_{\pi\sim{}p}\brk*{J^{\sss{M\ind{i}}}(\pi_{\sss{M\ind{i}}})-J^{\sss{M\ind{i}}}(\pi)}}.
  \end{align*}
  To lower bound this quantity, observe that for any index $i$ and any
  policy $\pi$, we have
  \begin{align*}
    J^{\sss{M\ind{i}}}(\pi_{\sss{M\ind{i}}})-J^{\sss{M\ind{i}}}(\pi)
    &= \frac{1}{N}\sum_{j=1}^{N}
      \bbP^{\sss{M}\ind{i},\pi}\brk*{x_H\neq\sfin\ind{\psi_i(j),j}\mid{}x_1=\sinit\ind{\psi_i(j)}}\\
    &= 1-\frac{1}{N}\sum_{j=1}^{N}
      \bbP^{\sss{M}\ind{i},\pi}\brk*{x_H=\sfin\ind{\psi_i(j),j}\mid{}x_1=\sinit\ind{\psi_i(j)}}\\
    &= 1-\frac{1}{N}\sum_{j=1}^{N}
      \bbP^{\sMbar,\pi}\brk*{x_H=\sfin\ind{\psi_i(j),j}\mid{}x_1=\sinit\ind{\psi_i(j)}}\\
    &= 1-\frac{1}{N}\sum_{j=1}^{N}
      \bbP^{\sMbar,\pi}\brk*{x_H=\sfin\ind{j,\psi_i^{-1}(j)}\mid{}x_1=\sinit\ind{j}},
  \end{align*}
  where the third inequality uses \cref{eq:identical_dynamics}.
  We conclude that for any distribution $p,q\in\Delta(\Pi)$,
  \begin{align*}
    &\sup_{i\in\cI_q}
      \crl*{\En_{\pi\sim{}p}\brk*{J^{\sss{M\ind{i}}}(\pi_{\sss{M\ind{i}}})-J^{\sss{M\ind{i}}}(\pi)}}
    \\
    &\geq\En_{i\sim\unif(\cI_q)}
      \crl*{\En_{\pi\sim{}p}\brk*{J^{\sss{M\ind{i}}}(\pi_{\sss{M\ind{i}}})-J^{\sss{M\ind{i}}}(\pi)}}\\
    &\geq
      1- \frac{1}{N}\sum_{j=1}^{N}
      \En_{i\sim\unif(\cI_q)}\bbP^{\sMbar,\pi}\brk*{x_H=\sfin\ind{j,\psi_i^{-1}(j)}\mid{}x_1=\sinit\ind{j}}\\
    &=
      1- \frac{1}{N}\sum_{j=1}^{N}
      \frac{1}{\abs*{\cI_q}}\sum_{i\in\cI_q}\bbP^{\sMbar,\pi}\brk*{x_H=\sfin\ind{j,\psi_i^{-1}(j)}\mid{}x_1=\sinit\ind{j}}
    \geq{} 
    1- \frac{1}{\abs*{\cI_q}} \geq{} \frac{1}{2}
  \end{align*}
  as long as $N\geq{}4$, where the second-to-last inequality uses that
  for all $j$, the events
  $\crl[\big]{x_H=\sfin\ind{j,\psi_i^{-1}(j)}\mid{}x_1=\sinit\ind{j}}$
  are disjoint for all $i$. Since this lower bound holds uniformly for
  all $q,p\in\Delta(\Pi)$, we conclude that
  \begin{align*}
    \dec(\cM,\Mbar) \geq \frac{1}{2}.
  \end{align*}
\end{proof}


%% file: appendix_main.tex
The structure of this section is as follows. We first provide the proof of the 
representation learning result (\pref{thm:rep_learn}) in \pref{sec:rep_learn_proof_app}.
Then in \pref{app:mainalg_detailed}, we prove the formal version of the 
guarantee of \algname{} (\pref{thm:main_alg_formal}),  and in \pref{sec:mainalg_proof} we 
present the proof for \pref{thm:main_alg_formal}. Finally in \pref{app:homer}, 
we prove \pref{prop:homer}, the negative result for the prior approaches in \framework.

\subsection{Proof of \creftitle{thm:rep_learn}}\label{sec:rep_learn_proof_app}
The proof of \pref{thm:rep_learn} follows standard concentration
arguments, combined with structural properties of the \framework
framework that follow from
\cref{ass:lipschitz,assum:decoder_realizability}. We first prove some
technical lemmas, as well as a more general version of
\cref{thm:rep_learn} (\cref{sec:rep_learn_tools}), then prove
\cref{thm:rep_learn} as a consequence (\cref{sec:rep_learn_proof}).
Throughout this section of the appendix, when using a decoder $\phi:
\cX \to \cS$, we abbreviate $\phi(x,a) := (\phi(x),a)$ to keep
notation compact.

\subsubsection{Technical Tools}
\label{sec:rep_learn_tools}

In this section, we prove a more general version of \cref{thm:rep_learn},
\cref{thm:rep_learn_f}, which supports a setting in which the dataset is 
collected in an online fashion; this result will be used later to
prove our online RL guarantees for \cref{alg:main_alg}. Beyond
allowing for an online dataset, we also prove a result that holds for
an arbitrary discriminator class $\cF$;  \pref{thm:rep_learn} will
follow by instantiating the discriminator class as the composition of
the decoder class and a class of Lipschitz functions.

We recall notation for covering numbers for a function class. 
Denote 
the $\ell_\infty$ norm for a function $f: \cX \to \bbR$ 
as $\nrm*{f}_\infty = \sup_{x \in \cX} \abs*{f(x)}$. 
For any function class $\cF$, we denote a $\gamma$-cover in the 
$\ell_\infty$ norm by $\cN_\infty(\cF, \gamma)$; i.e., for any $f \in \cF$, there
exists $\tilde f \in \cN_\infty(\cF, \gamma)$ such that
$\nrm[\big]{f - \tilde f}_\infty \leq \gamma$. We define $N_\infty(\cF,\gamma)=\abs*{\cN_\infty(\cF,\gamma)}$.

\begin{theorem}[Guarantee of \replearn]\label{thm:rep_learn_f}
    Let $h \in [H]$ be fixed, and consider a dataset $\cD_h := \{(x^i_h, a_h^i, x_{h+1}^i)\}_{i=1}^t$, where 
    $(x_h^i, a_h^i, x_{h+1}^i) \sim \bar \rho^t_h \circ P_h$, and $\bar \rho^t_h = \frac{1}{t}\sum_{i=1}^t \rho^i_h$,
    where $\rho^i_h\in\Delta(\cX\times\cA)$ is an arbitrary
    state-action distribution that may depend upon the randomness of
    rounds $1 \dots i-1$. Suppose
    \cref{ass:lipschitz,assum:decoder_realizability} hold, and let
  $\phi_h^t$ be the output of \pref{alg:rep_learn} with inputs
  $\cD_h$, $\cF: \cX \to [0,L]$, and $\Lip: \cS \to [0,L]$. Then for any $\delta\in(0,1)$, with probability $1-\delta$,
  for all $\gamma \in (0,1)$,
\begin{align*}
    \max_{f \in \cF_{h+1}} \En_{x,a \sim  \bar \rho_h^t}\brk*{
    \prn*{\psb_{\cD_h, \phi^t_h}[f](x,a) - \realb_h[f](x,a)}^2 } \leq \erep(t,\delta, \gamma),
\end{align*}
where
\begin{align*}
   \erep(t,\delta, \gamma) =  \frac{88L^2 (2L)^{\dimsa{}} \prn*{\frac{1}{\gamma}}^{\dimsa}
    \log\prn*{2L \abs*{\Phi} \cdot N_{\infty}(\cF, \gamma)/(\delta \gamma)}}{t} + 24L^2\gamma^2.
\end{align*}
\end{theorem}

Before proving \cref{thm:rep_learn_f}, we state a basic concentration
lemma, proven in the sequel. 
Recall that for any $h \in [H]$, given dataset $\cD_h$, the least squares loss 
between a decoder $\phi$, lipschitz function $g$ and discriminator $f$ is defined as
\begin{align*}
  \ell_{\cD_h}(\phi, g, f) \ldef \widehat \En_{(x_h,a_h,x_{h+1})\sim \cD_h} \brk*{ (g(\phi(x_h,a_h)) - f(x_{h+1}))^2 }.
\end{align*}
Similarly, the population least square loss between a decoder $\phi$, lipschitz function $g$ and discriminator $f$ under the distribution $\bar \rho_h^t$ is defined as
\begin{align*}
  \ell_{\bar \rho^t_h}(\phi, g, f) \ldef \En_{(x_h,a_h)\sim \bar \rho^t_h} \brk*{ \prn*{ g(\phi(x_h,a_h)) - \En \brk*{f \mid x_h,a_h}}^2 }.
\end{align*}

\begin{lemma}[Concentration of the \replearn
  loss] \label{lem:loss_concentration}
  Let $h \in [H]$ be fixed, and consider a dataset $\cD_h := \{(x^i_h, a_h^i, x_{h+1}^i)\}_{i=1}^t$, where 
    $(x_h^i, a_h^i, x_{h+1}^i) \sim \bar \rho^t_h \circ P_h$, and $\bar \rho^t_h = \frac{1}{t}\sum_{i=1}^t \rho^i_h$,
    where $\rho^i_h\in\Delta(\cX\times\cA)$ is an arbitrary
    state-action distribution that may depend upon the randomness of
    rounds $1 \dots i-1$. Suppose
\cref{ass:lipschitz,assum:decoder_realizability} hold.
If all $f\in\cF_{h+1}$ have $f(x)\in\brk{0,L}$, then for all $\delta\in(0,1)$, with probability $1-\delta$, we have for all $\phi \in \Phi$,
    $f \in \cF_{h+1}$, $g \in \Lip$, and $\gamma \in (0,1)$,
  \begin{align*}
    &\abs*{\prn*{\ell_{\bar \rho^t_h}(\phi, g, f) - \ell_{\bar \rho^t_h}(\phi^\ast, g_f^\ast, f)}
    - \prn*{\ell_{\cD_h}(\phi, g, f) - \ell_{\cD_h}(\phi^\ast, g^\ast, f)}}\\
    &\leq 
    \frac{1}{2}\abs*{\prn*{\ell_{\bar \rho^t_h}(\phi, g, f) - \ell_{\bar \rho^t_h}(\phi^\ast, g_f^\ast, f)}} 
    + \frac{22L^2 (2L/\gamma)^{\dimsa{}} 
    \log\prn*{2L \abs*{\Phi} \cdot N_{\infty}(\cF, \gamma)/(\delta \gamma)}}{t} + 6L^2\gamma^2,
  \end{align*}
where $g^\ast_f := \argmin_{g \in \Lip}
\ell_{\bar \rho^t_h}(\phi^\ast, g, f)$.
\end{lemma}

Before proving this result, we use it to prove \cref{thm:rep_learn_f}.

\begin{proof}[\pfref{thm:rep_learn_f}]
  Let us condition on the event from
  \pref{lem:loss_concentration}. Going forward, we fix $\gamma$ and
  $h$, and omit the dependence on $h$ and $t$ for compactness.
  Define as shorthand
    \begin{align*}
      \veps := \erep(t,\delta) =  \frac{22L^2 (2L)^{\dimsa{}} \dgamma 
      \log\prn*{2L \abs*{\Phi} \cdot N_{\infty}(\cF, \gamma)/(\delta \gamma)}}{t} + 6L^2\gamma^2.
    \end{align*}
For all $f \in \cF_h$,
    denote the solution of the inner optimization from 
    \pref{eq:rep_learn} by $\widetilde \phi_f, \widetilde g_f$, and
    let $g^\ast_f$ be the 
    predictor of backup of $f$ using ground truth decoder $\phi^\ast$, 
    i.e.,  $g^\ast_f := \argmin_{g \in \Lip}
    \ell_{\bar \rho^t_h}(\phi^\ast, g, f)$.
    Conditioned the event in \pref{lem:loss_concentration}, we have
    \begin{align*}
      \ell_{\cD_h}(\phi^\ast, g^\ast_f, f) - \ell_{\cD_h}(\widetilde \phi_f, \widetilde g_f, f) 
      &\leq \frac{1}{2}\ell_{\bar \rho^t_h}(\phi^\ast, g^\ast_f, f) 
      - \frac{1}{2}\ell_{\bar \rho^t}(\widetilde \phi_f, \widetilde g_f, f) 
      + \veps \tag{\pref{lem:loss_concentration}} \\
      &\leq \veps. \tag{$\ell_{\bar \rho^t_h}(\phi^\ast, g^\ast_f, f) = 0$.}
    \end{align*}
    Then, letting $g_f$ be the predictor of pseudobackup of $f$ using learned decoder $\phi$, i.e., 
    $g_f \ldef \psb_{\cD_h, \phi}[f]$,
    we have 
    \begin{align*}
      \ell_{\cD_h}(\phi, g_f, f) - \ell_{\cD_h}(\widetilde \phi_f, \widetilde g_f, f)
      &= \ell_{\cD_h}(\phi, g_f, f) - \ell_{\cD_h}(\phi^\ast, g_f^\ast, f) 
       + \ell_{\cD_h}(\phi^\ast, g_f^\ast, f)- \ell_{\cD_h}(\widetilde \phi_f, \widetilde g_f, f) \\
      &\geq \ell_{\cD_h}(\phi, g_f, f) - \ell_{\cD_h}(\phi^\ast, g_f^\ast, f) 
      \tag{Construction of $\widetilde \phi_f, \widetilde g_f$}\\
      &\geq \frac{1}{2} \prn*{\ell_{\bar \rho^t_h}(\phi, g_f, f) - \ell_{\bar \rho^t_h}(\phi^\ast, g_f^\ast, f)} - \veps
       \tag{\pref{lem:loss_concentration}}.
    \end{align*}
    Finally, combining the results above gives
    \begin{align*}
      \ell_{\bar \rho^t_h}(\phi, g_f, f) - \ell_{\bar \rho^t_h}(\phi^\ast, g_f^\ast, f)
      & \leq  2\prn*{\ell_{\cD_h}(\phi, g_f, f) - \ell_{\cD_h}(\widetilde \phi_f, \widetilde g_f, f)} + 2\veps \\
      & \leq  2\max_{k \in \cF_{h+1}} \prn*{\ell_{\cD_h}(\phi, g_k, k) -
      \ell_{\cD_h}(\widetilde \phi_f, \widetilde g_k, k)} + 2\veps \\
      & \leq  2\max_{k \in \cF_{h+1}} \prn*{\ell_{\cD_h}(\phi^\ast, g^\ast_k, k) 
      - \ell_{\cD_h}(\widetilde \phi_f, \widetilde g_k, k)} + 2\veps \\
      & \leq  4\veps,
    \end{align*}
    completing the proof.
  \end{proof}
  
Finally, we provide the proof of \pref{lem:loss_concentration}. 
\begin{proof}[Proof of \pref{lem:loss_concentration}]
First let us use the shorthand $g^\ast = g^\ast_f$.
  At any round $t$, denote $\mathscr{F}^{t-1}$ as the filtration for the random variable 
\begin{align*}
    Y^t := \prn*{g(\phi(x^t,a^t)) - f(x'^t)}^2 - \prn*{g^\ast(\phi^\ast(x^t,a^t)) - f(x'^t)}^2.
\end{align*}
Then by \pref{lem:bernstein}, we have that at round $t$, for all
  $\delta'\in(0,1)$, with probability $1-\delta'$,
\begin{equation}\label{eq:bernstein}
    \abs*{\prn*{\ell_{\bar \rho^t_h}(\phi, g, f) - \ell_{\bar \rho^t_h}(\phi^\ast, g^\ast, f)}
- \prn*{\ell_{\cD_h}(\phi, g, f) - \ell_{\cD_h}(\phi^\ast, g^\ast, f)}} \leq 
    \sqrt{\frac{2\bbV\prn*{Y^t\mid\mathscr{F}\ind{t-1}} \log(2/\delta')}{t}} + \frac{16L^2 \log(2/\delta')}{3t},
  \end{equation}
where $\bbV \prn*{Y^t\mid\mathscr{F}\ind{t-1}} := \En \brk*{\prn*{Y^t - \En\brk*{Y^t \mid \mathscr{F}\ind{t-1}}}^2 \mid \mathscr{F}\ind{t-1}} = \En \brk*{\prn*{Y^t}^2  \mid \mathscr{F}\ind{t-1}}$ denotes the conditional variance of $Y^t$ given $\mathscr{F}\ind{t-1}$.
To bound $\bbV[Y^t \mid \mathscr{F}\ind{t-1}]$, we start with: 
\begin{align*}
    \En\brk*{Y^t \mid \mathscr{F}^{t-1}} &= 
    \En_{\rho^t \mid \mathscr{F}^{t-1}} \brk*{\prn*{g(\phi(x^t,a^t)) - f(x'^t)}^2 - \prn*{g^\ast(\phi^\ast(x^t,a^t)) - f(x'^t)}^2}\\
    &= \En_{\rho^t \mid \mathscr{F}^{t-1}} \brk*{\prn*{g(\phi(x^t,a^t)) + g^\ast(\phi^\ast(x^t,a^t)) - 2f(x'^t)}
    \prn*{g(\phi(x^t,a^t)) - g^\ast(\phi^\ast(x^t,a^t))}}\\
    &= \En_{\rho^t \mid \mathscr{F}^{t-1}} \brk*{\prn*{g(\phi(x^t,a^t)) - g^\ast(\phi^\ast(x^t,a^t))}^2} .
    \tag{$ \En_{\rho^t \mid \mathscr{F}^{t-1}} \brk*{\prn*{g^\ast(\phi^\ast(x^t,a^t)) - f(x')}
    \prn*{g(\phi(x^t,a^t)) - g^\ast(\phi^\ast(x^t,a^t))}} = 0$.}
\end{align*}
Then we can show that
\begin{align*}
    \bbV[Y^t \mid \mathscr{F}^{t-1}] &= \En\brk*{\prn*{Y^t}^2 \mid \mathscr{F}^{t-1}} \\
    &= \En_{\rho^t \mid \mathscr{F}^{t-1}} \brk*{\prn*{\prn*{g(\phi(x^t,a^t)) - f(x'^t)}^2 - \prn*{g^\ast(\phi^\ast(x^t,a^t)) - f(x'^t)}^2}^2} \\
    &= \En_{\rho^t \mid \mathscr{F}^{t-1}} \brk*{\prn*{g(\phi(x^t,a^t)) + g^\ast(\phi^\ast(x^t,a^t)) - 2f(x')}^2
    \prn*{g(\phi(x^t,a^t)) - g^\ast(\phi^\ast(x^t,a^t))}^2}\\
    &\leq 16L^2\En_{\rho^t \mid \mathscr{F}^{t-1}} \brk*{\prn*{g(\phi(x^t,a^t)) - g^\ast(\phi^\ast(x^t,a^t))}^2} \\
    &= 16L^2 \En[Y^t \mid \mathscr{F}^{t-1}].
\end{align*}
Plugging back into \pref{eq:bernstein}, we have, with probability $1-\delta'$, 
\begin{align*}
    \abs*{\prn*{\ell_{\bar \rho^t_h}(\phi, g, f) - \ell_{\bar \rho^t_h}(\phi^\ast, g^\ast, f)}
    - \prn*{\ell_{\cD_h}(\phi, g, f) - \ell_{\cD_h}(\phi^\ast, g^\ast, f)}}\leq
    \sqrt{\frac{32L^2 \En[Y^t] \log(2/\delta')}{t}} + \frac{16L^2 \log(2/\delta')}{3t}.
\end{align*}

Now, let $\wt{\cG}$ denote any $\gamma$-cover for the $\Lip$ function
w.r.t., the $\ell_\infty$ norm.
Since $\Lip$ is bounded by $L$, we have that $\abs*{\widetilde \cG} =
\prn*{\frac{1}{\gamma}}^{\prn*{\frac{2L}{\gamma}}^{\dimsa{}}}$ \citep{wainwright2019high}.
We now take a union bound over $\Phi \times \widetilde{\cG} \times \widetilde{\cF}$,
 where $\widetilde{\cF}$ is an $\ell_\infty$-cover of the discriminator class $\cF$,
 denoted by $\cN_{\infty}(\cF,\gamma)$. Then by union bound, with probability 
$1- \abs*{\Phi}\abs*{\widetilde{\cG}} \abs*{\widetilde{\cF}}\delta'$,
we have that for all $g \in \Lip$ and $f \in \cF$, 
let $\widetilde g \in \widetilde \cG$ such that 
$\nrm*{g-\widetilde g}_{\infty} \leq \gamma$, 
$\widetilde f$ is defined similarly,
and $\widetilde Y^t \ldef \prn*{\widetilde g(\phi(x^t,a^t)) - \widetilde f(x'^t)}^2 - \prn*{g^\ast(\phi^\ast(x^t,a^t)) - f(x'^t)}^2 $
we have
\begin{align*}
    &\abs*{\prn*{\ell_{\bar \rho^t_h}(\phi, g, f) - \ell_{\bar \rho^t_h}(\phi^\ast, g^\ast, f)}
    - \prn*{\ell_{\cD_h}(\phi, g, f) - \ell_{\cD_h}(\phi^\ast, g^\ast, f)}}\\ \leq& 
    \abs*{\prn*{\ell_{\bar \rho^t_h}(\phi, \widetilde g, \widetilde f) - \ell_{\bar \rho^t_h}(\phi^\ast, g^\ast, f)}
    - \prn*{\ell_{\cD_h}(\phi, \widetilde g, \widetilde f) - \ell_{\cD_h}(\phi^\ast, g^\ast, f)}} + 4L^2 \gamma^2 \\
    \leq& \sqrt{\frac{32L^2 \En[\widetilde Y^t] \log(2/\delta')}{t}} + \frac{16L^2 \log(2/\delta')}{3t} + 4L^2 \gamma^2\\
    \leq& \frac{1}{2}\En[\widetilde Y^t] +  \frac{16L^2 \log(2/\delta')}{t} 
    + \frac{16L^2 \log(2/\delta')}{3t} + 4L^2 \gamma^2 \tag{AM-GM}\\
    \leq& \frac{1}{2}\En[Y^t] +  \frac{22L^2 \log(2/\delta')}{t} + 6L^2 \gamma^2 \\
    =& \frac{1}{2}\prn*{\ell_{\bar \rho^t_h}(\phi, g, f) - \ell_{\bar \rho^t_h}(\phi^\ast, g^\ast, f)}
    +  \frac{22L^2 \log(2/\delta')}{t} + 6L^2 \gamma^2.
\end{align*}
Finally, we apply the result above with $\delta' =
\frac{\delta}{\abs*{\Phi}\abs*{\widetilde{\cG}}
  \abs*{\widetilde{\cF}}}$ to complete the proof.
\end{proof}

\subsubsection{Proof of \creftitle{thm:rep_learn}}
\label{sec:rep_learn_proof}

\begin{proof}[Proof of \pref{thm:rep_learn}]
By applying \cref{thm:rep_learn_f} with the discriminator class $\cF_{h+1}
= \Lip \circ \Phi: \cX \to [0,L]$, we have that with probability at
least $1-\delta$,
  \begin{align*}
    \max_{f \in \cF_{h+1}} \En_{x,a \sim  \bar \rho_h^t}\brk*{
    \prn*{\psb_{\cD_h, \phi^t_h}[f](x,a) - \realb_h[f](x,a)}^2 } \leq \erep(t,\delta),
  \end{align*}
  where 
  \begin{align*}
    \erep(t,\delta) =  \frac{88L^2 (2L)^{\dimsa{}} \dgamma 
    \log\prn*{2L \abs*{\Phi} \cdot N_{\infty}(\cF_{h+1}, \gamma)/(\delta \gamma)}}{t} + 24L^2\gamma^2.
  \end{align*}
  By a similar calculation to \pref{lem:covering_num}, we have that the covering 
  number of $\cF_{h+1}$ is
  \begin{align*}
    N_{\infty}(\cF_{h+1}, \gamma) \leq \prn*{\frac{1}{\gamma}}^{\prn*{\frac{2L}{\gamma}}^{\dimsa{}}}\cdot\abs*{\Phi},
  \end{align*}
  so that 
  \begin{align*}
    \erep(t,\delta) =  \widetilde {\cO} \prn* {\frac{L^{2\dimsa{}} {\prn*{\frac{1}{\gamma}}^{2\dimsa{}}} 
    \log\prn*{2L \abs*{\Phi} \cdot \prn*{\frac{1}{\gamma}}/(\delta \gamma)}}{t} + L^2\gamma^2}.
  \end{align*}
  Taking $\gamma = t^{\frac{1}{2\dimsa{}+2}}$ completes the proof.
\end{proof}

\subsection{Formal Version of \creftitle{thm:main_alg}}
\label{app:mainalg_detailed}

The following theorem, proved in \cref{sec:mainalg_proof}, is the
formal version of \cref{thm:main_alg}.

\begin{thmmod}{thm:main_alg}{$'$}[Guarantee for \mainalg; formal version 
  of \cref{thm:main_alg}]\label{thm:main_alg_formal} 
With probability at least $1-\delta$, setting parameters \loose
\begin{align*}
  \lambda^t = \Theta\prn*{t^{\frac{\dimsa}{\widetilde{d}+2}} \log\prn*{\frac{t|\Phi|}{\delta}}}, 
  \quad \widehat \alpha^t = \Theta\prn*{ t^{\frac{\widebar d}{\widetilde d}} \log \prn*{\frac{t|\Phi|}{\delta}}},
\end{align*}
let $\widehat \pi$ be the output of the \mainalg, we have 
\begin{align*}
  J(\pi^\ast) - J(\widehat{\pi}) \leq \epsilon,
\end{align*}
with a total number of samples at most
\begin{align*}
  H\cdot T = \cO\prn*{\frac{H^{2\widetilde d + 3} \log \prn*{TH|\Phi| / \delta \epsilon}} 
  {\epsilon^{\widetilde d + 1}}},
\end{align*}
where $\widetilde d = 3\dimsa^2 + 4\dimsa \dima + 5\dimsa + 4\dima + 1$, 
$\widebar d = 1.5 \dimsa^2 + 2\dimsa\dima + \dimsa + \dima$.
\end{thmmod}

\subsection{Proof of \pref{thm:main_alg_formal}}
\label{sec:mainalg_proof}

We begin by introducing some preliminary notation and giving an overview of the proof structure.

\subsubsection{Preliminaries and Proof Organization}\label{sec:golf_proof_pre}

\paragraph{Additional notation for pseudobackups}

Recall that in \pref{eq:pseudobackup}, given a function class $\cV \subset (\cX \times \cA) \to [0,L]$,
we defined the pseudobackup operator as
\begin{align*}
  \psb_{\cD_h,\cV}: f \mapsto \argmin_{v \in \cV} \widehat \En_{\cD_h}\brk*{ \prn*{v(x_h,a_h) - f(x_{h+1})}^2 },
\end{align*}
and we specialized to $\psb_{\cD_h,\phi_h}$ 
where $\cV \ldef \Lip \circ \phi_h$ for some $\phi_h \in \Phi$. We call this 
pseudobackup operator the \emph{continuous pseudobackup} as we will distinguish 
it from the other pseudobackup operators involving discretization.\loose

Before we describe the other pseudobackup operators, we first introduce 
notation for the discretized version of the decoder. Given a decoder $\phi_h$,
we let $\disc{\eta}[\phi](x,a) := \disc{\eta}[\phi(x,a)]$,
and one should interpret $\disc{\eta}[\phi_h](\cdot)$ as a one-hot vector in $\bbR^{\dimeta}$ (recall that $\dimeta = \prn*{\frac{1}{\eta}}^{\dimsa}$).
With this we can define the function class 
\begin{align*}
  \cW = \crl*{w^{\top} \disc{\eta}[\phi_h] \mid w \in \bbR^{\dimeta}, \|w\|_{\infty} \leq L},
\end{align*}
the linear functions over the discretized decoder. 
We call the pseudobackup operator induced by this class as \emph{linear pseudobackup},
denoted as $\psb_{\cD_h, \cW}$.
Following the definition, the linear pseudobackup is given by
\begin{align*}
  &\psb_{\cD_h, \cW}[f] = w_f^{\top} \disc{\eta}[\phi_h], ~~\text{where}\\
  & w_f = \argmin_{w \in \bbR^{\dimeta}, \|w\|_{\infty} \leq L} 
    \En_{x,a,x' \sim \cD_h}\prn{w^{\top} \disc{\eta}[\phi_h](x,a) - f(x')}^2,
\end{align*}
Now since $\disc{\eta}[\phi_h](\cdot)$ is a one-hot vector, and the target function 
$f$ is bounded, we obtain a closed-form solution for the linear pseudobackup:
\begin{align*}
  \psb_{\cD_h, \cW}[f](x,a) = \frac{\sum_{(x,a,x') \in \cD_h} \indic\{x,a \in \ball{\eta}[\phi_h](x,a)\} f(x')}{\sum_{(x,a) \in \cD_h} \indic\{x,a \in \ball{\eta}[\phi_h](x,a)\}},
\end{align*}
with the convention that $\frac{0}{0}=0$.
Note that the linear pseudobackup $\psb_{\cD_h,
  \cW}[\cdot]$ is equivalent to the value iteration 
backup operator in \pref{alg:main_alg}.

Finally, we will define the \emph{discretized pseudobackup} operator 
$\psb_{\cD_h,\disc{\eta}[\phi_h]}[f]$.
The intuition is that the discretized pseudobackup bridges between the 
continuous pseudobackup and the linear pseudobackup, which is using a 
linear function over the discretized decoder to predict the value of the 
continuous pseudobackup. Formally, we will deviate from the definition 
\pref{eq:pseudobackup} and define 
\begin{align*}
  &\psb_{\cD_h,\disc{\eta}[\phi_h]}[f] = w^{\top} \disc{\eta}[\phi_h], ~~\text{where}\\
  & w = \argmin_{w \in \bbR^{\dimeta}, \|w\|_{\infty} \leq L} 
    \En_{x,a \sim \cD_h}\prn{w^{\top} \disc{\eta}[\phi_h](x,a) -\psb_{\cD_h, \phi_h}[f](x,a) }^2.
\end{align*}
and similarly as above, for each $x,a$, we have 
\begin{align}\label{eq:pseudobackup_disc}
  \psb_{\cD_h,\disc{\eta}[\phi_h]}[f](x,a) := \frac{\sum_{(\widetilde x,\widetilde a) \in \cD_h}\indic\{\widetilde x, \widetilde a \in \ball{\eta}[\phi_h](x,a)\} \psb_{\cD_h, \phi_h}[f](\widetilde x, \widetilde a)}{\sum_{(\widetilde x, \widetilde a) \in \cD} \indic\{\widetilde x, \widetilde a \in \ball{\eta}[\phi_h](x,a)\}},
\end{align}

and now we can see that the discretized pseudobackup is serving as a
discretized version of the continuous pseudobackup, capturing the
value of the continuous pseudobackup averaged over the ball defined by
the same decoder that defines the pseudobackups.

Finally, we will often nest $h$ applications of the pseudobackup operator, with
actions selected according to a given policy $\pi$. To do this, we use the
notation
\begin{align*}
  \brk*{\psb_{\rho,\cH}^\pi}^{\otimes h}[g] := f_1(x,\pi(x)), \textrm{ where } 
  f_h := \psb_{\rho_h,\cH_h}[g], \mathand f_{i} = \psb_{\rho_i,\cH_i}[f_{i+1}(\cdot,\pi(\cdot))] ,\;\; \forall i \in [h-1].
\end{align*}
Similarly, we will define nested reward-free Bellman backup
\begin{align*}
  \brk*{\cP}^{\otimes h}[g] := f_1(x,\pi(x)), \textrm{ where } 
  f_h := \cP_h[g], \mathand f_{i} = \cP_i[f_{i+1}(\cdot,\pi(\cdot))] ,\;\; \forall i \in [h-1].
\end{align*}

\paragraph{Construction of the discriminator class}
We now specify the discriminator class $\cF_h$ used in
\mainalg. Recalling that $\eta>0$ is the discretization
  parameter for \mainalg,  we define the discriminator class to be the union of the following two function classes: 
\begin{small}
  \begin{align} \label{eq:discriminator_class} &\cF^1_{\eta,{h+1}} =
    \crl*{f(x): \En_{a \sim \unifpi}\brk*{w^\top
        \disc{\eta}[\phi](x,a) - g\prn[\big]{\widetilde{\phi}(x,a)}}
      \mid \phi, \widetilde \phi \in \Phi_{h+1}, g \in \Lip, w \in
      \bbR^{\dimeta}, \|w\|_{\infty} \leq 1},\\ \notag
                                               &\cF^2_{\eta,{h+1}} =  \notag\\
                                               &\crl*{f(x): \max_a
                                                 \prn*{\frac{R_{h+1}(x,a)
                                                 + \min\crl*{w^\top
                                                 \disc{\eta}[\phi](x,a),2}}{2H+1}
                                                 + \widetilde w^\top
                                                 \disc{\eta}[\phi](x,a)}
                                                 \mid \phi \in
                                                 \Phi_{h+1}, w,
                                                 \widetilde w \in
                                                 \bbR^{\dimeta},
                                                 \|w\|_{\infty} \leq
                                                 c, \|\widetilde
                                                 w\|_{\infty} \leq 2}
                                                 \notag.
  \end{align}
\end{small}We then set $\cF_{h+1}=\cF^1_{\eta,h+1}\cup\cF^2_{\eta,h+1}$, leaving
  the dependence on $\eta$ implicit.

Let us give some brief intuition behind the construction of the discriminator class.
The first class $\cF^1_{\eta,h+1}$ is the class of functions that can represent the 
following: given a state $x_{h+1}$, what is the expected error between the
linear pseudobackup $\psb_{\cD,\cW}$ (the first term in the expectation) and 
the Bellman backup $\realb$ (the second term in the expectation),
under an action $a_{h+1}$ sampled uniformly via $\piunif_\eta$? We
will show later how to use this construction to transfer the  error
under the linear pseudobackup to the continuous pseudobackup $\psb_{\cD_h,\phi_h}$.
The second discriminator class $\cF^2_{\eta,h+1}$, is
more direct, and aims to represent the optimistic value functions
induced by \pref{alg:optdp} (weighted by $\frac{1}{2H+1}$). 

\subsubsection{Proof Organization}

We organize the proof of \cref{thm:main_alg_formal} into the following
modules:
\begin{itemize}
\item In \pref{sec:simulation_lemma}, we show how to decompose the
  pseudoregret for \cref{alg:main_alg} into an error term defined as a
  difference between the linear pseudobackups defined in the prequel
  and the true Bellman backups---conditioned on establishing optimism,
  a property we will return to later. As in prior work
  \citep{uehara2022representation,zhang2022efficient}, the challenge
  in proceeding from here is that we need to control on-policy error,
  yet we are only guaranteed control over error under the data
  collection distribution (\textbf{Challenge 1}). In addition, we have
  a mismatch between the linear pseudobackup and the continuous
  pseudobackup; the latter of which is close to the true Bellman
  backup with respect to the data collection distributions
  (\textbf{Challenge 2}).
\item  In \pref{sec:pseudobackup_closeness}, we address
  \textbf{Challenge 2} by showing show that the linear
  pseudobackup is close to the continuous pseudobackup \emph{under the
    data collection distribution}.
\item In \pref{sec:one_step_back} we address \textbf{Challenge 1} by adapting the ``one-step-back'' trick used
  in \citet{agarwal2020flambe,uehara2022representation,zhang2022efficient},
  but with a careful treatment to account for misspecification arising
  from discretization.
\item Next, in \pref{sec:optimism}, after combining the results
  above with an in-distribution guarantee for the continuous
  pseudobackup from \cref{thm:rep_learn}, we establish optimism.
\item Finally, in \cref{sec:regret_proof}, we combine the results
  above to prove \pref{thm:main_alg_formal}.
\end{itemize}
Supporting technical lemmas are deferred to \cref{sec:supporting_lemmas}.

\subsubsection{Regret Decomposition by Simulation Lemma}\label{sec:simulation_lemma}
We start with a simulation lemma for the pseudobackups.

\begin{lemma}[Simulation Lemma for pseudobackups]\label{lem:pseudo_simulation}
    Fixed a pseudobackup $\psb_h \ldef \psb_{\cD_h,\cV}$ defined in \pref{eq:pseudobackup} 
    for any dataset $\cD_h$, $\cV \subset (\cX \times \cA) \to [0,L]$,
    and reward function $\widetilde R_h$, for any policy $\pi$, let $f_h$ be its estimated Q functions through the pseudobackup $\psb_h$,
    i.e., for all $h \in [H], f_{h} = \widetilde R_h + \psb^{\pi}_h[f_{h+1}]$, and $f_{H+1} = 0$. Denote the value function 
    induced by $f$ and $\pi$ as $f^\pi_h(x) = f_h(x, \pi(x))$.
    Let $Q^\pi$ and $V^\pi$ be the true Q function and value function induced by 
    $\pi$ and reward function $R$. Then for any $x_1 \in \cX$,
    \begin{align*}
      f^\pi_{1}(x_1) - V^\pi_1(x_1) = \sum_{h=1}^H \brk*{\psb}^{\otimes (h-1)}
      [\delta_h](x_1, \pi_1(x_1)),
    \end{align*}
    where 
    \begin{align*}
      \delta_h = \widetilde R_h - R_h  + (\psb_h^\pi - \realb^{\pi}_h)[Q^\pi_{h+1}] + 
      (\psb^\pi_h - \realb^{\pi}_h)[f_{h+1}]+ (\psb_h^\pi - \realb^{\pi}_h)[(f_{h+1}-Q^\pi_{h+1})] + 
      (\realb^{\pi}_h- \psb_h^\pi)[Q^\pi_{h+1}].
    \end{align*}
    When the pseudobackup operator is linear, i.e., $\psb[f + f'] = \psb[f] + \psb[f']$, 
    we have
    \begin{align*}
    f^\pi_{1}(x_1) - V^\pi_1(x_1) =  \sum_{h=1}^H \brk*{\psb}^{\otimes (h-1)}
      [\widetilde R_h - R_h  + (\psb_h^\pi - \realb^{\pi}_h)[Q^\pi_{h+1}]](x_1, \pi(x_1)).
    \end{align*}
  \end{lemma}
  \begin{proof}[\pfref{lem:pseudo_simulation}]
    By construction, we have
    \begin{align*}
      &~~~~f_{1}^{\pi}(x_1) - V_1^\pi(x_1) \\
      &= \widetilde R_1(x_1, \pi(x_1)) + \psb_1\brk*{f_2^\pi}(x_1, \pi(x_1)) - R_1(x_1, \pi_1(x_1)) - \realb_1[V_2^\pi](x_1, \pi(x_1)) \\
      &= (\widetilde R_1 - R_1)(x_1, \pi(x_1)) +  \psb_1[f_2^\pi](x_1, \pi(x_1)) - 
      \psb_1[V_2^\pi](x_1, \pi(x_1)) + \psb_1[V^\pi_2](x_1, \pi(x_1)) - \realb_1[V_2^\pi](x_1, \pi(x_1)) \\
      &= (\widetilde R_1 - R_1)(x_1, \pi(x_1)) + (\psb_1 - \realb_1)[V^\pi_2](x_1, \pi(x_1)) 
      + \psb_1[f_2^\pi](x_1, \pi(x_1)) - \psb_1[V^\pi_2](x_1, \pi(x_1)) \\
      &= (\widetilde R_1 - R_1)(x_1, \pi(x_1)) + (\psb_1 - \realb_1)[V^\pi_2](x_1, \pi(x_1)) + 
      (\psb_1 - \realb_1)[f^\pi_{2}](x_1, \pi(x_1)) +
      (\psb_1 - \realb_1)[(f^\pi_{2}-V^\pi_2)](x_1, \pi(x_1)) +  \\&~~~~
      (\realb_1 - \psb_1)[V^\pi_2](x_1, \pi(x_1)) + \underbrace{{\widetilde\realb_1[(f^\pi_{2}-V^\pi_2)](x_1, \pi(x_1))}}_{\text{recursion}}.
    \end{align*}
    Then let 
    \begin{align*}
      \delta_h := \widetilde R_h - R_h  + (\psb_h - \realb_h)[V^\pi_{h+1}] + 
      (\psb_h - \realb_h)[f^\pi_{h+1}]+ (\psb_h - \realb_h)[(f^\pi_{h+1}-V^\pi_{h+1})] + 
      (\realb^{\pi}_h- \psb_h^\pi)[V^\pi_{h+1}],
    \end{align*}
    we get 
    \begin{align*}
      f^\pi - V^{\pi} = \sum_{h=1}^H \brk*{\psb}^{\otimes (h-1)}
      [\delta_h](x_1, \pi(x_1)).
    \end{align*}
    Note that when $\psb$ is linear, i.e., $\psb[f_1 + f_2] = 
    \psb[f_1] + \psb[f_2]$, we have 
    \begin{align*}
      (\psb_h - \realb_h)[f^\pi_{h+1}]+ (\psb_h - \realb_h)[(f^\pi_{h+1}-V^\pi_{h+1})] + 
      (\realb^{\pi}_h- \psb_h^\pi)[V^\pi_{h+1}] = 0,
    \end{align*}
    and thus 
    \begin{align*}
      f^\pi - V^{\pi} = \sum_{h=1}^H \brk*{\psb}^{\otimes (h-1)}
      [\widetilde R_h - R_h  + (\psb_h - \realb_h)[V^\pi_{h+1}]](x_1, \pi(x_1)),
    \end{align*}
    finally, note that $\realb_h[V^\pi_{h+1}] = 
    \realb^\pi_h[Q^\pi_{h+1}]$ and thus we complete the proof.
  \end{proof}

\begin{lemma}[Simulation Lemma for Bellman backups]\label{lem:bellman_simulation}
  Under the similar setup as \pref{lem:pseudo_simulation}, we have
  \begin{align*}
    f^\pi_{1}(x_1) - V^\pi_1(x_1) =  \sum_{h=1}^H \brk*{\psb}^{\otimes (h-1)}
      [\widetilde R_h - R_h  + (\psb_h^\pi - \realb^{\pi}_h)[f_{h+1}]](x_1, \pi(x_1)).
    \end{align*}
\end{lemma}
The above result is the classic simulation lemma rewritten in the reward-free 
Bellman backup notation so we omit the proof here (e.g.,
see \citep{sun2019model}).

\subsubsection{Closeness between Pseudobackups}\label{sec:pseudobackup_closeness}

In this section we solve one major difficulty of our analysis: to show that the following 
two pseudobackups are close to each other: 1) $\psb_{\cD_h,\phi^t_h}$, which we have 
the representation learning guarantee, and 2) $\psb_{\cD_h,\cW}$, which we use for 
planning and exploration. First we can see that for any $f \in \cF_{h+1}$, 
due to the fact that the predictor class
used to define $\psb_{\cD_h, \phi_h^t}$ is 1-Lipschitz, we have
\begin{align*}
\nrm*{\psb_{\cD_h,\disc{\eta}[\phi_h^t]}[f] - \psb_{\cD_h,\phi_h^t}[f] }_{\infty} \leq \eta.
\end{align*}
Then the remaining part is to show that $\psb_{\cD_h,\disc{\eta}[\phi_h^t]}$ and
$\psb_{\cD_h, \cW}$ are close. For this part, we only need to show that 
$\psb_{\cD_h,\disc{\eta}[\phi_h^t]}$ and $\psb_{\cD_h, \cW}$ are close
under the training distribution. We first introduce some notations that 
we use in this section: for any function $f \in \cF: \cX \to \bbR$, 
and $P$ which is a probability measure over $\cX$, we denote 
$\nrm*{f}^2_{L_2(P)} = \En_{x \sim P}\brk*{f^2(x)}.$ With this we are ready
to state the closeness result:

\begin{lemma}[Closeness between pseudobackups]\label{lem:pseudobackup_closeness}
  For any round $t$, for all $h \in [H]$ and $f \in \cF_{h+1}$, let 
  $\rho^t_h$ be the data generating distribution of $\cD^t_h$, we have
  \begin{align*}
    \|\psb_{\cD^t_h,\cW}[f] - \psb_{\cD^t_h,\disc{\eta}(\phi^t_h)}[f]\|_{L_2(\rho_h^t)} \leq 
    \erep(t) + 2\veps_{\rm{hist}}(t).
  \end{align*}
  And thus 
  \begin{align*}
    \|\psb_{\cD^t_h,\cW}[f] - \psb_{\cD^t_h,\phi^t_h}[f]\|_{L_2(\rho_h^t)} \leq 
    \erep(t) + 2\veps_{\rm{hist}}(t) + 2\eta.
  \end{align*}
\end{lemma}
\begin{proof}[Proof of \pref{lem:pseudobackup_closeness}]
The proof is based on the following observation that both the linear 
pseudobackup and the discretized pseudobackup are histograms.
Given a distribution $\rho$ over $\cX$, let $\cD^n$ denote a dataset with $n$ samples
drawn i.i.d. from $\rho$. Given any function $f \in \cF: \cX \to \bbR$, 
the population histogram with bins $\cB = \{b_i\}_{i=1}^B$ of $\Gamma(f)$ is defined as 
\begin{align*}
\Gamma(f)(x) = \sum_{i=1}^B \indic\{x \in b_i\} \int_{b_i} f(x') \rho(x'), 
\end{align*}
and the empirical histogram of $\Gamma_n(f)$ is defined as
\begin{align*}
\Gamma_n(f)(x) = \sum_{i=1}^B \indic\{x \in b_i\} \frac{1}{n_i} \sum_{j=1}^n \indic\{x_j \in b_i\} f(x_j) .
\end{align*}
where $n_i$ is the number of data in bin $b_i$. 

Then  we observe that, 
for any $t,h$, denote the histogram according to the discretized feature 
$\disc{\eta}(\phi^t_h)$ and the distribution $\rho_h^t$ as $\Gamma_{h,t}$,
and the empirical histogram according to $\cD^t_h$ as $\widehat \Gamma_{h,t}$,
then we have the following identities: for any $f \in \cF_h$,
\begin{align}\label{eq:pseudo_histo}
\psb_{\cD,\cW}[f](x,a) = \widehat \Gamma_{h,t}\prn*{\En[f(x_{h+1}) \mid x_h = x, a_h = a]} \mathand 
\psb_{\cD,\disc{\eta}[\phi]}[f](x,a) = \widehat \Gamma_{h,t}\prn*{\psb_{\cD,\phi^t_h}[f](x,a)}.
\end{align}
In \pref{lem:hist_conv}, we prove that, 
if the target functions of two empirical histograms (with the same set of bins) are close to each other, then 
the two empirical histograms are close to each other in distribution as well. 
In our case, the target functions $\cP[f]$ and $\psb_{\cD,\phi^t_h}[f]$
are indeed close to each other in distribution (by \pref{thm:rep_learn}), then plugging
in the guarantee of \pref{thm:rep_learn} into \pref{lem:hist_conv}
we complete the proof. 
\end{proof}

Now we state and prove the result that histograms approximately
preserve closeness of the target functions:

\begin{lemma}[Empirical histograms approximately preserve closeness]
  \label{lem:emp_closeness}
Using notations from \pref{lem:pseudobackup_closeness}, suppose we have functions $f, g \in \cF: \cX \to \bbR$
are close: $\|f-g\|_{L_2(\rho)} \leq \veps$. Then we have
\begin{align*}
    \|\Gamma_n(f) - \Gamma_n(g)\|_{L_2(\rho)} \leq \veps + 2\veps_{\rm{hist}}(n).
\end{align*}
\end{lemma}
\begin{proof}[\pfref{lem:emp_closeness}]
  Since $\nrm{\cdot}_{L_2(P)}$ is a metric, by triangle inequality we have
  \begin{align*}
    \|\Gamma_n(f) - \Gamma_n(g)\|_{L_2(\rho)} \leq \|\Gamma_n(f) - \Gamma(f)\|_{L_2(\rho)} + \|\Gamma(f) - \Gamma(g)\|_{L_2(\rho)} + \|\Gamma_n(g) - \Gamma(g)\|_{L_2(\rho)}.
  \end{align*}
  Now the first and third terms are the differences between an empirical
  histogram and the population one with the same target, and we prove 
  the difference is small by \pref{lem:hist_conv} with standard 
  concentration result. For the second term, since histogram
  is a convex projection, then by \pref{lem:convex_proj} we have that
  convex projection preserves closeness, and by the assumption that 
  $f,g$ are close we complete the proof. 
\end{proof}

\begin{lemma}[Convex projections preserve closeness] \label{lem:convex_proj}
    Let $f, g \in \cF: \cX \to \bbR$ be two functions, and $P$ be a probability measure 
    over $\cX$. Suppose that $f$ and $g$ are close:
    \begin{align*}
      \|f-g\|_{L_2(P)}^2 = \En_{x \sim P} \brk*{ (f(x) - g(x))^2 } \leq \veps.
    \end{align*} 
    Then suppose we have a convex projection $\Gamma: \cF \to \cH$, where 
    $\cH$ is a convex function class, i.e., 
    \begin{align*}
      \Gamma(f) = \argmin_{h \in \cH} \nrm*{h-f}_{L_2(P)},
    \end{align*}
    then we have 
    \begin{align*}
      \|\Gamma(f) - \Gamma(g)\|_{L_2(P)}^2 \leq \|f-g\|_{L_2(P)}^2 \leq \veps.
    \end{align*}
  \end{lemma}
\begin{proof}[\pfref{lem:convex_proj}]
  Let us define the notation of inner product in the function space under $P$:
  \begin{align*}
    \langle f, g \rangle_P = \En_{x \sim P} \brk*{f(x) \cdot g(x)}.
  \end{align*}
  Then by the convexity of $\cH$, we have
  \begin{align*}
    \langle \Gamma(f) - \Gamma(g), \Gamma(f)-f \rangle_P \leq 0 \mathand 
    \langle \Gamma(g) - \Gamma(f), \Gamma(g)-g \rangle_P \leq 0.
  \end{align*}
  Expanding the inner product and summing them together and rearranging the terms, we have
  \begin{align*}
    \|\Gamma(f) - \Gamma(g)\|_{L_2(P)}^2 \leq \langle \Gamma(g) - \Gamma(f), f-g \rangle_P.
  \end{align*}
  Finally by AM-GM we complete the proof.
\end{proof}

\begin{lemma}[Concentration of histograms]\label{lem:hist_conv}
Let $\rho$ be a distribution over $\cX$, where $\rho = \frac{1}{n}\sum_{i=1}^n \rho^i$, 
and each $\rho^i$ may depend on the randomness in previous rounds. Let $\cD^n$ be a dataset with $n$ samples
drawn i.i.d. from $\rho$. Then for any $f \in \cF: \cX \to \bbR, \|f\|_{\infty} \leq L$, 
for histogram $\Gamma$ with $B$ bins, with probability at least $1-\delta$,
\begin{align*}
    \|\Gamma(f) - \Gamma_n(f)\|^2_{L_2(\rho)} \leq \veps_{\rm{hist}}(n) = \cO\prn*{\frac{BL^2\log(|\cF|/\delta)}{n}}.
\end{align*}
\end{lemma}
\begin{proof}[\pfref{lem:hist_conv}]
  Let us fix $f \in \cF$, and define $\rho_{\cB}$ as the distribution of each bin under $\rho$. Then we have 
  \begin{align*}
    \nrm{\Gamma(f) - \Gamma_n(f)}^2_{L_2(\rho)} &= \En_{x \sim \rho} \brk*{ \prn*{\Gamma(f)(x) - \Gamma_n(f)(x)}^2 } \\
    &= \En_{b \sim \rho_{\cB}} \brk*{\En \brk*{ \prn*{\Gamma(f)(x) - \Gamma_n(f)(x)}^2 \mid b} }\\
    &\leq \En_{b \sim \rho_{\cB}} \frac{L^2 \log(1/\delta)}{n(b)} \\
    &\leq \frac{BL^2\log(1/\delta)}{n},
  \end{align*}
  where the first inequality is by standard concentration argument from 
  observing that conditioned on bin $b$, the empirical histogram $\Gamma_n(f)$
  converges to expected histogram $\Gamma(f)$ in bin $b$. Note that $b$ here is a random variable, and $n(b)$ 
  denotes the number of data from bin $b$ in the dataset $\cD^n$. Finally taking a union bound over $\cF$ we complete the proof.
\end{proof}
  
\subsubsection{Error Transfer by the One-step-back Trick}\label{sec:one_step_back}

  One difficulty we mentioned in the main text is that, in general, the pseudobackup operators do not preserve the order of functions that they take on. Without this property, it is hard to bound the representation error under the induced policies
  by transferring to the representation error
  under the data collection distribution, in order to prove optimism. Luckily, we can show that the linear pseudobackup operator preserves the order of functions, but in general it is not clear if the continuous pseudobackup preserves the order of functions. Combined with the result in the last section that the linear pseudobackup is close to the continuous pseudobackup in-distribution, we can use the one-step-back trick to leverage the representation learning results. We show that the linear pseudobackup preserves the order of functions in \pref{lem:order_preserve}, which gives the following distribution shift 
  result. We instantiate $\tilde \cP$ to be the linear 
  pseudobackup, but the result holds for any pseudobackup that is 
  piecewise constant with respect to $\disc{\eta}[\phi^t_h]$ and monotone: 
    
  \begin{lemma}[One-step-back for linear pseudobackup] \label{lem:osb_pseudobackup} Let $\tilde P \ldef \psb_{\cD_h, \cW}$.
    Conditioned on the event that \pref{thm:rep_learn} holds for testing distribution $\rho_h$ with 
    error $\erep$, and \pref{lem:pseudobackup_closeness} holds with error $\ehist$. Then for any 
    set of functions $\{f_h\}_{h=1}^H$ where $f_h \in (\cX \times \cA \to [-L,L])$, and 
    $f_h(\cdot, \unifpi) \in \cF_{h}$ for all $h \in [H]$, for any policy $\pi$,  
    \begin{align*}
      &\sum_{h=1}^H \brk*{\psb}^{\otimes (h-1)}[f_h](x_1, \pi(x_1)) \leq \\
      &\sum_{h=2}^H \brk*{\psb}^{\otimes (h-1)} \min \crl*{\sqrt{\frac{1}{t \bar \rho^t_h\brk*{\ball{\eta}[\phi^t_h](\cdot)}+ \lambda^t} }
      \sqrt{2tA_\eta^2 \En_{x,a \sim \bar \gamma_h^t} \brk*{ f_{h+1}^2(x,a) } + \zeta(t)},L} + \sqrt{A \En_{x,a \sim \bar \rho^t_1}\brk*{f_1^2(x,a)}} + 4H\eta,
    \end{align*}
    where 
    \begin{align*}
      \zeta(t) = 4\veps_{\rm{hist}}(t)  + 4A_\eta^2\erep(t) + 18\lambda^t L^2 \dimeta.
    \end{align*}
    \end{lemma}
    
    \begin{proof}[\pfref{lem:osb_pseudobackup}]
    For $h=1$, we have:
    \begin{align*}
      g_1(x_1, \pi(x_1)) &= \En_{a \sim \pi(x_1)} \brk*{g_1(x_1,a)} \\
      &\leq \sqrt{\max_{a \in \cA_\eta} \frac{\pi(a \mid x_1)}{\unifpi(a \mid x_1)} \En_{x,a \sim \bar \rho^t_1} \brk*{g_1^2(x,a)}} 
      \tag{Jensen}\\
      &\leq \sqrt{A_\eta \En_{x,a \sim \bar \rho^t_1}\brk*{g_1^2(x,a)}}.
    \end{align*}
    For $h = 2, \dots, H$, we have
    \begin{align*}
      &\brk*{\psb}^{\otimes h}[f_{h+1}] \\
      = &\brk*{\psb}^{\otimes h-1} \brk*{\psb^\pi_h[f_{h+1}]} \\
      = &\brk*{\psb}^{\otimes h-1} \brk*{
        \sum_{b_\eta \in \cB_\eta[\phi_{h}]}
        \indic\{\phi_{h}(\cdot) \in b_\eta\}
        \psb^\pi_h[f_{h+1}](x_\eta,a_\eta)} \\
      \leq &\brk*{\psb}^{\otimes h-1}\brk*{\sum_{b_\eta \in \cB_\eta[\phi_{h}]}
      \min\crl*{\frac{\indic\{\phi_{h}(\cdot) \in b_\eta\}}{\sqrt{t\bar{\rho}^t_{h}(b_\eta)+ \lambda^t} }
      \sqrt{\prn*{t\bar{\rho}^t_{h}(b_\eta)+\lambda^t} \prn*{\psb^\pi_h[f_{h+1}](x_\eta,a_\eta)}^2},L}},
      \end{align*}
      where each $x_\eta,a_\eta \in {\phi^t_h}^{-1}(s_\eta, a_\eta)$, where $s_\eta, a_\eta$ is the covering point in the latent space for the ball $b_\eta$, and the second equality is due to the fact that $\psb_h$ is piecewise constant with respect to $\disc{\eta}[\phi^t_h]$.
       
      Focusing on the function inside the pseudobackup, we have the following pointwise inequality:
      \begin{align*}
      &\sum_{b_\eta \in \cB_\eta[\phi^t_h]}\frac{\indic\{\phi_{h}(\cdot) \in b_\eta\}}{\sqrt{t\bar{\rho}^t_{h}(b_\eta)+ \lambda^t} }
      \sqrt{\prn*{t\bar{\rho}^t_{h}(b_\eta)+\lambda^t} \prn*{\psb^\pi_h[f_{h+1}](x_\eta,a_\eta)}^2}\\
      \leq &\sum_{b_\eta \in \cB_\eta[\phi^t_h]}\frac{\indic\{\phi_{h}(\cdot) \in b_\eta\}}{\sqrt{t\bar{\rho}^t_{h}(b_\eta)+ \lambda^t} }
      \sqrt{ \int_{b_\eta} \prn*{t\bar \rho^t_h(x,a) + \lambda^t} \prn*{\psb^\pi_h[f_{h+1}](x,a)}^2 \bm{x,a}} \tag{piecewise constant} \\
      \leq &\sum_{b_\eta \in \cB_\eta[\phi^t_h]}\frac{\indic\{\phi_{h}(\cdot) \in b_\eta\}}{\sqrt{t\bar{\rho}^t_{h}(b_\eta)+ \lambda^t} }
      \sqrt{  \int_{b_\eta} \prn*{t\bar \rho^t_h(x,a) + \lambda^t}\prn*{\prn*{\psb^\pi_h - \psb^\pi_{\cD_h, \disc{\eta}[\phi^t_h]}}[f_{h+1}](x,a) + 
      \psb^\pi_{\cD_h, \disc{\eta}}[f_{h+1}](x,a)}^2 \bm{x,a}} \\
      \leq &\sum_{b_\eta \in \cB_\eta[\phi^t_h]}\frac{\indic\{\phi_{h}(\cdot) \in b_\eta\}}{\sqrt{t\bar{\rho}^t_{h}(b_\eta)+ \lambda^t} }
      \sqrt{  \int_{b_\eta} \prn*{t\bar \rho^t_h(x,a) + \lambda^t}\prn*{\prn*{\psb^\pi_h - \psb^\pi_{\cD_h, \disc{\eta}[\phi^t_h]}}[f_{h+1}](x,a) + 
      \psb^\pi_{\cD_h,\phi^t_h}[f_{h+1}](x,a) + 2\eta}^2 \bm{x,a}} \\
      \leq &\sum_{b_\eta \in \cB_\eta[\phi^t_h]}\frac{\indic\{\phi_{h}(\cdot) \in b_\eta\}}{\sqrt{t\bar{\rho}^t_{h}(b_\eta)+ \lambda^t} }
      \sqrt{ 2 \int_{b_\eta} \prn*{t\bar \rho^t_h(x,a) + \lambda^t}\prn*{\prn*{\psb^\pi_h - \psb^\pi_{\cD_h, \disc{\eta}[\phi^t_h]}}[f_{h+1}](x,a)}^2 + \prn*{
      \psb^\pi_{\cD_h,\phi^t_h}[f_{h+1}](x,a) }^2 \bm{x,a}} \\ & \hskip0.9\textwidth + 4\eta + \lambda^t L \dimeta \\
      \leq &\sqrt{\sum_{b_\eta \in \cB_\eta[\phi^t_h]}\frac{\indic\{\phi_{h}(\cdot) \in b_\eta\}}{t\bar{\rho}^t_{h}(b_\eta)+ \lambda^t} } \cdot \\
      &\sqrt{2\sum_{b_\eta \in \cB_\eta[\phi^t_h]} \int_{b_\eta} \prn*{t \bar \rho^t_h(x,a) + \lambda^t} \prn*{\prn*{\psb^\pi_h - \psb^\pi_{\cD_h, \disc{\eta}[\phi^t_h]}}[f_{h+1}](x,a)}^2 + 
      \prn*{\psb_{\cD_h,\phi^t_h}^{\pi}[f_{h+1}](x,a) }^2 \bm{x,a}} + 4\eta+ \lambda^t L \dimeta.
      \end{align*}
      Then we focus on the terms inside the second square root:
      \begin{align*}
        &\sum_{b_\eta \in \cB_\eta[\phi^t_h]} \int_{b_\eta} \prn*{t \bar \rho^t_h(x,a) + \lambda^t} \prn*{\prn*{\psb^\pi_h - \psb^\pi_{\cD_h, \disc{\eta}[\phi^t_h]}}[f_{h+1}](x,a) + 
      \psb_{\cD_h,\phi^t_h}^{\pi}[f_{h+1}](x,a) }^2 \bm{x,a}\\
      \leq& t\En_{x,a \sim \bar \rho^t_h} \brk*{ \prn*{\prn*{\psb^\pi_h - \psb^\pi_{\cD_h, \disc{\eta}[\phi^t_h]}}[f_{h+1}](x,a)}^2 + 
      \prn*{\psb_{\cD_h,\phi^t_h}^{\pi}[f_{h+1}](x,a) }^2} + 9L^2 \lambda^t \dimeta\\
      \leq& t\erep(t) + 2t\ehist(t) + t\En_{x,a \sim \bar \rho^t_h} \brk*{ \prn*{\psb_{\cD_h,\phi^t_h}^{\pi}[f_{h+1}](x,a) }^2}
      \tag{\pref{lem:pseudobackup_closeness}} +  9L^2 \lambda^t \dimeta \\  
      \leq& t\erep(t) + 2t\ehist(t) + tA_\eta^2 \En_{x,a \sim \bar \rho^t_h} \brk*{ \prn*{\psb_{\cD_h,\phi^t_h}^{\unifpi}[f_{h+1}](x,a) }^2}
      +  9L^2 \lambda^t \dimeta \tag{Importance sampling}\\
      \leq& t\erep(t) + 2t\ehist(t) + tA_\eta^2 \En_{x,a \sim \bar \rho^t_h} \brk*{ \brk*{\En_h^{\unifpi}[f_{h+1}](x,a) }^2} + tA_\eta^2 \erep(t)
      +  9L^2 \lambda^t \dimeta \tag{\pref{thm:rep_learn}}\\
      \leq& t\erep(t) + 2t\ehist(t) + tA_\eta^2 \En_{x,a \sim \bar \gamma_h^t} \brk*{ f_{h+1}^2(x,a) } + tA_\eta^2 \erep(t) +  9L^2 \lambda^t \dimeta
      \tag{Jensen}.
    \end{align*}
    Since every inequality above holds in the point-wise way, then by \pref{lem:order_preserve}, 
    putting everything together we have for each $h \geq 2$,
    \begin{align*}
      &\brk*{\psb}^{\otimes h}[f_{h+1}] \\
      \leq & \brk*{\psb}^{\otimes (h-1)} \sqrt{\sum_{b_\eta \in \cB_\eta[\phi^t_h]}\frac{\indic\{\phi_{h}(\cdot) \in b_\eta\}}{t\bar{\rho}^t_{h}(b_\eta)+ \lambda^t} }
      \sqrt{2tA_\eta^2 \En_{x,a \sim \bar \gamma_h^t} \brk*{ f_{h+1}^2(x,a) } + 4tA_\eta^2 \erep(t) + 4t\ehist(t)+  18L^2 \lambda^t \dimeta} + 4\eta \\
      \leq&  \brk*{\psb}^{\otimes (h-1)} \sqrt{\frac{1}{t \bar \rho^t_h\brk*{\ball{\eta}[\phi^t_h](\cdot)}+ \lambda^t} }
      \sqrt{2tA_\eta^2 \En_{x,a \sim \bar \gamma_h^t} \brk*{ f_{h+1}^2(x,a) } + \zeta(t)} + 4\eta, 
    \end{align*}
    where 
    \begin{align}\label{eq:zeta_def}
      \zeta(t) := 4tA_\eta^2 \erep(t) + 4t\ehist(t)+  18L^2 \lambda^t \dimeta,
    \end{align}
    finally, summing over $h \in [H]$ we complete the proof.
    \end{proof}

    After the distribution shift result for linear pseudobackup, we 
    next state a similar result for the reward-free Bellman backup: 
    
    \begin{lemma}[One-step-back for Bellman backup]\label{lem:osb_bellman}
       For any 
      set of functions $\{f_h\}_{h=1}^H$ where $f_h \in (\cX \times \cA \to [-L,L])$, 
      and any policy $\pi$, we have
      \begin{align*}
        &\sum_{h=1}^H \brk*{\realb}^{\otimes (h-1)}[f_h](x_1, \pi(x_1)) \leq \\
        &\sum_{h=2}^H \brk*{\realb}^{\otimes (h-1)} \sqrt{\frac{1}{t \bar \rho^t_h\brk*{\ball{\eta}[\phi^\ast_h](\cdot)}+ \lambda^t} }
        \sqrt{tA_\eta^2 \En_{x,a \sim \bar \gamma_h^t} \brk*{ f_{h+1}^2(x,a) } + \lambda^t L^2 \dimeta} + \sqrt{A \En_{x,a \sim \bar \rho^t_1}\brk*{f_1^2(x,a)}} + 3H\eta.
      \end{align*}
    \end{lemma}

    \begin{proof}[\pfref{lem:osb_bellman}]
      The proof is mostly similar to the proof of the previous one-step-back lemma. To start, for $h=1$, by Jensen's inequality, we have:
      \begin{align*}
        f_1(x_1, \pi(x_1)) &= \En_{a \sim \pi(x_1)} \brk*{f_1(x_1,a)} 
        \leq \sqrt{\max_{a \in \cA_\eta} \frac{\pi(a \mid x_1)}{\unifpi(a \mid x_1)} \En_{x,a \sim \bar \rho^t_1} \brk*{f_1^2(x,a)}} 
        \leq \sqrt{A_\eta \En_{x,a \sim \bar \rho^t_1}\brk*{f_1^2(x,a)}}.
      \end{align*}
      Then for $h\geq 2$, we have
      \begin{align*}
        &\brk*{\realb}^{\otimes h}[f_{h+1}] \\
        = &\brk*{\realb}^{\otimes h-1} \brk*{\realb_h[f_{h+1}]} \\
        \leq &\brk*{\realb}^{\otimes h-1} \brk*{\realb_{\disc{\eta}[\phi^\ast_h]}[f_{h+1}]} + \eta\\
        = &\brk*{\realb}^{\otimes h-1}\brk*{\sum_{b_\eta \in \cB_\eta[\phi^\ast_h]}\frac{\indic\{\phi^\ast_{h}(\cdot) \in b_\eta\}}
        {\sqrt{t\bar{\rho}^t_{h}(b_\eta)+ \lambda^t} }
        \sqrt{\prn*{t\bar{\rho}^t_{h}(b_\eta)+\lambda^t} \brk*{\realb_{\disc{\eta}[\phi^\ast_h]}[f_{h+1}](x_\eta,a_\eta)}^2}} + \eta,
        \end{align*}
      Once again focusing on the terms inside the expectation, we have the following pointwise inequality:
      \begin{align*}
        &\sum_{b_\eta \in \cB_\eta[\phi^\ast_h]}\frac{\indic\{\phi^\ast_{h}(\cdot) \in b_\eta\}}{\sqrt{t\bar{\rho}^t_{h}(b_\eta)+ \lambda^t} }
        \sqrt{\prn*{t\bar{\rho}^t_{h}(b_\eta)+\lambda^t} \brk*{\realb_{\disc{\eta}[\phi^\ast_h]}[f_{h+1}](x_\eta,a_\eta)}^2}\\
        \leq &\sum_{b_\eta \in \cB_\eta[\phi^\ast_h]}\frac{\indic\{\phi^\ast_{h}(\cdot) \in b_\eta\}}{\sqrt{t\bar{\rho}^t_{h}(b_\eta)+ \lambda^t} }
        \sqrt{ \int_{b_\eta} \prn*{t\bar \rho^t_h(x,a) + \lambda^t} \brk*{\realb_{\disc{\eta}[\phi^\ast_h]}[f_{h+1}](x,a)}^2 \bm{x,a}} \tag{piecewise constant} \\
        \leq &\sum_{b_\eta \in \cB_\eta[\phi^\ast_h]}\frac{\indic\{\phi^\ast_{h}(\cdot) \in b_\eta\}}{\sqrt{t\bar{\rho}^t_{h}(b_\eta)+ \lambda^t} }
        \sqrt{ \int_{b_\eta} \prn*{t\bar \rho^t_h(x,a) + \lambda^t} \brk*{\realb_{h}[f_{h+1}](x,a) + \eta}^2 \bm{x,a}} \\
        \leq &\sum_{b_\eta \in \cB_\eta[\phi^\ast_h]}\frac{\indic\{\phi^\ast_{h}(\cdot) \in b_\eta\}}{\sqrt{t\bar{\rho}^t_{h}(b_\eta)+ \lambda^t} }
        \sqrt{ \int_{b_\eta} 2\prn*{t\bar \rho^t_h(x,a) + \lambda^t} \brk*{\realb_{h}[f_{h+1}](x,a)}^2 \bm{x,a}} + 2\eta \\
        \leq &\sqrt{\sum_{b_\eta \in \cB_\eta[\phi^\ast_h]}\frac{\indic\{\phi^\ast_{h}(\cdot) \in b_\eta\}}{t\bar{\rho}^t_{h}(b_\eta)+ \lambda^t} }
        \sqrt{2\sum_{b_\eta \in \cB_\eta[\phi^\ast_h]} \int_{b_\eta} \prn*{t \bar \rho^t_h(x,a) + \lambda^t} \brk*{\realb_{h}[f_{h+1}](x,a)}^2 \bm{x,a}} + 2\eta \\
        \leq &\sqrt{\sum_{b_\eta \in \cB_\eta[\phi^\ast_h]}\frac{\indic\{\phi^\ast_{h}(\cdot) \in b_\eta\}}{t\bar{\rho}^t_{h}(b_\eta)+ \lambda^t} }
        \sqrt{2\sum_{b_\eta \in \cB_\eta[\phi^\ast_h]} \int_{b_\eta} t \bar \rho^t_h(x,a) \brk*{\realb_{h}[f_{h+1}](x,a)}^2 \bm{x,a} + L^2 \lambda^t \dimeta} + 2\eta \\
        = &\sqrt{\sum_{b_\eta \in \cB_\eta[\phi^\ast_h]}\frac{\indic\{\phi^\ast_{h}(\cdot) \in b_\eta\}}{t\bar{\rho}^t_{h}(b_\eta)+ \lambda^t} }
        \sqrt{2t \En_{\bar \rho^t_h} \brk*{\realb_{h}[f_{h+1}](x,a)}^2 \bm{x,a} + L^2 \lambda^t \dimeta} + 2\eta \\
        \leq &\sqrt{\sum_{b_\eta \in \cB_\eta[\phi^\ast_h]}\frac{\indic\{\phi^\ast_{h}(\cdot) \in b_\eta\}}{t\bar{\rho}^t_{h}(b_\eta)+ \lambda^t} }
        \sqrt{2tA_\eta^2 \En_{\bar \rho^t_h} \brk*{\brk*{\En^{\unifpi}_{h}[f_{h+1}](x,a)}^2}  + L^2 \lambda^t \dimeta} + 2\eta \tag{Importance sampling} \\
        \leq &\sqrt{\sum_{b_\eta \in \cB_\eta[\phi^\ast_h]}\frac{\indic\{\phi^\ast_{h}(\cdot) \in b_\eta\}}{t\bar{\rho}^t_{h}(b_\eta)+ \lambda^t} }
        \sqrt{2tA_\eta^2 \En_{\bar \gamma^t_h} \brk*{f^2_{h+1}(x,a)} + L^2 \lambda^t \dimeta} + 2\eta \tag{Jensen}. 
        \end{align*}
        Finally putting everything together we complete the proof.
      \end{proof}

      \begin{lemma}[Monotonicity of linear pseudobackup] \label{lem:order_preserve}
        Let $f, f' \in \cF_{h+1}: \cX \to [0,L]$ be two functions such that $f(x) \leq f'(x)$ for all $x \in \cX$. 
        Then we have $\psb_{\cD_h, \cW}[f](x) \leq \psb_{\cD_h, \cW}[f'](x)$.
        \end{lemma}
      
      \begin{proof}[Proof of \pref{lem:order_preserve}]
        Recall the definition of the linear pseudobackup, fix the decoder $\phi$, the discretized 
        decoder $\disc{\eta}[\phi](x,a)$, is a $\dimeta \ldef \prn*{\frac{1}{\eta}}^{\dimsa}$ dimensional
        one-hot vector, for any $x,a \in \cX \times \cA$. Then if we set the $\ell_\infty$ constraint on 
        $\cW$ to be $\cW = \{w \in \bbR^\dimeta \mid \nrm*{w}_{\infty} \leq L\}$, we have that 
        \begin{align*}
          \psb_{\cD_h,\cW}[f](x,a) := \frac{\sum_{(\widetilde x,\widetilde a, x') \in \cD_h}\indic\{\widetilde x, \widetilde a \in \ball{\eta}[\phi_h](x,a)\} f(x')}{\sum_{(\widetilde x, \widetilde a) \in \cD} \indic\{\widetilde x, \widetilde a \in \ball{\eta}[\phi_h](x,a)\}},
        \end{align*}
        and by the non-negativity of indicator functions we complete the proof.
      \end{proof}

\subsubsection{Proving Optimism}\label{sec:optimism}
        
\begin{lemma}[Almost optimism]\label{lem:optimism}
For any round $t$, let 
\begin{align*}
  \widehat \alpha^t = \sqrt{2tA^2_\eta \erep(t) + \zeta(t)}/c \mathand 
  \lambda_T = \Theta\prn*{\dimeta \ln\prn*{T \abs*{\Phi} / \delta}},
\end{align*} 
where $\zeta(t)$ is defined in \pref{eq:zeta_def}, and $c$ is a constant from \pref{lem:concentration_bonus}.
Then for any policy $\pi^t$ derived by the linear pseudobackup operator,
let $f^{\pi^t}$ denotes the value function induced from the pseudobackup operator in \pref{alg:optdp}, then with probability at least $1-\delta$, we have
\begin{align*}
    f^{\pi^t}_1(x_1) - V^{\pi^t}_1(x_1) \geq -\prn*{\sqrt{A_\eta \erep(t)} + 3H\eta}.
\end{align*}
\end{lemma}
\begin{proof}[\pfref{lem:optimism}]
By \pref{lem:pseudo_simulation}, we have
\begin{align*}
    f^{\pi^t}(x_1) - V^{\pi^t}(x_1) = \sum_{h=1}^H \prn*{\psb_{\cD_h, \disc{\eta}[\phi^t]}}^{\otimes (h-1)}
    \brk*{\widehat b^t_h + (\psb_{\cD^t_h, \disc{\eta}[\phi^t_h]}  
    - \realb_h)[V^\pi_{h+1}]}(x_1, \pi(x_1)).
\end{align*}
In the following we focus on iteration $t$, so we will drop the superscript $t$ for notational simplicity.
Let $g_h := (\realb_h - \psb_{ \cD_h, \disc{\eta}[\phi_h]})[V^\pi_{h+1}]$,
note that $g_h \in \cF_h^1$ (by construction of $\cF_h^1$). We first check that 
\pref{lem:order_preserve} holds since $\|g_h\|_{\infty} \leq 2$. Then by \pref{lem:osb_pseudobackup}:
\begin{align*}
    &\sum_{h=1}^H \prn*{\psb_{\cD, \disc{\eta}[\phi]}}^{\otimes (h-1)} \brk*{g_h}
    \\
    \leq& \sum_{h=2}^H  \prn*{\psb_{\cD, \disc{\eta}[\phi]}}^{\otimes (h-1)}
    \min\crl*{\sqrt{\frac{1}{t \bar \rho^t_h\prn*{\ball{\eta}[\phi_h](\cdot)}+ \lambda^t} } 
    \sqrt{2t A_\eta^2 \En_{x,a \sim \bar \gamma_h^t} \brk*{g_h^2(x,a)} + \zeta(t)},2} + \sqrt{A \En_{x,a \sim \bar \rho^t_1}\brk*{g_1^2(x,a)}} + 3H\eta \\
    \leq & \sum_{h=2}^H \prn*{\psb_{\cD, \disc{\eta}[\phi]}}^{\otimes (h-1)}
    \min\crl*{\sqrt{\frac{1}{t \bar \rho^t_h\prn*{\ball{\eta}[\phi_h](\cdot)}+ \lambda^t} }
    \underbrace{\sqrt{2t A_\eta^2 \erep(t) + \zeta(t)}}_{\alpha^t},2} + \sqrt{A \erep(t)} + 3H\eta. 
\end{align*}
By the construction of bonus we have 
\begin{align*}
    \sum_{h=1}^H \prn*{\psb_{\cD^t, \disc{\eta}[\phi^t]}}^{\otimes (h-1)}
    \brk*{\widehat b^t_h} &= 
    \sum_{h=1}^H \prn*{\psb_{\cD^t, \disc{\eta}[\phi^t]}}^{\otimes (h-1)} 
    \brk*{\min \crl*{\widehat \alpha^t \cdot \sqrt{\frac{1}{N_{\eta,\phi^t_h}(\cdot ,\cD^t_{1,h}) + \lambda^t}}, 2}} \\
    &\geq \sum_{h=1}^H \prn*{\psb_{\cD^t, \disc{\eta}[\phi^t]}}^{\otimes (h-1)}
    \brk*{\min \crl*{c\widehat \alpha^t \cdot \sqrt{\frac{1}{t \bar \rho^t_h\prn*{\ball{\eta}[\phi_h](\cdot)}+ \lambda^t} }, 2}}. 
    \tag{Concentration of the bonus, \pref{lem:concentration_bonus}}
\end{align*}
Putting everything together we have 
\begin{align*}
    &f^{\pi^t}(x_1) - V^{\pi^t}(x_1) \geq
    \sum_{h=1}^H \prn*{\psb_{\cD^t, \disc{\eta}[\phi^t]}}^{\otimes (h-1)}
    \Bigg[\min \crl*{c\widehat \alpha^t \cdot \sqrt{\frac{1}{t \bar \rho^t_h\prn*{\ball{\eta}[\phi_h](\cdot)}+ \lambda^t} }, 2} - \\
    &~~~~~~~~~~~~~~~~~~~~~~~~~~~~~~~~~~~~~~~~~~~~~~~~~~~~~~~~~~
    \prn*{\min\crl*{\sqrt{\frac{1}{t \bar \rho^t_h\prn*{\ball{\eta}[\phi_h](\cdot)}+ \lambda^t} }
    \alpha^t,2} + \sqrt{A \erep(t)} + 3H\eta}\Bigg],
\end{align*}
and finally by construction of $\widehat \alpha^t$ we complete the proof.
\end{proof}

\subsubsection{Proving the Regret}\label{sec:regret_proof}
\begin{theorem}[Pseudo regret for pseudobackups]\label{thm:regret}
    With probability at least $1-\delta$, setting parameters 
    \begin{align*}
      \lambda^t = \Theta\prn*{t^{\frac{\dimsa}{\widetilde{d}+2}} \log\prn*{\frac{t|\Phi|}{\delta}}}, 
      \quad \widehat \alpha^t = \Theta\prn*{ t^{\frac{\widebar d}{\widetilde d}} \log \prn*{\frac{t|\Phi|}{\delta}}},
    \end{align*}
    let $\widehat \pi$ be the output of the \mainalg, we have 
    \begin{align*}
        \reg(T) \leq O\prn*{H^2 T^{\frac{\widetilde d}{\widetilde d + 2}} \sqrt{\log (T \abs*{\Phi} / \delta)} },
    \end{align*}
    where $\widetilde d =  6\dimsa^2 + 8\dimsa \dima + 10\dimsa + 8\dima +2$.
\end{theorem}
\begin{proof}[\pfref{thm:regret}]
  By the standard decomposition \citep{jiang2017contextual} we have 
  \begin{align*}
    \sum_{t=1}^T J(\pi^\ast) - J(\pi^t) = \sum_{t=1}^T \prn*{ V^{\pi^\ast}(x_1) - f^{t;\pi^\ast}(x_1)} 
    + \prn*{ f^{t;\pi^\ast}(x_1) - V^{\pi^t}(x_1)}.
  \end{align*}
  By \pref{lem:optimism}, we have
  \begin{align*}
    \sum_{t=1}^T \prn*{ V^{\pi^\ast} - V^{\pi^t}} &\leq \sum_{t=1}^T \prn*{ f^{t;\pi^\ast} - V^{\pi^t}} 
    + \prn*{ \sqrt{A \erep(t)} + 3H\eta} 
    \leq \sum_{t=1}^T \prn*{ f^{\pi^t} - V^{\pi^t}} + \prn*{ \sqrt{A \erep(t)} + 3H\eta}.
  \end{align*}
  By \pref{lem:bellman_simulation}, we have
  \begin{align*}
    \sum_{t=1}^T \prn*{f^{\pi^t}(x_1) - V^{\pi^t}(x_1)} \leq \sum_{t=1}^T \sum_{h=1}^H \prn*{\realb}^{\otimes (h-1)}
    [\underbrace{\widehat b^t_h}_{A} + \underbrace{(\psb_{\cD^t_h, \disc{\eta}[\phi^t_h]}
     - \realb_h)[f^{\pi^t}_{h+1}]}_{B}](x_1, \pi^t(x_1)).
  \end{align*}
  To bound $A$, note that by construction $\nrm*{\widehat b^t_h}_\infty \leq 2$, i.e., $L=2$, and $b^t_h \in \cF_h$, 
  then by \pref{lem:osb_bellman}, we have
  \begin{align*}
    &\sum_{h=1}^H \prn*{\realb}^{\otimes (h-1)}[\widehat b^t_h]\\
     \leq& \sum_{h=2}^H \prn*{\realb}^{\otimes (h-1)} \sqrt{\frac{1}{t \bar \rho^t_h\prn*{\ball{\eta}[\phi^\ast_h](\cdot)}+ \lambda^t} }
     \sqrt{tA_\eta^2 \En_{x,a \sim \bar \gamma_h^t} \brk*{ \prn*{\widehat   b^t_h}^2(x,a) } + \lambda^t 4 \dimeta} 
     + \sqrt{tA \En_{x,a \sim \bar \rho^t_1}\brk*{\prn*{\widehat b^t_1}^2(x,a)}} + 3H\eta. \\
  \end{align*}
  Note that 
  \begin{align*}
    \En_{x,a \sim \rho_h}\brk*{\prn*{\widehat b^t_h}^2(x,a)}
    & \leq c^2\prn*{\widehat \alpha^t}^2 \En_{x,a \sim \rho_h} \brk*{ \frac{1}{t \bar \rho^t_h\prn*{\ball{\eta}[\phi^t_h](x,a)}+ \lambda^t} }
    \tag{\pref{lem:concentration_bonus}}\\
    & = \prn*{ \alpha^t}^2 \En_{x,a \sim \rho_h} 
    \brk*{\frac{1}{t \bar \rho^t_h\prn*{\ball{\eta}[\phi^t_h](x,a)}+ \lambda^t}} \\
    &\leq \prn*{\alpha^t}^2  \frac{\dimeta}{t}.
  \end{align*}

  To bound term $B$, first let us denote the shorthand:
  \begin{align*}
    \delta_h := \frac{1}{2H+1}(\psb_{\cD^t_h, \disc{\eta}[\phi^t_h]}
    - \realb_h)[f^\pi_{h+1}],
  \end{align*}
  since again $\nrm*{\widehat b^t_h}_\infty \leq 2$, we have $\nrm*{f^\pi_{h+1}}_{\infty} \leq 2H+1$,
  and thus $\nrm*{\delta_h}_\infty \leq 2$. Then by \pref{lem:osb_pseudobackup}, 
  term $B$ can be bounded as the following:
  \begin{align*}
    B &\leq (2H+1)\sum_{h=2}^H \prn*{\realb}^{\otimes (h-1)} 
    \sqrt{\frac{1}{t \bar \rho^t_h\prn*{\ball{\eta}[\phi^\ast_h](\cdot)}+ \lambda^t} } 
    \sqrt{tA_\eta^2 \En_{x,a \sim \gamma^t_h} \brk*{\delta_h^2(x,a)} + \lambda^t 4 \dimeta} +\\
    &\hspace{0.6\textwidth} 
    (2H+1)\sqrt{tA\En_{x,a \sim \bar \rho^t_1}\brk*{\delta_1^2(x,a)}} + 3H\eta \\
    &\leq (2H+1)\sum_{h=2}^H \prn*{\realb}^{\otimes (h-1)} 
    \sqrt{\frac{1}{t \bar \rho^t_h\prn*{\ball{\eta}[\phi^\ast_h](\cdot)}+ \lambda^t} } 
    \sqrt{tA_\eta^2 \erep(t) + \lambda^t 4 \dimeta} + (2H+1)\sqrt{tA\erep(t)} + 3H\eta,
  \end{align*}
where the second line is because $\frac{1}{2H+1}f^\pi_{h+1} \in \cF_{h+1}$ and thus we invoke \pref{thm:rep_learn}.
  Finally we bound the potential term in front of term $A$ and $B$, by \pref{lem:concentration_potential}, we have 
  \begin{align*}
    &\sum_{t=1}^T \prn*{\realb}^{\otimes (h-1)}
    \sqrt{\frac{1}{t \bar \rho^t_h\prn*{\ball{\eta}[\phi^\ast_h](\cdot)}+ \lambda^t} }(x_1,\pi^t(x_1)) \\
    \leq& \sqrt{T \sum_{t=1}^T \prn*{\realb}^{\otimes (h-1)} \frac{1}{t \bar \rho^t_h\prn*{\ball{\eta}[\phi^\ast_h](\cdot)}+ \lambda^t}(x_1,\pi^t(x_1)) } \tag{Cauchy Schwarz}\\
    \leq& \sqrt{\dimeta T \ln \prn*{1 + \frac{T }{\lambda_1 \dimeta}}}.
  \end{align*}

  Now recall the construction of $\alpha^t$:
  \begin{align*}
    \alpha^t &= \sqrt{tA_\eta^2 \erep(t) + \zeta(t)} 
    =  O\prn*{\sqrt{tA_\eta^2 \erep(t) + \lambda^t \dimeta}} \tag{$\ehist(t) \leq \erep(t)$}. 
  \end{align*}
  By \pref{thm:rep_learn} and note that $L = 2$, we have
  \begin{align*}
    \erep(t) = \frac{352 (4^{\dimsa{}}) \dgamma 
    \log\prn*{4 \abs*{\Phi} \cdot \cN_{\infty}(\cF, \gamma)/(\delta \gamma)}}{t} + 96\gamma^2,
  \end{align*}
  then by \pref{lem:covering_num}, we have
  \begin{align*}
    \erep(t) =O\prn*{\frac{\dimeta^2 \dgamma^2 \log(t \abs*{\Phi} / (\gamma\delta))}{t} + \gamma^2},
  \end{align*}
  combining everything together and taking leading terms we have:
  \begin{align*}
    \sum_{t=1}^T \prn*{ V^{\pi^\ast} - V^{\pi^t}} \leq O\prn*{H^2 A_\eta^2 \sqrt{\dimeta^3 \dgamma^2 T \log (T\dimeta \abs*{\Phi} / (\gamma \delta)) + T^2\dimeta \gamma^2} + TH\eta }.
  \end{align*}
  Now we can take $\gamma = T^{-\frac{1}{2d+2}}$, and $d = \dimsa{}$, 
  and we get
  \begin{align*}
    \sum_{t=1}^T \prn*{ V^{\pi^\ast} - V^{\pi^t}} \leq O\prn*{H^2 A_\eta^2 \sqrt{\dimeta^3 T^{\frac{2d+1}{d+1}} \log (T\dimeta \abs*{\Phi} / \delta)} + TH\eta }.
  \end{align*}
  Finally taking (note that $A_\eta = (1/\eta)^{\dima}$)\[\eta = T^{-\frac{2}{(3\dimsa + 4\dima + 2)(2\dimsa + 2)}}\] we get:
  \begin{align*}
    \sum_{t=1}^T \prn*{ V^{\pi^\ast} - V^{\pi^t}} \leq O\prn*{H^2 T^{\frac{\widetilde d}{\widetilde d + 2}} \sqrt{\log (T \abs*{\Phi} / \delta)} },
  \end{align*}
  where $\widetilde d =  6\dimsa^2 + 8\dimsa \dima + 10\dimsa + 8\dima +2$.
\end{proof}

Finally, to prove \pref{thm:main_alg_formal}, we divide the right-hand-side by $T$, bound by $\veps$ and solve for $T$.

\subsubsection{Supporting Lemmas}
\label{sec:supporting_lemmas}

 \begin{lemma}[Concentration of potential]\label{lem:concentration_potential}
   We have that 
   \begin{align*}
     \sum_{t=1}^T \prn*{\cP}^{\otimes (h-1)} \brk*{\frac{1}{t \bar 
     \rho^t_h\prn*{\ball{\eta}[\phi^\ast_h](\cdot)}+ \lambda_t}}(x_1,\pi(x_1)) 
     \leq \dimeta \ln \prn*{1 + \frac{T}{\lambda^1\eta^{\dimeta} }},
   \end{align*}
   where $\dimeta = \prn*{\frac{1}{\eta}}^{\dimsa}$.
 \end{lemma}
 \begin{proof}[\pfref{lem:concentration_potential}]
   First let us define the vector $\phi^\ast_{\eta;h}$ as a $d_\eta$-dimensional
    one-hot vector, where the $i$-th entry is $1$ if $\phi^\ast_h(x,a) \in 
    \ball{\eta}^i[\phi^\ast_h]$, and $0$ otherwise. Then we have the following identity:
   \begin{align*}
     \frac{1}{t \bar \rho^t_h\prn*{\ball{\eta}[\phi^\ast_h](x,a)}+ \lambda_t} = 
     \phi^\ast_{\eta;h}(x,a) \Sigma^{-1}_{t;\phi^\ast_{\eta;h};\rho_h} \phi^\ast_{\eta;h}(x,a),
   \end{align*}
   where 
   \begin{align*}
     \Sigma_{t;\phi^\ast_{\eta;h};\rho_h} = \sum_{\tau=1}^t \En_{x,a \sim \rho^\tau_h}
      \phi^\ast_{\eta;h}(x,a) \prn*{\phi^\ast_{\eta;h}(x,a)}^{\top} + \lambda_t I.
   \end{align*}
   Hence we establish the relationship between the potential function that we 
   are interested in and linear models. Then by 
   \pref{lem:trace_to_det} and~\pref{lem:potential}, we have that:
   \begin{align*}
     \sum_{t=1}^T \En_{x,a \sim \rho^t_h} \brk*{\frac{1}{t \bar 
     \rho^t_h\prn*{\ball{\eta}[\phi^\ast_h](x,a)}+ \lambda_t}} 
     \leq 2 \ln \det\prn*{\Sigma_{T;\phi^\ast_{\eta;h};\rho_h}} - 2 \ln \log \prn*{\lambda I} 
     \leq d_\eta \ln \prn*{1 + \frac{T}{\lambda_1 \dimeta }}, 
   \end{align*}
   where we use the fact $B=1$ because the vectors are one-hot.
 \end{proof}
 
 Finally we can also leverage the connection to linear models to prove the 
 concentration of bonus:
 \begin{lemma}[Concentration of bonus; Lemma 22 of \citet{zhang2022efficient}]\label{lem:concentration_bonus}
   Set $\lambda_T = \Theta\prn*{\dimeta \ln\prn*{T \abs*{\Phi} / \delta}}$, then with probability at least $1-\delta$, we have 
   for all $t \in [T]$ and $\phi \in \Phi$:
   \begin{align*}
      c_1\sqrt{\frac{1}{t \bar \rho^t_h\prn*{\ball{\eta}[\phi_h](x,a)}+ \lambda_t}} 
     \leq \sqrt{\frac{1}{N_{\eta,\phi}(x,a,\cD^t_{1,h}) + \lambda_t}} 
     \leq c_2 \sqrt{\frac{1}{t \bar \rho^t_h\prn*{\ball{\eta}[\phi_h](x,a)}+ \lambda_t}} .
   \end{align*}
 \end{lemma}
 
 \begin{lemma}[Covering number of the discriminator class]\label{lem:covering_num}
  For any $h \in [H]$, define 
  \begin{align*} 
    &\cF^1_{\eta,h} = \crl*{f(x): \En_{a \sim \unifpi}\brk*{w^\top \disc{\eta}[\phi](x,a) - 
    g\prn*{\widetilde{\phi}(x,a)}} \mid \phi, \widetilde \phi \in \Phi_{h+1}, g \in \Lip, w \in \bbR^{\dimeta}, \|w\|_{\infty} \leq 1},\\ \notag
    &\cF^2_{\eta,h} =  \notag\\
    &\crl*{f(x): \max_a \prn*{\frac{R_{h+1}(x,a) + \min\crl*{w^\top \disc{\eta}[\phi](x,a),2}}{2H+1} + 
    \widetilde w^\top \disc{\eta}[\phi](x,a)} \mid \phi \in \Phi_{h+1}, w, \widetilde w \in \bbR^{\dimeta}, \|w\|_{\infty} \leq c, \|\widetilde w\|_{\infty} \leq 2} \notag.
  \end{align*}
  Then $\cF^1_{\eta,h}$ and $\cF^2_{\eta,h}$ have $\gamma$-covering number 
  \begin{align*}
     N_{\infty}(\cF^1_{\eta,h}, \gamma) \leq  |\Phi_{h+1}|^2 
      \prn*{\frac{2}{\gamma}}^{\dimeta}\prn*{\frac{1}{\gamma}}^{2 d_\gamma}\mathand
     N_{\infty}(\cF^2_{\eta,h}, \gamma) \leq |\Phi_{h+1}| 
      \prn*{\frac{4}{\gamma}}^{\dimeta} \prn*{\frac{2c}{\gamma}}^{\dimeta}.
  \end{align*}
\end{lemma}
\begin{proof}[\pfref{lem:covering_num}]
  For $\cF^1_{\eta,h}$, note that the size of 
  a $\frac{\gamma}{2}$-cover of $\{w \in \bbR^{\dimeta} \mid 
  \nrm{w}_{\infty} \leq 1\}$ is bounded by $\prn*{\frac{2}{\gamma}}^{\dimeta}$, 
  and the size of a $\frac{\gamma}{2}$-cover of $\Lip$ is bounded by $\prn*{\frac{1}{\gamma}}^{2 d_\gamma}$ \citep{wainwright2019high},
  and taking union with $\Phi \times \Phi$ we complete the calculation of the covering number of $\cF^1_{\eta,h}$.
  The calculation of the covering number of $\cF^2_{\eta,h}$ is similar by calculating the 
  size of $\frac{\gamma}{2}$-cover of $\{w \in \bbR^{\dimeta} \mid 
  \nrm{w}_{\infty} \leq 2\}$ and $\{w \in \bbR^{\dimeta} \mid \nrm{w}_{\infty} \leq c\}$.
\end{proof}

\subsection{Proof of \pref{prop:homer}}
\label{app:homer}
\newcommand{\Phat}{\wh{P}}
\newcommand{\Pmusik}{P_{\mathsf{musik}}}
\newcommand{\Pcl}{P_{\mathsf{cl}}}
\newcommand{\aone}{\mathfrak{a}}
\newcommand{\atwo}{\mathfrak{b}}

\icml{
\paragraph{Background on multi-step inverse kinematics}
The multistep inverse kinematics approach
\citep{lamb2023guaranteed,mhammedi2023representation} involves predicting the
action $a_h\sim\unif(\cA)$ at time $h$ from the observation $x_h$ and a future
observation $x_t$ for fixed $t>h$ under a roll-out policy $\pi$. When
the action space is finite, one can show that the optimal
population-level objective takes the form
\begin{align}
  (x_h,a_h,x_{t}) \mapsto \frac{\bbP^{\pi}(\phi^\star(x_{t}) \mid
  \phi^\star(x_h),a_h)}{\sum_{a\in\cA}\bbP^{\pi}(\phi^\star(x_{t}) \mid
  \phi^\star(x_h),a)}.
  \label{eq:musik}
\end{align}
Similar to contrastive learning, the a key property of this objective
is that it depends on the observation only through the corresponding
latent state, a central property used
by~\citep{mhammedi2023representation}. However, analogously to
contrastive learning, this property alone is not sufficient for
sample-efficient learning in the \framework, because we need to construct a low-complexity function
class to express the optimal predictor \eqref{eq:musik} in this latent
space. Unfortunately, the optimal predictor for multistep inverse
kinematics objective may not be a Lipschitz function of the latent
state, even when though transition dynamics themselves are Lipschitz.}

\begin{proof}[\pfref{prop:homer}]
  To begin, we focus on multi-step inverse kinematics
  \citep{lamb2023guaranteed,mhammedi2023representation}. Consider the
  case where $\abs*{\cA}=\crl*{\aone,\atwo}$ (two actions will
  suffice, as the difficulty arises from continuity of the latent
  state space). Let $t>h$ be fixed. Consider a setting where we sample
  $x_h\sim{}\rho_h$ (an arbitrary roll-in distribution),
  $a_h\sim\piunif$, and sample $x_t$ by executing a given policy $\pi$
  from steps $h+1,\ldots,t$. The multi-step inverse kinematics
  objective performs conditional density estimation under this
  process:\loose
  \[
    (\Phat,\phi) =
    \argmax_{P\in\cP,\phi\in\Phi}\Ehat\brk*{\log(P(a_t\mid{}\phi(x_t),\phi(x_h))}
  \]
  for a function class $\cP$. If $\cP$ is unconstrained and
  $\phistar\in\Phi$, the population-level optimizer for this objective
  is
  \begin{align*}
    \Pmusik(a_h\mid{}x_t,x_h) \ldef \frac{\bbP^{\pi}(\phi^\star(x_{t}) \mid
    \phi^\star(x_h),a_h)}{\sum_{a\in\cA}\bbP^{\pi}(\phi^\star(x_{t}) \mid
    \phi^\star(x_h),a)}.
  \end{align*}
  Now let us focus on the case where $t = h+1$, i.e., the one-step inverse
  kinematics, which is captured by the multi-step inverse kinematics objective. 
  And the one-step inverse kinematics optimizer is
  \begin{align*}
    \Pmusik(a_h\mid{}x_{h+1},x_h) \ldef \frac{\bbP^{\pi}(\phi^\star(x_{h+1}) \mid
    \phi^\star(x_h),a_h)}{\sum_{a\in\cA}P_h(\phi^\star(x_{h+1}) \mid
    \phi^\star(x_h),a)},
  \end{align*}
  where $P_h$ is the latent dynamics at timestep $h$.
  Now recall the latent state space $\cS$ which is a metric space with metric $\DS$.
  Then for each $x \in \cX$, let $s = \phistar(x)$ be the corresponding latent state
  according to the ground truth decoder. Then we can rewrite the one-step inverse kinematics
  optimizer as
  $\Pmusik(a_{h}\mid{}x_{h+1},x_h)=\Pmusik(a_{h}\mid{}\phistar(x_{h+1}),\phistar(x_h))$
  for
  \begin{align*}
    \Pmusik(a_{h}\mid{}s_{h+1},s_h) \ldef \frac{\bbP^{\pi}(s_{h+1} \mid
    s_h,a_h)}{\sum_{a\in\cA}P_h(s_{h+1}\mid
    s_h,a)}.
  \end{align*}
  To prove the result, it suffices to show that $\Pmusik$ is not Lipschitz 
  with respect to the parameter $s_{h}$, i.e., for some fixed $s_{h+1}$ and $a_h$,
  there exists $s^1_h$ and $s^2_h$ such that 
  \begin{align}\label{eq:nonlipschitz}
  \abs*{\Pmusik(a_h \mid s_{h+1}, s^1_h) - \Pmusik(a_h \mid s_{h+1}, s^2_h)} > \DS(s^1_h,s^2_h).
  \end{align}
  
  Now let us fix $s_{h+1}$ and $a_h = \aone$. Consider a pair of states
  $s^1_h$ and $s^2_h$ with $\DS(s^1_h,s^2_h) = \delta$ for a given
  parameter $\delta>0$. Let us define the dynamics such that (i)
  $P_h(s_{h+1} \mid s^1_h, \aone) = 2\delta$, (ii)
  $P_h(s_{h+1} \mid s^2_h, \aone) = \delta$, (iii)
  $P_h(s_{h+1} \mid s^1_h, \atwo) = P_h(s_h \mid s^2_h,
  \atwo) = \delta$,
  (iv) there is one state $\widetilde s_{h+1}$ such that 
  $P_h(\widetilde s_{h+1} \mid s^1_h, \aone) = \delta$, and 
  $P_h(\widetilde s_{h+1} \mid s^2_h, \aone) = 2\delta$,
  (v) $P_h(s'_{h+1} \mid s^1_h,a) = P_h(s'_{h+1} \mid s^2_h,a),$ for all $a \in \cA, s'_{h+1} \notin \{s_{h+1}, \widetilde s_{h+1}\}$. 
  We can check that $\nrm*{P_h(\cdot \mid s^1_h,a) - P_h(\cdot \mid s^2_h,a)}_{\tv} = \frac{1}{2}\nrm*{P_h(\cdot \mid s^1_h,a) - P_h(\cdot \mid s^2_h,a)}_{1} = \delta = \DS(s^1_h,s^2_h)$
  and thus the construction satisfies the
  Lipschitz latent dynamics condition. However,
  \begin{align*}
    \frac{\bbP^{\pi}(s_{h+1} \mid s^1_h, \aone)}{\sum_a \bbP^{\pi}(s_{h+1} \mid s^1_h, a)} - \frac{\bbP^{\pi}(s_{h+1} \mid s^2_h, \aone)}{\sum_a \bbP^{\pi}(s_{h+1} \mid s^2_h, a)} = \frac{2\delta}{2\delta + \delta} - \frac{\delta}{\delta + \delta} = \frac{1}{6},
  \end{align*}
  but we can take $\delta$ arbitrarily small. This proves \pref{eq:nonlipschitz}.

  For contrastive learning \citep{misra2020kinematic}, recall from the
  main text that for $h\in\brk{H}$, the optimal classifier takes the
  form
  $\Pcl(x_{h+1}, x_h,a_h)=\Pcl(\phistar(x_{h+1}), \phistar(x_h),a_h)$,
  where
  \begin{align*}
    \Pcl(s_{h+1}, s_h,a_h) \ldef  \frac{P_h(s_{h+1} \mid s_h,a_h)}{P_h(s_{h+1} \mid s_h,a_h) + \widetilde{\rho}_{h+1}(s_{h+1})},
  \end{align*}
  for the process in which we sample $s_h\sim\rho_h$ for a data collection
  distribution $\rho_h$; $\wt{\rho}_{h+1}$ denotes the law of
  $s_{h+1}$ when we sample $s_h\sim\rho_h$ and $a_h\sim{}\piunif$. For
  this objective, we can follow exactly the same construction as above, 
  where we define the dynamics as 
  (i)
  $P_h(s_{h+1} \mid s^1_h, \aone) = 2\delta$, (ii)
  $P_h(s_{h+1} \mid s^2_h, \aone) = \delta$, (iii)
  $P_h(s_{h+1} \mid s^1_h, \atwo) = P_h(s_h \mid s^2_h,
  \atwo) = \delta$,
  (iv) there is one state $\widetilde s_{h+1}$ such that 
  $P_h(\widetilde s_{h+1} \mid s^1_h, \aone) = \delta$, and 
  $P_h(\widetilde s_{h+1} \mid s^2_h, \aone) = 2\delta$,
  (v) $P_h(s'_{h+1} \mid s^1_h,a) = P_h(s'_{h+1} \mid s^2_h,a),$ for all $a \in \cA, s'_{h+1} \notin \{s_{h+1}, \widetilde s_{h+1}\}$. 
  Now consider the data collection distribution that puts all the mass on $s_h^2$,
  then by construction we have $\tilde \rho_{h+1}(s_{h+1}) = \delta$.
  Plugging everything in we get 
  \begin{align*}
    \Pcl(\aone \mid s_{h+1}, s^1_h) - \Pcl(\aone \mid s_{h+1}, s^2_h) = \frac{2\delta}{2\delta + \delta} - \frac{\delta}{\delta + \delta} = \frac{1}{6},
  \end{align*}
  and we prove that the optimal classifier for contrastive learning is not Lipschitz
  with respect to $s_h$ as well.
\end{proof}
